\documentclass[11pt]{article}
\usepackage{acl}

\usepackage{times}
\usepackage{latexsym}
\usepackage[T1]{fontenc}
\usepackage[utf8]{inputenc}
\usepackage{microtype}
\usepackage{inconsolata}
\usepackage{graphicx}
\usepackage{float}
\usepackage{enumitem}
\usepackage{subcaption}
\usepackage{tcolorbox}
\tcbuselibrary{breakable}
\usepackage{xcolor}
\usepackage{colortbl}
\usepackage{fontawesome}
\usepackage{multirow}
\usepackage{longtable}
\usepackage{tabularx}
\usepackage{amsmath}
\usepackage{amssymb}
\usepackage{pifont}
\usepackage{makecell}
\usepackage{amsthm}
\usepackage{algorithm}
\usepackage{algpseudocode}
\usepackage{listings}
\usepackage{booktabs}
\usepackage{siunitx}
\usepackage{threeparttable}
\usepackage{nicematrix}
\usepackage{tikz}
\usepackage{mathtools}

\definecolor{successbg}{RGB}{230,236,250}  
\definecolor{stepsbg}{RGB}{236,242,250}    
\definecolor{callbg}{RGB}{243,248,254}     
\definecolor{sfrbg}{RGB}{249,252,255}      

\definecolor{ds}{RGB}{255,244,224}   
\definecolor{gpt}{RGB}{255,248,236}  
\definecolor{cld}{RGB}{255,251,243}  

\newcolumntype{Y}{>{\centering\arraybackslash}X}
\newcommand{\cmark}{\ding{51}}
\newcommand{\xmark}{\ding{55}}
\newcommand{\tmark}{$\triangle$}

\newtheorem{theorem}{Theorem}[section]
\newtheorem{lemma}[theorem]{Lemma}
\newtheorem{corollary}[theorem]{Corollary}
\newtheorem{proposition}[theorem]{Proposition}

\theoremstyle{definition}

\theoremstyle{remark}
\newtheorem{remark}{Remark}[section]

\lstdefinelanguage{JavaScript}{
  keywords={typeof, new, true, false, catch, function, return, null, catch, switch, var, if, in, while, do, else, case, break},
  keywordstyle=\color{blue}\bfseries,
  ndkeywords={class, export, boolean, throw, implements, import, this},
  ndkeywordstyle=\color{darkgray}\bfseries,
  identifierstyle=\color{black},
  sensitive=false,
  comment=[l]{//},
  morecomment=[s]{/*}{*/},
  commentstyle=\color{purple}\ttfamily,
  stringstyle=\color{red}\ttfamily,
  morestring=[b]',
  morestring=[b]"
}

\definecolor{myborder}{RGB}{73, 86, 102}
\definecolor{myRed}{RGB}{240, 48, 159}
\definecolor{mylightblue}{RGB}{235, 245, 255}
\tcbuselibrary{skins}

\definecolor{back_color}{RGB}{255, 245, 245}
\definecolor{frame_color}{RGB}{240, 128, 128}

\title{SkillShift-Bench: Benchmarking Skill Learning\\ under Environment Evolution for Web Agents}

\author{
  Bo-Yu Chen$^{1,2}$ \quad
  Zhi Zhou$^{1,2}$\thanks{Corresponding authors.} \quad
  Yu-Feng Li$^{1,2}$\footnotemark[1] \\
  \normalsize
  $^1$School of Artificial Intelligence, Nanjing University \\
  $^2$National Key Laboratory for Novel Software Technology, Nanjing University \\
  \texttt{\{chenby, zhouz, liyf\}@lamda.nju.edu.cn}
}

\begin{document}
\maketitle

\begin{abstract}
Skill learning enables web agents to reuse prior experience for long-horizon interactive tasks. However, learned skills are typically evaluated in closed environments, overlooking the fact that real-world web deployment involves continuous shifts in invocation context, perceptual grounding, and execution feasibility. To this end, we introduce SkillShift-Bench, the first benchmark for evaluating skill learning with controlled open-environment shifts. SkillShift-Bench simulates Contextual Shift, Perceptual Shift, and Execution Shift across three self-hosted web systems, and evaluates agents with diagnostics covering task success, interaction cost, skill use, and skill failure. Comprehensive experiments reveal systematic robustness degradation under controlled shifts, characterized by lower success rates, higher interaction costs, and more frequent skill failures, with Execution Shift exerting the strongest pressure. We further formulate shifted skill execution as a nonstationary local executability problem and provide both a dynamic-regret analysis and a practical test-time adaptation method, demonstrating the feasibility of improving learned-skill execution without relearning the skill library. Code, tasks, environments, documentation, and the leaderboard are available at \url{https://skillshift-bench.github.io/}.
\end{abstract}

\section{Introduction}
\label{sec:sections-intro:1}

Web agents increasingly complete long-horizon tasks by reusing learned skills rather than planning each step from scratch~\cite{wang2023voyager,wang2025agentworkflowmemory,wang2025inducingprogrammaticskills}. This reuse improves efficiency, but it assumes local compatibility between the current web environment and the context in which the skill was learned. When interface structure, grounding evidence, or runtime conditions change~\cite{cheng2024seeclick,yang2025guirobust,kara2025warex}, a skill may remain semantically relevant while its next operation is no longer locally executable. We study this gap as learned skill executability under same-site environment evolution.

Existing evaluations provide important but incomplete views of this problem. Offline web datasets and simulated web tasks offer broad coverage, but do not test executable skill transfer in interactive browsers~\cite{yao2022webshop,deng2023mind2web}. Self-hosted benchmarks support reproducible execution, but usually instantiate tasks in fixed website snapshots~\cite{zhou2024webarena,koh2024visualwebarena,lesellier2025browsergym}. Remote and open web benchmarks expose agents to realistic change~\cite{drouin2024workarena,he2024webvoyager,pan2024webcanvas,yoran2024assistantbench}, while task content, interface state, and runtime conditions often vary together, limiting fine-grained attribution~\cite{xue2025illusionofprogress}. Recent robustness and evolution benchmarks further show the importance of non-static evaluation~\cite{ishmam2026timewarp,bai2026stressweb,li2026proevolve,kara2025warex}. However, they mainly target general agent robustness rather than whether previously learned reusable skills remain executable when task semantics are preserved. \looseness=-1

We introduce SkillShift-Bench, a controlled benchmark for diagnosing learned skill executability under environment evolution. It defines three mismatch modes aligned with the stages of reusable skill execution: Contextual Shift targets skill invocation, Perceptual Shift targets grounding, and Execution Shift targets runtime feasibility. Across these modes, task intent, task-critical entities, required functionality, and grading criteria are preserved. This design turns website evolution into controlled interventions on learned skill reuse rather than task redefinition.

Using SkillShift-Bench, we evaluate representative learned skill methods~\cite{wang2025agentworkflowmemory,wang2025inducingprogrammaticskills,zheng2025skillweaver,prabhu2026walt} across clean, evolved, and shifted environments. The results show consistent transfer degradation under controlled same-site evolution, with the strongest pressure under Execution Shift. This pattern suggests that many errors arise from local executability mismatch: the learned library can still contain task-aligned skills, while the executor must select an operation that is feasible in the current-page state.

Motivated by this finding, we formulate shifted skill execution as a nonstationary local executability problem, separating skill availability from current action feasibility. We analyze this setting with SWOP (Sliding Window Optimistic Prior-Guided Planning), an idealized prior-guided adaptation model with a variation-dependent dynamic regret guarantee, following the nonstationary online decision making perspective~\cite{besbes2014nonstationarybandits,cheung2020nonstationarymdp,mao2021nonstationaryepisodic}. We then instantiate the same principle as MPCR (Masked Prior-Guided Candidate Reranking), a fixed-library executor-side method that reranks candidate actions using current-page feasibility evidence~\cite{sun2023chatgptsearch,qin2024prp,wu2025pairwisereranking}. Since MPCR keeps the learned skill library fixed, its gains reflect test-time action selection rather than additional skill acquisition.

In summary, we make three contributions:
\begin{itemize}
    \item We introduce SkillShift-Bench, a controlled benchmark for evaluating learned skill executability under Contextual, Perceptual, and Execution Shifts.

    \item We formulate shifted skill execution as nonstationary local executability and analyze SWOP with a variation-dependent dynamic regret guarantee. \looseness=-1

    \item We propose MPCR, a test-time, fixed-library executor-side reranker that improves transfer through feasibility-aware action selection.
\end{itemize}
\section{Related Work}
\label{sec:related_work}

\subsection{Skill learning and reusable web procedures}

Web agents increasingly reuse prior experience, workflows, tools, and executable procedures to support long-horizon interaction~\cite{shinn2023reflexion,wang2023voyager}. 
Recent work studies such reuse through workflow memory~\cite{wang2025agentworkflowmemory}, programmatic skill induction~\cite{wang2025inducingprogrammaticskills}, autonomous browser skill discovery~\cite{zheng2025skillweaver}, learned tools~\cite{prabhu2026walt}, API-based interaction~\cite{song2025beyondbrowsing}, routing over large skill libraries~\cite{zheng2026skillrouter}, and neuro-symbolic induction of logic-grounded skill programs from interaction traces~\cite{shao2026nsi}. 
These methods show that reusable procedures can reduce repeated planning and improve web agent efficiency~\cite{wang2025agentworkflowmemory,wang2025inducingprogrammaticskills,zheng2025skillweaver,prabhu2026walt}. 
Our work studies a complementary question: whether a previously acquired procedure remains usable when the same website changes the context, interface evidence, or runtime conditions needed for execution.

\subsection{Web agent benchmarks and environment realism}

Web agent benchmarks vary in coverage, realism, reproducibility, and attribution. 
Offline or nonbrowser settings such as WebShop, Mind2Web, and ChinaTravel provide broad task coverage and, in constraint-rich planning domains, evaluation against explicitly specified requirements rather than only final outcomes~\cite{yao2022webshop,deng2023mind2web,shao2026chinatravel}. 
Self-hosted environments such as World of Bits, WebArena, VisualWebArena, and BrowserGym support reproducible execution in controlled website snapshots~\cite{shi2017worldofbits,zhou2024webarena,koh2024visualwebarena,lesellier2025browsergym}. 
Remote hosted benchmarks such as WorkArena and WorkArena++ evaluate agents in realistic workplace environments~\cite{drouin2024workarena,boisvert2024workarenapp}. 
Live or online web benchmarks such as WebVoyager, WebLinx, WebCanvas, Online Mind2Web, and AssistantBench improve ecological validity by evaluating agents in more realistic interaction settings~\cite{he2024webvoyager,lu2024weblinx,pan2024webcanvas,xue2025illusionofprogress,yoran2024assistantbench}. 
Together, these benchmarks have advanced general web agent evaluation. 
We focus on an aligned transfer setting for reusable procedures, where task semantics are preserved while local execution conditions vary.

\subsection{Environment variation and diagnostic robustness}

Recent work further studies how web agents behave under environment variation. 
REAL evaluates agents in deterministic simulations of real websites, and WebCanvas emphasizes realistic workflows with controllable benchmark construction~\cite{garg2025real,pan2024webcanvas}. 
TimeWarp studies temporal interface change, StressWeb introduces structured interaction variability, ProEvolve supports programmable benchmark evolution, and WAREX injects website and network failures into existing benchmarks~\cite{ishmam2026timewarp,bai2026stressweb,li2026proevolve,kara2025warex}. 
These studies show that static evaluation can obscure important robustness behavior~\cite{kara2025warex,bai2026stressweb}, a concern that also holds beyond the browser, where statically trained tool-use agents degrade once tool interfaces and specifications depart from their training distribution~\cite{lv2026openworldtooluse}, in line with the open-environment view that training-time assumptions need not hold at deployment~\cite{guo2025robustssl}. 
We build on this motivation, but use environment variation as an intervention on learned skill transfer. 
The focus is not only whether an agent remains successful, but whether a previously useful procedure remains invokable, groundable, and executable under aligned task semantics. Appendix~\ref{app:benchmark-comparison} provides a criterion-level comparison with existing web-agent benchmarks. \looseness=-1

\subsection{Breakage and repair of recorded web procedures}

Software engineering studies an analogous problem as web-test maintenance: recorded end-to-end scripts break when an application evolves, mostly through changed element locators and reorganized page structure rather than removed functionality~\cite{imtiaz2019testbreakage,hammoudi2016recordreplay}. 
Remedies either harden locators against structural change, as in robust XPath generation~\cite{leotta2016robula}, or self-heal a broken script against the evolved interface using structural~\cite{choudhary2011water} or visual~\cite{stocco2018visualrepair} evidence of the intended target, while a separate line stabilizes asynchronous waits~\cite{pei2025asyncwait}. 
Perceptual Shift instantiates these breakage causes at the level of a learned procedure, perturbing the same locator and DOM-structure signals that these methods defend~\cite{leotta2016robula,stocco2018visualrepair}. 
Two differences shape the remedy: a skill is selected at run time rather than executed unconditionally, so a still-valid procedure may be invoked where it no longer applies; and Execution Shift breaks execution through runtime state rather than interface change~\cite{kara2025warex}, which no repair of a stored script can anticipate. 
We therefore adapt the executor instead of repairing the library.

\subsection{Test-time executor-side adaptation for shifted skill execution}

Shifted skill execution is related to nonstationary decision making, GUI grounding, candidate reranking, and reliability analysis. 
Nonstationary bandits and reinforcement learning study decision making under changing rewards or transitions~\cite{besbes2014nonstationarybandits,cheung2020nonstationarymdp,mao2021nonstationaryepisodic,feng2023nonstationarygfa}, and test-time adaptation adjusts a deployed model from unlabeled test-time evidence, including under open-world data shift~\cite{zhou2023ods}, at the prompt level for zero-shot models~\cite{zhang2024adaprompt}, and with recent analyses of when such adaptation is learnable~\cite{zhou2026ttalearnability}. 
GUI grounding methods improve target identification under visual and structural variation~\cite{zheng2024gpt4vgrounded,cheng2024seeclick,gou2025uground}. 
Reranking methods support selection among candidate outputs~\cite{sun2023chatgptsearch,qin2024prp,wu2025pairwisereranking}, and reliability studies show that aggregate success may hide unstable execution behavior~\cite{kara2025warex}. 
These directions provide useful context for test-time executor adaptation. 
Our setting connects them through fixed-library skill reuse, where the executor selects among prior-guided skill and primitive candidates using current-page feasibility evidence.
\section{SkillShift-Bench: Benchmark Design}
\label{sec:benchmark}

SkillShift-Bench is designed to diagnose learned-skill executability under controlled same-site environment evolution. It preserves task semantics across aligned variants and varies only the local conditions that affect skill reuse. The following sections describe the transfer setting, shift modes, benchmark construction, and diagnostic metrics.

\begin{figure*}[t]
    \centering
    \includegraphics[width=\linewidth]{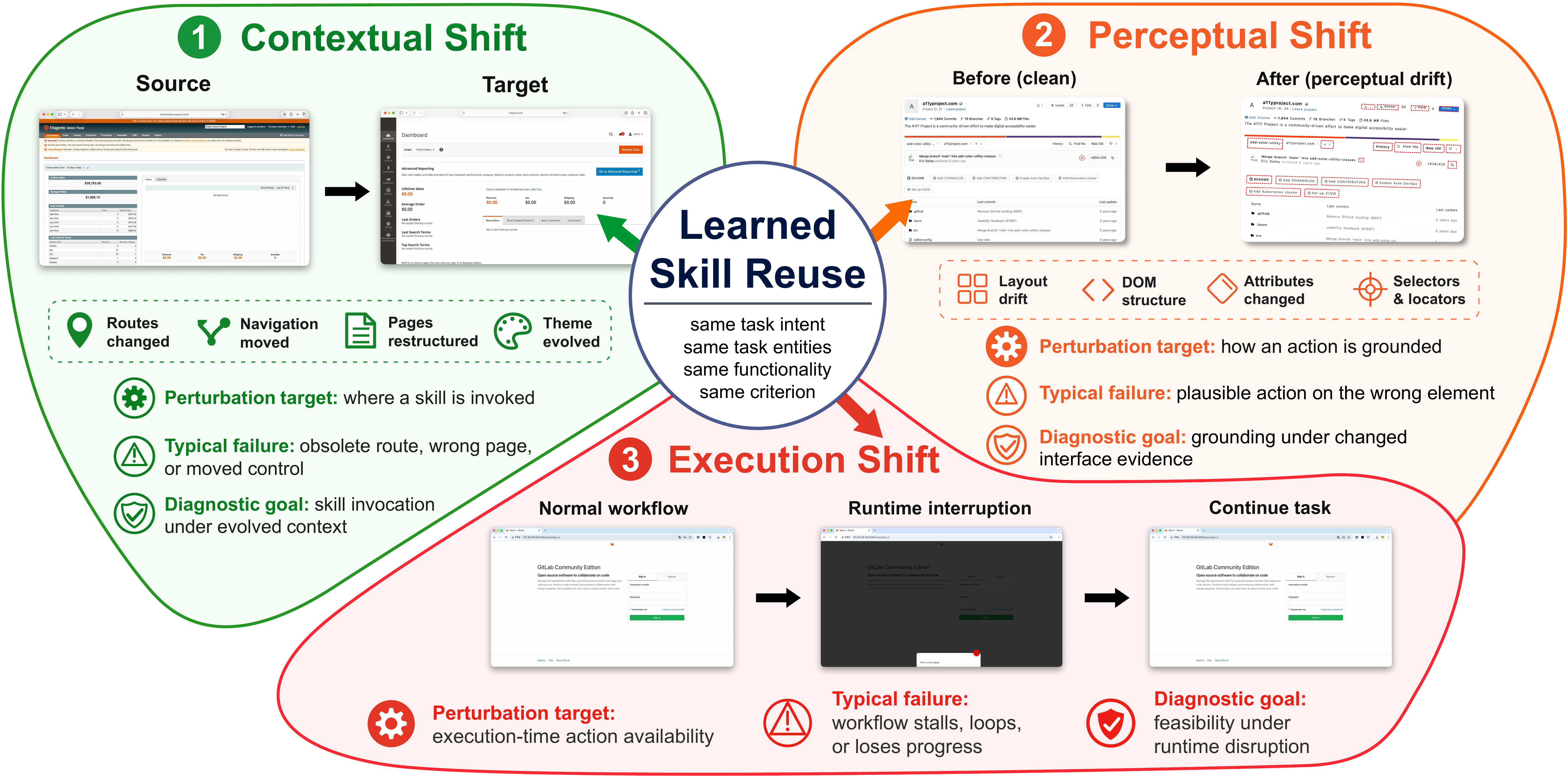}
    \caption{
    Three controlled shift modes in SkillShift-Bench. Aligned variants preserve task intent, task-critical entities, required functionality, and success criteria, while Contextual, Perceptual, and Execution Shifts respectively target skill invocation, grounding evidence, and runtime feasibility.
    }
    \label{fig:shift_modes}
\end{figure*}

\subsection{Controlled Transfer Setting}
\label{sec:controlled_transfer}

SkillShift-Bench is organized around aligned task families. Each family contains a source context, where a skill is learned or selected, and an evolved target context, where the same task is attempted under changed website conditions. To support fine-grained attribution, we enforce four invariants. Goal invariance keeps the instruction tied to the same task intent. Entity invariance aligns task-critical objects such as projects, issues, products, posts, pages, and user accounts. Functionality invariance ensures that the operations required for completion remain available. Outcome invariance keeps the automatic success criteria equivalent across variants.

This design keeps the task fixed while varying the website conditions under which learned skills are reused. 
SkillShift-Bench therefore evaluates skill transfer under controlled environment evolution rather than under task redefinition.

\subsection{Skill Environment Mismatch Modes}
\label{sec:mismatch_modes}

A reusable skill must be invoked in a compatible context, grounded to the intended interface evidence, and executed under feasible runtime conditions. 
SkillShift-Bench defines three mismatch modes that correspond to these stages, as shown in Figure~\ref{fig:shift_modes}.

\paragraph{Contextual Shift.}
Contextual Shift evaluates skill invocation after the deployment context evolves. 
We instantiate it with paired versions or themes of the same website, where routes, navigation, page structure, workflow layout, or control placement may change while the task remains aligned. \looseness=-1

\paragraph{Perceptual Shift.}
Perceptual Shift evaluates grounding under changed interface evidence. 
It applies task-preserving perturbations to layout, DOM structure, element attributes, and locator signals, while keeping the task-critical objects and required functionality intact~\cite{imtiaz2019testbreakage,leotta2016robula}. \looseness=-1

\paragraph{Execution Shift.}
Execution Shift evaluates runtime feasibility during task execution. 
It injects reproducible modal blockers that intercept interactions, testing whether the agent can recover from execution-time interruptions and continue the original task flow~\cite{kara2025warex}.

Implementation details for the perceptual and execution perturbation modules are given in Appendix~\ref{sec:app_shift_perturbation_impl}, together with the configurable injection schedules for Execution Shift in Appendix~\ref{sec:app_injection_regimes}. Appendix~\ref{sec:app_failure_modes} shows that recorded failure trajectories exhibit stage-distinct signatures matching this decomposition. \looseness=-1

\subsection{Benchmark Instantiation}
\label{sec:benchmark_construction}

We instantiate SkillShift-Bench on three controlled web systems: GitLab, Magento, and WordPress. 
They cover developer collaboration, e-commerce administration, and content management, and each supports reproducible deployment, controlled modification, and automatic state inspection. 
The selected version and theme pairs follow official release, migration, and theme documentation, with environment-selection criteria and pairing details provided in Appendix~\ref{sec:app_b_env_pairing}.

\paragraph{Task construction.} Task construction follows a validity-first pipeline. For each website, we instantiate task-relevant data and align the entities and states needed across source and target contexts. We then write natural language instructions around these aligned entities and implement task-specific automatic scorers for final state verification. Each task family is checked against the four invariants in Section~\ref{sec:controlled_transfer} before shifted variants are generated. Only validated families enter the shift generation stage, where local execution conditions are varied while the task remains fixed. Validation is manual and covers instruction validity, scorer correctness, functional availability in both variants, and the existence of an executable solution under each of the six conditions; Appendix~\ref{sec:app_task_validation} gives the full protocol and its distinction from a human performance baseline. This process yields 355 natural language tasks across three web systems.

\paragraph{Evaluation conditions.} Each task family is instantiated under six conditions. \texttt{Src} is the clean source context. \texttt{Src-Perc} and \texttt{Src-Exec} apply Perceptual Shift and Execution Shift to the source context. \texttt{Tgt} is the clean target context and measures transfer under Contextual Shift. \texttt{Tgt-Perc} and \texttt{Tgt-Exec} apply Perceptual Shift and Execution Shift to the target context. Together, these conditions allow controlled comparisons across clean execution, same-site transfer, and additional perceptual or execution shifts under matched tasks.

\subsection{Diagnostic Metrics}
\label{sec:evaluation_diagnostics}

We characterize each run with four complementary diagnostics. \textbf{Success Rate (SR)} measures the fraction of tasks that meet the specified success criterion. \textbf{Average Steps (AS)} measures the mean number of actions before success, failure, or timeout, capturing interaction cost and recovery burden. \textbf{Average Skill Calls (ASC)} measures the mean number of learned skill invocations per task, capturing reliance on the skill library. \textbf{Skill Failure Rate (SFR)} measures the proportion of invoked skills that fail during attempted execution, with runtime errors counted as skill failures. These diagnostics capture final correctness, trajectory cost, skill reuse behavior, and local executability under controlled environment evolution. SFR aggregates retrieval-level and executability-level failures, and is unstable when skills are rarely invoked, since the ratio is then taken over a very small denominator; we therefore report it as a supplementary diagnostic and characterize low-invocation regimes through ASC. Appendix~\ref{sec:app_failure_modes} decomposes the two causes.
\section{Prior-Guided Test-Time Executor Adaptation}
\label{sec:method_improved}

SkillShift-Bench separates task semantics from local execution conditions by preserving the former while varying the latter across aligned environments. This separation motivates a prior-guided view of skill transfer: a learned library may remain task-relevant, but the executor must still choose actions that are executable in the current page state. This section first formalizes shifted skill execution as nonstationary local executability and then introduces an executor-side reranking method that keeps the learned skill library fixed while adapting test-time action selection.

\subsection{SWOP as Prior-Guided Adaptation}
\label{sec:swop}

We formalize this view with SWOP (Sliding-Window Optimistic Prior-Guided Planning), an idealized finite-horizon model for adapting a skill-guided executor under nonstationary execution conditions. SWOP keeps the executor anchored to a reference prior while using recent evidence to update which actions appear reliable in the current environment. \looseness=-1

Following nonstationary episodic MDP formulations~\cite{cheung2020nonstationarymdp,mao2021nonstationaryepisodic,feng2023nonstationarygfa}, we model environment evolution as a finite-horizon information-state MDP sequence:
\begin{equation}
\begin{aligned}
\mathcal{M}^{(k)}
= \big(&\Xi,\{\Omega(\xi)\}_{\xi\in\Xi},
\{r^{(k)}_h\}_{h=1}^{H}, \\
&\{P^{(k)}_h\}_{h=1}^{H}, H\big),
\quad k=1,\ldots,K .
\end{aligned}
\label{eq:idealized_mdp_sequence}
\end{equation}
Here $\Xi$ is a finite information-state space, $\Omega(\xi)$ is the state-dependent action domain, and evolution is captured by changes in the rewards $r^{(k)}_h$ and transitions $P^{(k)}_h$.

At each episode, SWOP estimates rewards and transitions from a recent sliding window, adds optimism for uncertainty, and computes a KL-regularized policy around a reference prior $\pi_{\mathrm{ref}}$, combining sliding-window optimism and regularized MDP ideas~\cite{cheung2020nonstationarymdp,mao2021nonstationaryepisodic,geist2019regularizedmdp}:
\begin{equation}
\widehat{\pi}^k_h(\omega\mid \xi)
\propto
\pi_{\mathrm{ref}}(\omega\mid \xi)
\exp(\widehat{Q}^k_h(\xi,\omega)/\alpha).
\end{equation}
The prior represents the preference induced by the fixed skill library and base executor. The sliding-window estimate allows the policy to respond to recent execution evidence.

Following variation-budget dynamic-regret analyses in nonstationary bandits and MDPs~\cite{besbes2014nonstationarybandits,cheung2020nonstationarymdp,mao2021nonstationaryepisodic}, we measure nonstationarity by the adjacent reward and transition variation:
\begin{align}
    \Delta^{(k)}_r
    &:=
    \max_{h,\xi,\omega}
    \big| r^{(k+1)}_h(\xi,\omega)
    - r^{(k)}_h(\xi,\omega) \big|, \\
    \Delta^{(k)}_p
    &:=
    \max_{h,\xi,\omega}
    \big\| P^{(k+1)}_h(\cdot \mid \xi,\omega)
    - P^{(k)}_h(\cdot \mid \xi,\omega) \big\|_1 ,
\end{align}
and aggregate it as:
\begin{equation}
    V_K :=
    \sum_{k=1}^{K-1}\Delta^{(k)}_r
    + H\sum_{k=1}^{K-1}\Delta^{(k)}_p .
    \label{eq:main_variation_budget}
\end{equation}
SWOP is evaluated by dynamic regret against the episode-wise optimal KL-regularized policy:
\begin{equation}
    \mathrm{DynReg}(K)
    :=
    \sum_{k=1}^{K}
    \left(
        V^{\star,(k)}_1(\xi^k_1)
        -
        V^{\widehat{\pi}^{k},(k)}_1(\xi^k_1)
    \right).
    \label{eq:main_dynamic_regret}
\end{equation}

\begin{theorem}[Informal dynamic-regret guarantee]
\label{thm:swop-dynreg}
Consider the finite-state SWOP abstraction with bounded rewards,
fixed episode-invariant action domains, a full-support reference prior,
and the sampling and concentration conditions stated in Appendix~\ref{app:problem-setup}.
Let $\hat{\pi}^k$ be the policy produced by SWOP at episode $k$.
With an oracle-tuned fixed window chosen before the run as a deterministic
function of the variation budget $V_K$, SWOP satisfies, with probability
at least $1-\delta$,
\begin{equation}
    \mathrm{DynReg}(K)
    \le \widetilde{O}\!\left(K^{2/3}(V_K+1)^{1/3}\right),
\end{equation}
where fixed problem-dependent factors are suppressed.
Consequently, when $H$, $|\Xi|$, and $\max_{\xi\in\Xi}|\Omega(\xi)|$ are fixed
and $V_K=o(K)$, SWOP obtains sublinear dynamic regret.
\end{theorem}

This guarantee suggests that, when environment variation grows sublinearly, prior-guided adaptation can track the episode-wise comparator by combining learned priors with recent feasibility evidence. Appendix~\ref{app:problem-setup} provides the formal SWOP statement and proof.

\subsection{From SWOP to Practical Executor Design}
\label{sec:swop_to_mpcr}

SWOP is an analysis model, while deployed web agents operate over language-conditioned observations, browser states, page elements, and candidate actions rather than explicit MDP models. We therefore use SWOP as a design guide for test-time action selection. It suggests adapting with recent execution evidence, retaining uncertainty-aware alternatives when prior-preferred actions become unreliable, and regularizing selection toward learned skill preferences. These principles translate into a practical executor objective: selecting among primitive and skill-action candidates at test time based on current-page feasibility evidence. Appendix~\ref{sec:mpcr_swop_correspondence} states this correspondence componentwise and reports the empirical evidence supporting each of its components. \looseness=-1

\subsection{MPCR as Feasibility-Aware Candidate Reranking}
\label{sec:mpcr}

\begin{figure}[t]
    \centering
    \includegraphics[width=\linewidth]{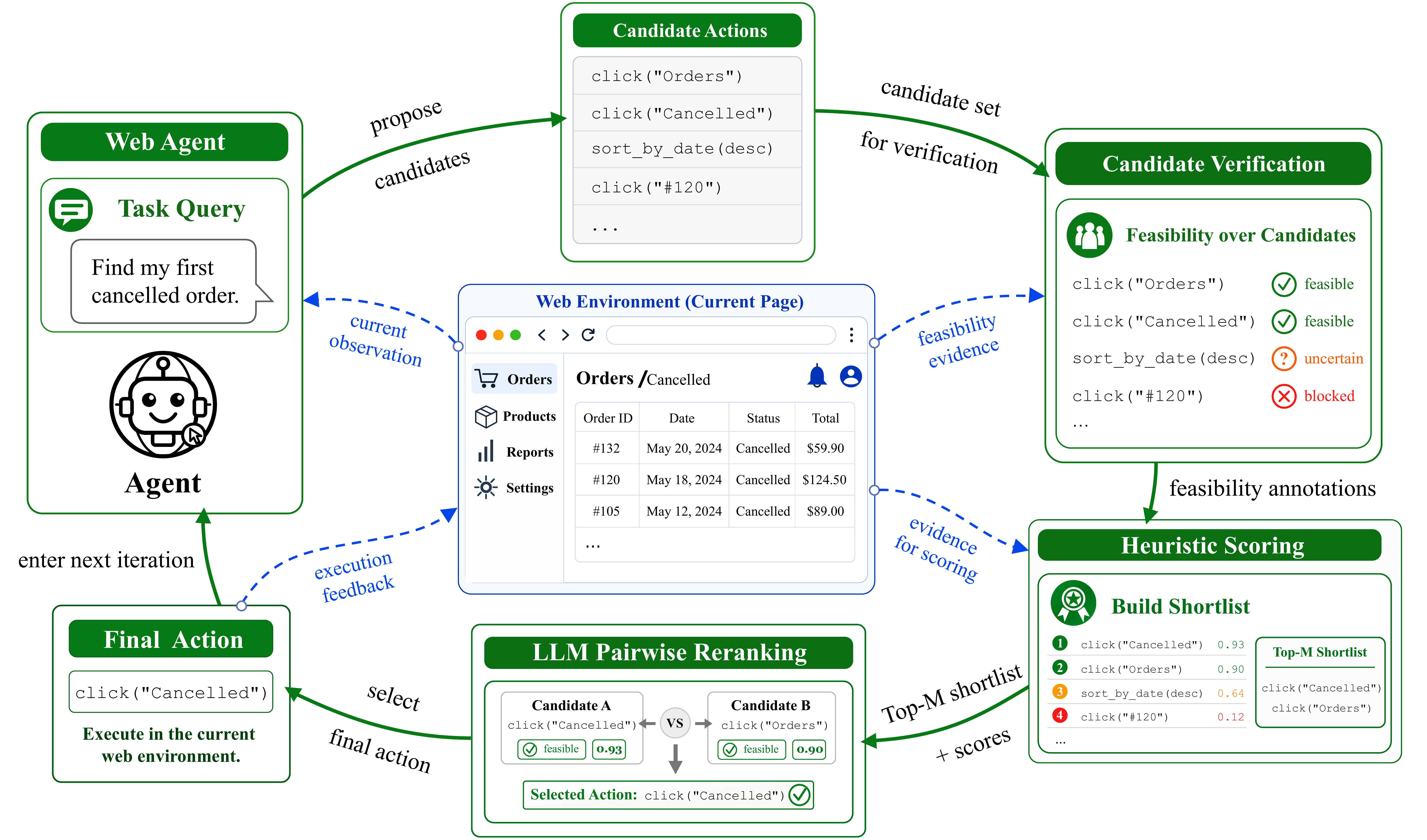}
    \caption{Execution flow of MPCR within the web-agent loop.}
    \label{fig:mpcr-executor}
\end{figure}

MPCR (Masked Prior-Guided Candidate Reranking) instantiates the above principles as an executor-side adaptation method. Figure~\ref{fig:mpcr-executor} shows the execution flow.

Given the current observation, task context, fixed skill library, and base executor, MPCR proposes primitive and skill-action candidates, verifies them with current-page feasibility evidence, scores them according to executability, and ranks the shortlist through LLM-based comparison or reranking~\cite{sun2023chatgptsearch,qin2024prp,wu2025pairwisereranking}.

MPCR selects only from proposed candidates and introduces no new skills. Execution feedback is carried into the next decision step. This realizes the SWOP principle in practical web execution while preserving the fixed learned skill library. Appendix~\ref{sec:appendix_theory_method_mapping} gives the full MPCR algorithm and implementation details.
\section{Experiments and Results}\label{sec:experiments}

We evaluate SkillShift-Bench and MPCR through three research questions: 
\textbf{RQ1} studies how learned-skill methods fail under controlled environment evolution; 
\textbf{RQ2} examines whether the observed adaptation dynamics are consistent with SWOP; 
and \textbf{RQ3} tests whether test-time executor-side adaptation improves fixed-library transfer.

\begin{table*}[t]
\centering
\small
\setlength{\tabcolsep}{6.5pt}
\renewcommand{\arraystretch}{1.1}

\begin{tabular}{cccccccccccc}
\toprule
\multirow{2}{*}[-1.75ex]{\textbf{Metric}}
& \multirow{2}{*}[-1.75ex]{\textbf{Method}}
& \multicolumn{3}{c}{\textbf{Website}}
& \multicolumn{6}{c}{\textbf{Environment}}
& \multirow{2}{*}[-1.75ex]{\textbf{O/A}} \\
\cmidrule(lr){3-5}
\cmidrule(lr){6-11}
&
& \makecell{\textbf{Git-}\\\textbf{Lab}}
& \makecell{\textbf{Ma-}\\\textbf{gento}}
& \makecell{\textbf{Word-}\\\textbf{Press}}
& \textbf{Src}
& \makecell{\textbf{Src-}\\\textbf{Perc}}
& \makecell{\textbf{Src-}\\\textbf{Exec}}
& \textbf{Tgt}
& \makecell{\textbf{Tgt-}\\\textbf{Perc}}
& \makecell{\textbf{Tgt-}\\\textbf{Exec}}
& \\
\midrule

\multirow{4}{*}{\makecell{SR \\(\%)}}
& \cellcolor{successbg} AWM
& \cellcolor{successbg} 17.71 & \cellcolor{successbg} 25.49 & \cellcolor{successbg} \textbf{62.66}
& \cellcolor{successbg} 50.07 & \cellcolor{successbg} 43.35 & \cellcolor{successbg} 22.45
& \cellcolor{successbg} 41.28 & \cellcolor{successbg} \textbf{37.65} & \cellcolor{successbg} 16.90
& \cellcolor{successbg} 35.29 \\
& \cellcolor{stepsbg} ASI
& \cellcolor{stepsbg} 20.81 & \cellcolor{stepsbg} 27.78 & \cellcolor{stepsbg} 60.27
& \cellcolor{stepsbg} \textbf{53.65} & \cellcolor{stepsbg} \textbf{48.80} & \cellcolor{stepsbg} 24.43
& \cellcolor{stepsbg} \textbf{41.84} & \cellcolor{stepsbg} 33.19 & \cellcolor{stepsbg} 15.79
& \cellcolor{stepsbg} \textbf{36.29} \\
& \cellcolor{callbg} SkillWeaver
& \cellcolor{callbg} 13.78 & \cellcolor{callbg} 14.37 & \cellcolor{callbg} 32.29
& \cellcolor{callbg} 35.13 & \cellcolor{callbg} 28.75 & \cellcolor{callbg} 15.63
& \cellcolor{callbg} 17.80 & \cellcolor{callbg} 13.85 & \cellcolor{callbg} 9.71
& \cellcolor{callbg} 20.15 \\
& \cellcolor{sfrbg} WALT
& \cellcolor{sfrbg} \textbf{21.07} & \cellcolor{sfrbg} \textbf{27.95} & \cellcolor{sfrbg} 55.38
& \cellcolor{sfrbg} 46.82 & \cellcolor{sfrbg} 43.53 & \cellcolor{sfrbg} \textbf{25.46}
& \cellcolor{sfrbg} 41.17 & \cellcolor{sfrbg} 32.74 & \cellcolor{sfrbg} \textbf{19.09}
& \cellcolor{sfrbg} 34.80 \\

\midrule

\multirow{4}{*}{\makecell{AS}}
& \cellcolor{successbg} AWM
& \cellcolor{successbg} 8.03 & \cellcolor{successbg} 8.22 & \cellcolor{successbg} 4.76
& \cellcolor{successbg} 6.04 & \cellcolor{successbg} 6.32 & \cellcolor{successbg} 8.31
& \cellcolor{successbg} 6.34 & \cellcolor{successbg} \textbf{6.59} & \cellcolor{successbg} 8.44
& \cellcolor{successbg} 7.01 \\
& \cellcolor{stepsbg} ASI
& \cellcolor{stepsbg} \textbf{7.87} & \cellcolor{stepsbg} \textbf{8.06} & \cellcolor{stepsbg} \textbf{4.73}
& \cellcolor{stepsbg} \textbf{5.90} & \cellcolor{stepsbg} \textbf{5.95} & \cellcolor{stepsbg} \textbf{8.23}
& \cellcolor{stepsbg} \textbf{6.29} & \cellcolor{stepsbg} 6.70 & \cellcolor{stepsbg} \textbf{8.24}
& \cellcolor{stepsbg} \textbf{6.89} \\
& \cellcolor{callbg} SkillWeaver
& \cellcolor{callbg} 14.79 & \cellcolor{callbg} 17.96 & \cellcolor{callbg} 12.63
& \cellcolor{callbg} 11.34 & \cellcolor{callbg} 13.48 & \cellcolor{callbg} 19.82
& \cellcolor{callbg} 12.59 & \cellcolor{callbg} 14.74 & \cellcolor{callbg} 18.80
& \cellcolor{callbg} 15.13 \\
& \cellcolor{sfrbg} WALT
& \cellcolor{sfrbg} 12.39 & \cellcolor{sfrbg} 11.46 & \cellcolor{sfrbg} 8.53
& \cellcolor{sfrbg} 8.77 & \cellcolor{sfrbg} 9.92 & \cellcolor{sfrbg} 12.09
& \cellcolor{sfrbg} 10.66 & \cellcolor{sfrbg} 11.36 & \cellcolor{sfrbg} 11.96
& \cellcolor{sfrbg} 10.79 \\

\midrule

\multirow{4}{*}{\makecell{ASC}}
& \cellcolor{successbg} AWM
& \cellcolor{successbg} -- & \cellcolor{successbg} -- & \cellcolor{successbg} --
& \cellcolor{successbg} -- & \cellcolor{successbg} -- & \cellcolor{successbg} --
& \cellcolor{successbg} -- & \cellcolor{successbg} -- & \cellcolor{successbg} --
& \cellcolor{successbg} -- \\
& \cellcolor{stepsbg} ASI
& \cellcolor{stepsbg} 0.12 & \cellcolor{stepsbg} 0.19 & \cellcolor{stepsbg} 0.17
& \cellcolor{stepsbg} 0.18 & \cellcolor{stepsbg} 0.26 & \cellcolor{stepsbg} 0.29
& \cellcolor{stepsbg} 0.10 & \cellcolor{stepsbg} 0.10 & \cellcolor{stepsbg} 0.05
& \cellcolor{stepsbg} 0.16 \\
& \cellcolor{callbg} SkillWeaver
& \cellcolor{callbg} 0.78 & \cellcolor{callbg} 0.63 & \cellcolor{callbg} 2.64
& \cellcolor{callbg} 1.13 & \cellcolor{callbg} 1.16 & \cellcolor{callbg} 1.14
& \cellcolor{callbg} 1.45 & \cellcolor{callbg} 1.50 & \cellcolor{callbg} 1.73
& \cellcolor{callbg} 1.35 \\
& \cellcolor{sfrbg} WALT
& \cellcolor{sfrbg} 0.07 & \cellcolor{sfrbg} 0.19 & \cellcolor{sfrbg} 0.00
& \cellcolor{sfrbg} 0.05 & \cellcolor{sfrbg} 0.05 & \cellcolor{sfrbg} 0.21
& \cellcolor{sfrbg} 0.08 & \cellcolor{sfrbg} 0.03 & \cellcolor{sfrbg} 0.11
& \cellcolor{sfrbg} 0.09 \\

\midrule

\multirow{4}{*}{\makecell{SFR \\(\%)}}
& \cellcolor{successbg} AWM
& \cellcolor{successbg} -- & \cellcolor{successbg} -- & \cellcolor{successbg} --
& \cellcolor{successbg} -- & \cellcolor{successbg} -- & \cellcolor{successbg} --
& \cellcolor{successbg} -- & \cellcolor{successbg} -- & \cellcolor{successbg} --
& \cellcolor{successbg} -- \\
& \cellcolor{stepsbg} ASI
& \cellcolor{stepsbg} 19.55 & \cellcolor{stepsbg} 48.27 & \cellcolor{stepsbg} 16.26
& \cellcolor{stepsbg} 9.83 & \cellcolor{stepsbg} 24.85 & \cellcolor{stepsbg} 55.20
& \cellcolor{stepsbg} 13.53 & \cellcolor{stepsbg} 28.77 & \cellcolor{stepsbg} 35.98
& \cellcolor{stepsbg} 28.03 \\
& \cellcolor{callbg} SkillWeaver
& \cellcolor{callbg} \textbf{4.61} & \cellcolor{callbg} \textbf{1.11} & \cellcolor{callbg} 0.21
& \cellcolor{callbg} \textbf{0.00} & \cellcolor{callbg} \textbf{0.41} & \cellcolor{callbg} \textbf{0.00}
& \cellcolor{callbg} \textbf{2.03} & \cellcolor{callbg} \textbf{4.24} & \cellcolor{callbg} \textbf{5.17}
& \cellcolor{callbg} \textbf{1.97} \\
& \cellcolor{sfrbg} WALT
& \cellcolor{sfrbg} 48.52 & \cellcolor{sfrbg} 23.88 & \cellcolor{sfrbg} \textbf{0.00}
& \cellcolor{sfrbg} 10.30 & \cellcolor{sfrbg} 23.29 & \cellcolor{sfrbg} 31.50
& \cellcolor{sfrbg} 23.28 & \cellcolor{sfrbg} 18.64 & \cellcolor{sfrbg} 37.76
& \cellcolor{sfrbg} 24.13 \\

\bottomrule
\end{tabular}
\caption{
Transfer diagnostics across websites and environment settings. 
O/A denotes the overall average across the corresponding settings. 
Bold values mark the highest SR and the lowest AS / SFR entries.
ASC is reported as a non-directional diagnostic.
``--'' indicates metrics not applicable to AWM because it uses prompted textual workflows rather than executable skill invocations~\cite{wang2025agentworkflowmemory}.
}
\label{tab:transfer_aggregate_expanded}
\end{table*}

\subsection{Experimental Setup}\label{sec:experimental_setup}

\paragraph{Benchmark and metrics.}
We evaluate all methods under the six SkillShift-Bench conditions defined in Section~\ref{sec:benchmark_construction}: 
\texttt{Src}, \texttt{Src-Perc}, \texttt{Src-Exec}, \texttt{Tgt}, \texttt{Tgt-Perc}, and \texttt{Tgt-Exec}.
For compact analysis, we also report three pooled views:
\emph{Base} combines \texttt{Src} and \texttt{Tgt}, 
\emph{Perceptual} combines \texttt{Src-Perc} and \texttt{Tgt-Perc}, and 
\emph{Execution} combines \texttt{Src-Exec} and \texttt{Tgt-Exec}.
We use the four diagnostics defined in Section~\ref{sec:evaluation_diagnostics}: SR, AS, ASC, and SFR.

\paragraph{Compared methods.}
We compare four skill learning methods: AWM, ASI, SkillWeaver, and WALT~\cite{wang2025agentworkflowmemory,wang2025inducingprogrammaticskills,zheng2025skillweaver,prabhu2026walt}.
Each method learns skills in \texttt{Src} using its native protocol and is evaluated under the aligned transfer conditions.
Method cards and library sizes are provided in Appendix~\ref{sec:appendix_method_cards} and Appendix~\ref{sec:appendix_learning_budgets}.

\paragraph{Fixed-library protocol.}
For RQ2 and RQ3, ASI first learns a skill library on \texttt{Src}.
The library is then frozen and shared by ASI and MPCR.
ASI uses its original selector, while MPCR changes only test-time selection over the same skills and primitive actions.
This protocol isolates executor-side adaptation from additional skill learning.

\paragraph{Models and execution.}
The benchmark-level analysis uses DeepSeek-V3.2, GPT-5 mini, and Claude Haiku 4.5~\cite{deepseek2025v32,openai2025gpt5mini,anthropic2025haiku45}.
The fixed-library ASI and MPCR comparison uses GPT-5 mini.
MPCR considers at most 2 candidate actions per step and keeps a top-2 shortlist, evaluating it as a lightweight test-time reranking method rather than a broad-search procedure. Appendix~\ref{sec:app_budget_selection} sweeps this candidate budget and reports the empirical basis for the default. Appendix~\ref{sec:appendix_site_backbone} further reports backbone-level diagnostic results.

\subsection{RQ1: How Do Existing Learned-Skill Methods Fail under Controlled Environment Evolution?}
\label{sec:RQ1}

\begin{figure}[t]
    \centering
    \includegraphics[width=\linewidth]{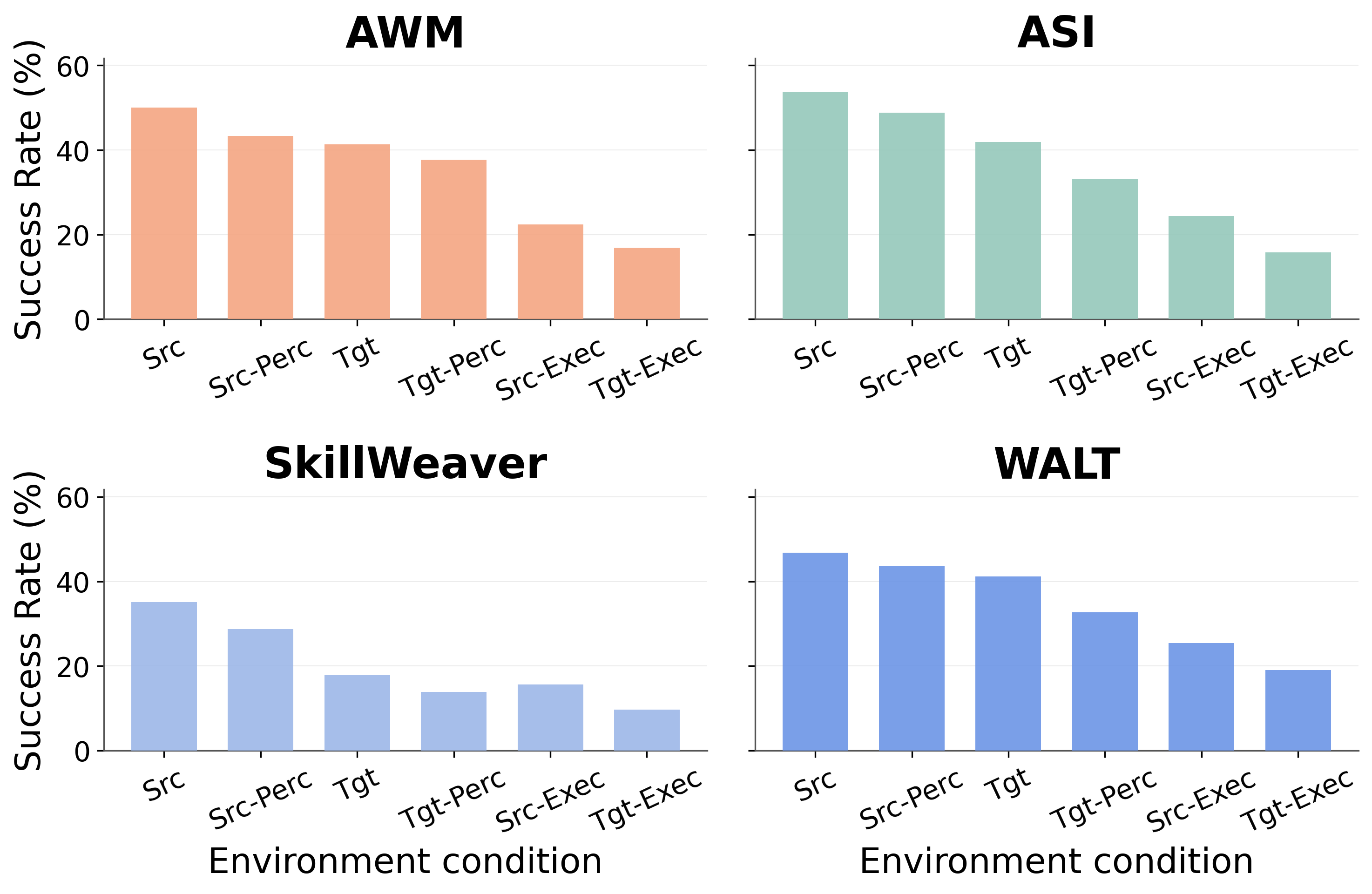}
    \caption{Robustness profiles under controlled environment evolution.}
    \label{fig:transfer_grouped_bar}
\end{figure}

\begin{figure*}[t]
    \centering
    \includegraphics[width=0.9\linewidth]{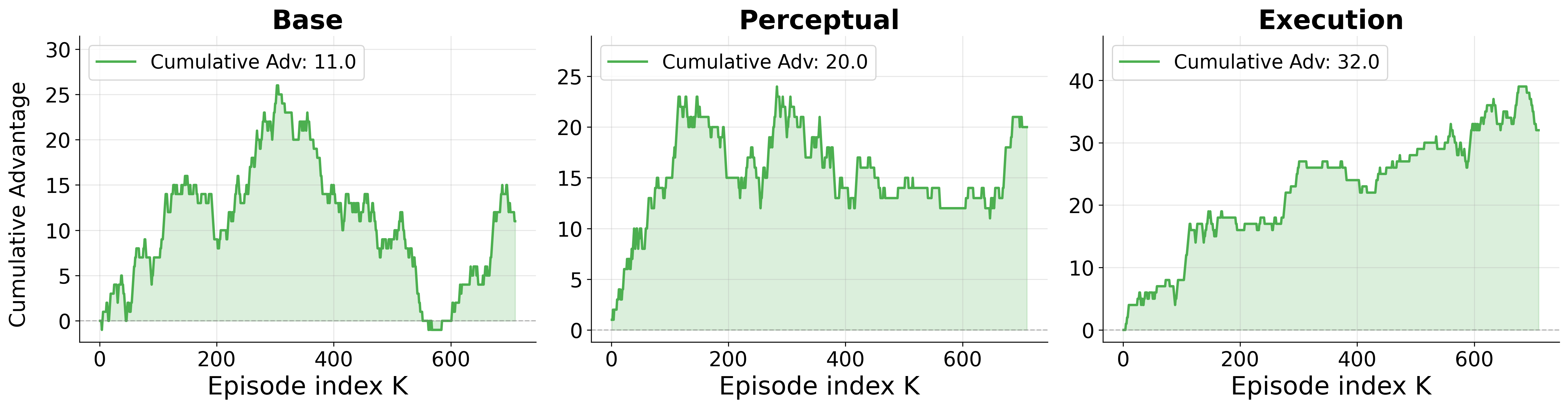}
    \caption{Cumulative task-failure advantage of MPCR over ASI.}
    \label{fig:cumulative_advantage}
\end{figure*}

Figure~\ref{fig:transfer_grouped_bar} and Table~\ref{tab:transfer_aggregate_expanded} show a consistent decline from clean source execution to evolved and shifted environments across all evaluated methods. For example, ASI decreases from 53.65\% on \texttt{Src} to 41.84\% on \texttt{Tgt}, and to 15.79\% on \texttt{Tgt-Exec}. Similar trends appear for AWM, SkillWeaver, and WALT. Under the aligned transfer setting, these drops indicate that skill reuse is sensitive to local execution conditions rather than only to task-level relevance.

The pooled views separate the pressure introduced by different shifts. 
Success decreases from 40.97\% in the \emph{Base} view to 35.23\% in \emph{Perceptual} and 18.69\% in \emph{Execution}. 
Trajectory cost also increases across the same views, with average steps rising from 8.49 to 9.38 and 11.99, while skill failure rate rises from 9.83\% to 16.70\% and 27.60\%. This pattern shows that environment evolution affects both task completion and the execution process, with runtime feasibility shifts creating the strongest pressure.

The diagnostics also reveal complementary failure channels. AWM is evaluated through workflow reuse, so explicit skill-call diagnostics are not applicable. ASI and WALT expose execution pressure through skill failure rate, while SkillWeaver often reflects it through longer trajectories and more skill calls. These patterns motivate reporting process-level diagnostics in addition to final success.

\subsection{RQ2: Are the Observed Adaptation Dynamics Consistent with the SWOP Theoretical Account?}
\label{sec:RQ2}

\begin{figure*}[t]
    \centering
    \includegraphics[width=0.9\linewidth]{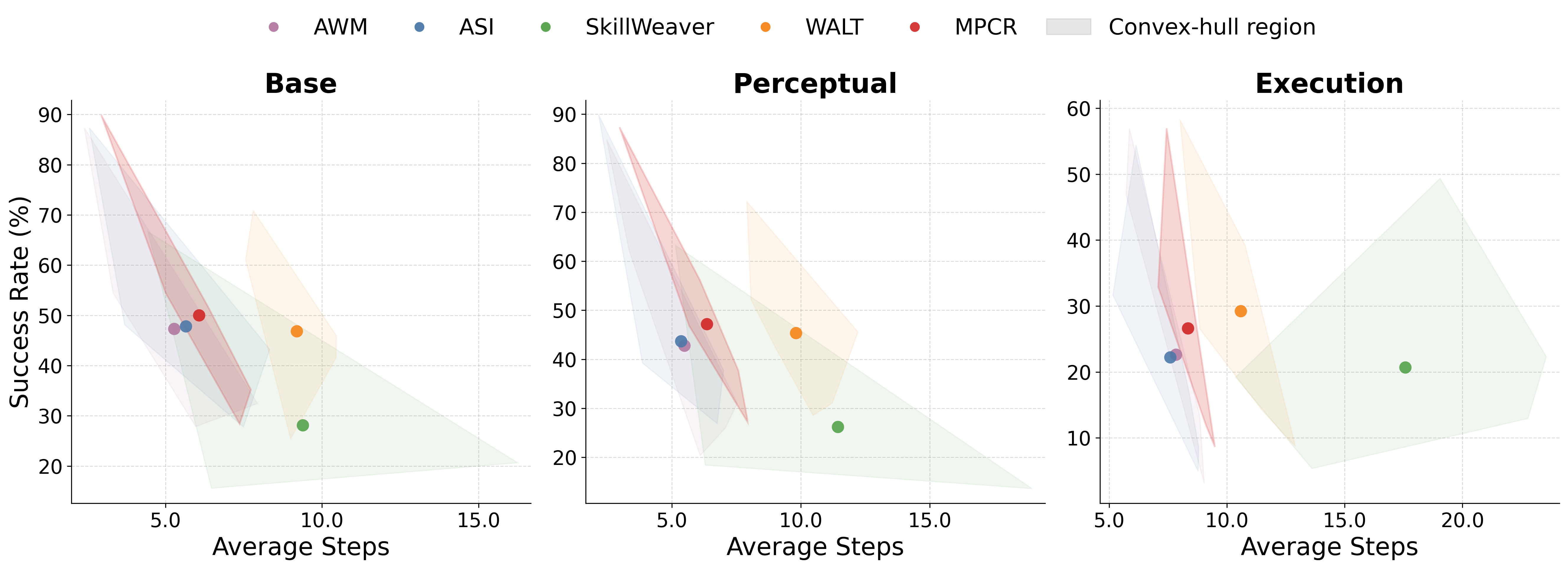}
    \caption{Success and trajectory cost for fixed-library executor adaptation. Points denote method centers across environment condition groups, and shaded hulls show configuration-level variation.}
    \label{fig:success_steps_hull_centers}
\end{figure*}

SWOP models shifted skill execution as nonstationary local executability, where an executor stays close to learned priors while using recent feasibility evidence. Since the theoretical per-step loss is not directly observable in web interaction, we use final failure as a task-level proxy.

For method \(m\) and task \(i\), let \(Y_{m,i}\in\{0,1\}\) be the automatic success label and \(L_{m,i}=1-Y_{m,i}\) the proxy failure loss. We define the cumulative advantage of MPCR over ASI as \(A(K)=\sum_{i=1}^{K}(L_{\mathrm{ASI},i}-L_{\mathrm{MPCR},i})\), where positive values indicate fewer accumulated failures under the same fixed skill library.

Figure~\ref{fig:cumulative_advantage} shows positive cumulative advantage in all condition groups. The final advantage is 11.0 under \emph{Base}, 20.0 under \emph{Perceptual}, and 32.0 under \emph{Execution}. This ordering follows the executability pressure observed in RQ1. When the learned prior remains reliable, the room for correction is smaller. When grounding evidence changes or runtime blockers appear, current feasibility evidence becomes more useful for action selection.

These dynamics are consistent with the SWOP mechanism. In practical web execution, MPCR follows the same prior-guided adaptation principle by conditioning test-time action selection on current execution evidence. The stronger advantage under shifted conditions suggests that part of the transfer loss reflects current executability mismatch, which can be partly mitigated through executor adaptation. \looseness=-1

\subsection{RQ3: Can Test-Time Executor-Side Adaptation Improve Transfer under a Fixed Skill Library?}
\label{sec:RQ3}

Table~\ref{tab:mpcr_gpt5mini_asi} shows that MPCR improves ASI in all six environments. Average success increases from 37.90\% to 41.26\%, corresponding to an 8.87\% relative improvement.

\begin{table}[t]
\centering
\small
\setlength{\tabcolsep}{6.5pt}
\renewcommand{\arraystretch}{1.1}
\begin{tabular}{lcccc}
\toprule
\multirow{2}{*}[-1.0ex]{\makecell{\textbf{Environ-}\\\textbf{ment}}}
& \multicolumn{2}{c}{\textbf{MPCR}} 
& \multicolumn{2}{c}{\textbf{ASI}} \\
\cmidrule(lr){2-3} \cmidrule(lr){4-5}
& \makecell{\textbf{SR (\%)}}
& \makecell{\textbf{ASC}}
& \makecell{\textbf{SR (\%)}}
& \makecell{\textbf{ASC}} \\
\midrule
Src 
& 56.67 {\color{green!50!black}\scriptsize(+7.35\%)}
& 0.22
& 52.79 & 0.05 \\

Src-Perc
& 57.09 {\color{green!50!black}\scriptsize(+7.68\%)}
& 0.20
& 53.02 & 0.13 \\

Src-Exec 
& 33.23 {\color{green!50!black}\scriptsize(+24.55\%)}
& 0.19
& 26.68 & 0.07 \\

Tgt 
& 43.32 {\color{green!50!black}\scriptsize(+1.14\%)}
& 0.12
& 42.83 & 0.04 \\

Tgt-Perc 
& 37.24 {\color{green!50!black}\scriptsize(+8.51\%)}
& 0.14
& 34.32 & 0.05 \\

Tgt-Exec 
& 19.99 {\color{green!50!black}\scriptsize(+12.62\%)}
& 0.03
& 17.75 & 0.02 \\

\midrule
\textbf{Average}
& \textbf{41.26} {\color{green!50!black}\scriptsize(+8.87\%)}
& \textbf{0.15}
& \textbf{37.90}
& \textbf{0.06} \\
\bottomrule
\end{tabular}
\caption{MPCR versus ASI across transfer conditions. Relative SR gains are computed against ASI.}
\label{tab:mpcr_gpt5mini_asi}
\end{table}

The gains are clearest when local execution becomes more fragile. MPCR improves \texttt{Src-Exec} from 26.68\% to 33.23\% and \texttt{Tgt-Exec} from 17.75\% to 19.99\%, with relative gains of 24.55\% and 12.62\%. It also maintains a positive gain in clean target transfer, where the learned prior is already more aligned with the current environment. This pressure-dependent pattern supports the local executability framing: executor adaptation is most useful when the learned skill remains relevant but direct reuse becomes less reliable.

Because ASI and MPCR share the same learned library, these improvements are attributable to test-time selection rather than additional skill acquisition. MPCR increases average skill calls from 0.06 to 0.15, while remaining well below one skill call per task. This indicates evidence-conditioned activation of existing skills rather than broad reuse.

Figure~\ref{fig:success_steps_hull_centers} further shows that the improvement is not simply obtained by longer trajectories. Across \emph{Base}, \emph{Perceptual}, and \emph{Execution}, MPCR improves success over ASI while staying in a comparable trajectory-cost regime. This suggests that feasibility-aware selection improves fixed-library transfer behavior by recovering part of the loss caused by local executability mismatch.

Supplementary analyses support this attribution. Appendix~\ref{sec:app_component_ablation} ablates the executor stages and shows, through a randomized control that preserves the pipeline while emitting no reranking call, that the gain follows from the content of the feasibility-conditioned decision rather than from an extra model call per step. Appendix~\ref{sec:app_awm_generalization} reproduces the same pressure-dependent pattern on a second learned library, Appendix~\ref{sec:app_seed_stability} reports cross-seed dispersion indicating that the gain is stable, and Appendix~\ref{sec:app_relearning} contrasts this test-time route with relearning the library in the target environment.
\section{Conclusion}\label{sec:conclusion}

We introduced SkillShift-Bench, a controlled benchmark for evaluating learned-skill executability under same-site environment evolution. By preserving task semantics while varying contextual, perceptual, and execution conditions, it exposes cases where task-relevant skills become difficult to invoke, ground, or execute. We further formalized this skill-environment mismatch with SWOP and instantiated it as MPCR, which improves fixed-library transfer through feasibility-aware reranking. These results suggest that robust web agents should complement reusable skills with test-time adaptation to current feasibility evidence under evolving environments. \looseness=-1

\section*{Limitations}

This work is designed as a controlled diagnostic study. SkillShift-Bench focuses on self-hosted same-site evolution over representative web systems, which supports reproducible execution and reliable failure attribution but leaves broader website coverage to future work. Our experiments mainly use relatively capable LLM-based web agents to reduce confounding from basic instruction-following or tool-use failures, and future studies can further examine smaller or more resource-constrained models. The current benchmark also focuses on English desktop web interactions, with multilingual, mobile, and accessibility-oriented settings as natural extensions.

\section*{Ethics Statement}

SkillShift-Bench is evaluated in controlled, self-hosted, and authorized web environments using scripted tasks, synthetic or benchmark-specific accounts, seeded task data, and automatic grading. It does not interact with third-party live services, real user data, or private user accounts. Before release, we check task instructions, seeded entities, account names, and grader-visible content to avoid personally identifying information and offensive content, and replace any accidental cases with synthetic placeholders. The perturbations are designed to study environment evolution and execution robustness, not to bypass access controls, evade platform policies, or automate unauthorized activity. We use third-party systems, baseline methods, and model APIs under their released licenses, documentation, or provider terms. Released artifacts will include only benchmark-specific code, task definitions, perturbation scripts, and metadata that we are permitted to distribute, and are intended for reproducible robustness analysis and research evaluation under controlled-use settings.

\section*{Acknowledgment}

This research was supported by the Jiangsu Science Foundation (BK20243012, BG2024036), Natural Science Foundation of China (624B2068, 62576162), and the Fundamental Research Funds for the Central Universities (022114380023).

\bibliography{refs/references}

\onecolumn
\appendix

\section{Comparison with Existing Web Agent Benchmarks}
\label{app:benchmark-comparison}

To situate SkillShift-Bench within the broader design space of web-agent evaluation, Table~\ref{tab:benchmark-comparison} summarizes key properties of representative benchmarks and evaluation frameworks.

\paragraph{Evaluation criteria.}
\textbf{Variation Diversity} measures whether a benchmark covers multiple qualitatively distinct types of environment variation. \textbf{Variation Controllability} measures whether the benchmark supports controlled and reproducible environment perturbations. \textbf{Evaluation Reproducibility} measures whether tasks, states, perturbations, and grading results can be stably reproduced. \textbf{Website Fidelity} measures whether the benchmark preserves realistic web interfaces, interaction structures, page complexity, and long-horizon workflows. \textbf{Diagnostic Granularity} measures whether the benchmark supports evaluation through complementary metrics across multiple dimensions. \textbf{Experience Transferability} measures whether the benchmark evaluates how robustly reusable experience or procedures acquired by an agent remain effective under environment changes. \textbf{Failure Attributability} measures whether the benchmark supports attributing agent failures to specific components of the evaluation process.

\paragraph{Static and live web benchmarks.}
Mind2Web, WebArena, VisualWebArena, WorkArena, WorkArena++, Online-Mind2Web, WebCanvas, and REAL provide important foundations for web-agent evaluation~\cite{deng2023mind2web,zhou2024webarena,koh2024visualwebarena,drouin2024workarena,boisvert2024workarenapp,xue2025illusionofprogress,pan2024webcanvas,garg2025real}.
They cover offline action prediction, executable self-hosted tasks, enterprise workflows, live web evaluation, and deterministic website replicas. These benchmarks are valuable for measuring general web-agent capability across realistic or executable web settings. However, they are not primarily designed to compare aligned variants under controlled environment evolution. As a result, they usually measure whether an agent can solve a task in a given environment, rather than whether experience acquired in one environment remains valid after the same task is transferred to a changed environment.

\paragraph{Robustness and evolution benchmarks.}
WAREX, TimeWarp, ProEvolve, and StressWeb are closer to SkillShift-Bench because they explicitly study environment change, perturbation, or reliability. WAREX injects website and network failures into existing benchmarks, making runtime instability visible~\cite{kara2025warex}. TimeWarp evaluates agents across temporally different interface versions~\cite{ishmam2026timewarp}. ProEvolve studies programmable evolution of agent environments~\cite{li2026proevolve}. StressWeb introduces controlled interaction variability for diagnostic robustness evaluation~\cite{bai2026stressweb}. These benchmarks show that environment variation can substantially affect web-agent performance. However, their main evaluation object is still general agent robustness under shifted, perturbed, or harder conditions. In other words, environment variation is primarily used as a test condition for end-to-end agent or multimodal-model capability.

\paragraph{Skill transfer as the evaluation object.}
SkillShift-Bench studies a different unit of evaluation. In our setting, the source environment is not merely a baseline condition; it is the context in which reusable skills are acquired, selected, or induced. The target and shifted environments then test whether the same learned skills remain executable under preserved task intent, target entities, functional support, and grading semantics. Therefore, the measured gap is not simply an end-to-end performance drop under a harder environment. It is a skill-transfer gap: a previously useful skill may remain semantically relevant while becoming locally non-executable because its invocation context, perceptual grounding evidence, or runtime preconditions no longer match the deployment environment. This setting is especially important for skill-based agents because skill reuse carries implicit assumptions about when the skill should be invoked, what interface evidence grounds its next action, and which local conditions must hold for execution~\cite{wang2025agentworkflowmemory,wang2025inducingprogrammaticskills,zheng2025skillweaver,prabhu2026walt}.

\begin{table*}[t]
\centering
\scriptsize
\renewcommand{\arraystretch}{2.5}
\setlength{\tabcolsep}{2pt}

\resizebox{\textwidth}{!}{
\begin{tabular}{lccccccc}
\toprule
\textbf{Benchmark} &
\makecell{\textbf{Variation}\\\textbf{Diversity}} &
\makecell{\textbf{Variation}\\\textbf{Controllability}} &
\makecell{\textbf{Evaluation}\\\textbf{Reproducibility}} &
\makecell{\textbf{Website}\\\textbf{Fidelity}} &
\makecell{\textbf{Diagnostic}\\\textbf{Granularity}} &
\makecell{\textbf{Experience}\\\textbf{Transferability}} &
\makecell{\textbf{Failure}\\\textbf{Attributability}} \\
\midrule
\textbf{Mind2Web}                  & \xmark & \xmark & \cmark & \tmark & \tmark & \xmark & \xmark \\
\textbf{WebArena}                  & \xmark & \xmark & \cmark & \cmark & \tmark & \xmark & \xmark \\
\makecell[l]{\textbf{Visual}\\\textbf{WebArena}}
                                    & \xmark & \xmark & \cmark & \cmark & \tmark & \xmark & \xmark \\
\makecell[l]{\textbf{WorkArena /}\\\textbf{WorkArena++}}
                                    & \xmark & \xmark & \tmark & \cmark & \tmark & \xmark & \xmark \\
\textbf{WebCanvas}                 & \tmark & \xmark & \tmark & \cmark & \cmark & \xmark & \tmark \\
\makecell[l]{\textbf{Online-}\\\textbf{Mind2Web}}
                                    & \xmark & \xmark & \tmark & \cmark & \tmark & \xmark & \xmark \\
\textbf{REAL}                      & \xmark & \xmark & \cmark & \cmark & \tmark & \xmark & \xmark \\
\textbf{WAREX}                     & \cmark & \cmark & \cmark & \tmark & \cmark & \xmark & \cmark \\
\textbf{TimeWarp}                  & \xmark & \cmark & \cmark & \cmark & \xmark & \tmark & \tmark \\
\textbf{ProEvolve}                 & \cmark & \cmark & \cmark & \xmark & \tmark & \tmark & \tmark \\
\textbf{StressWeb}                 & \cmark & \cmark & \cmark & \tmark & \cmark & \xmark & \cmark \\
\midrule
\makecell[l]{\textbf{SkillShift-}\\\textbf{Bench (Ours)}}  
                                    & \cmark & \cmark & \cmark & \cmark & \cmark & \cmark & \cmark \\
\bottomrule
\end{tabular}
}
\caption{Comparison of benchmarks across evaluation criteria. \cmark{} = explicitly supported as a central benchmark objective under the criterion definitions above, \tmark{} = partially supported or indirectly covered, and \xmark{} = not a primary target of the original benchmark design or not publicly verifiable.}
\label{tab:benchmark-comparison}
\end{table*}
\section{Benchmark Construction Details}
\label{app:benchmark_details}

This section describes how SkillShift-Bench is constructed to support controlled evaluation under environment evolution. We detail the selection of environments and source--target pairs, the validation of aligned task families, and the implementation of shifted variants.

\subsection{Environment Selection and Version Pairing}
\label{sec:app_b_env_pairing}

SkillShift-Bench selects environments and version pairs under two levels of constraints. At the website level, candidate systems must be mature open-source projects that support complete local self-hosted deployment. This requirement decouples evaluation from the public Internet and removes uncontrolled factors such as live service updates, network variability, third-party availability, and account-side changes. In addition, the targeted notion of environment evolution must primarily reflect structural changes in the front-end interaction layer rather than ordinary data-content updates. We therefore exclude websites whose evolution is mainly expressed as changing data content, such as map or knowledge-base growth, because such changes would entangle UI robustness with dynamic information retrieval.

At the version-pair level, the source and target variants must satisfy two further criteria. First, the version gap must contain substantial interface evolution, preferably a discontinuity in design language, navigation organization, workflow structure, or interaction pattern. Minor cosmetic changes are insufficient, because they do not create a meaningful test of learned-skill executability. Second, cross-version data migration must be feasible. For version-driven systems, we follow the official upgrade or migration path, including required intermediate upgrade points when applicable, so that the agent faces the same task-level semantic data state across versions~\cite{gitlabUpgradePaths,adobe2024datamigrationtool}. This means that task-critical entities, attributes, relations, and grader-visible outcomes are preserved even when the underlying database schema changes.

Applying these criteria, we instantiate SkillShift-Bench with three complementary environments: GitLab, Magento, and WordPress. They cover developer collaboration, e-commerce administration, and content management, while exposing different forms of interface evolution under reproducible self-hosted deployments.

\paragraph{GitLab.}
GitLab provides the developer-platform setting in SkillShift-Bench. The selected tasks cover repositories, issues, merge requests, labels, milestones, project settings, and user or group administration. These workflows require agents to operate inside nested project contexts and to rely on navigation paths, sidebars, management panels, and local action controls. As a mature open-source platform with self-hosted deployment support, GitLab satisfies our deployment criterion and can be reset in isolated containers. More importantly, its version evolution affects the interaction layer rather than only the underlying project data: project navigation, management pages, issue views, merge-request views, and settings workflows all undergo interface-level changes.

We use GitLab \textbf{v12.0.0}~\cite{gitlab2019v120release} as the source variant and GitLab \textbf{v16.0.6}~\cite{gitlab2023v1606release} as the target variant. This pair creates a realistic contextual-shift setting. Core developer operations remain available, but the routes, page contexts, menu hierarchy, and action locations needed to execute those operations may differ across versions~\cite{gitlab2023v160release}. To preserve task-level comparability, the deployment is migrated through a controlled upgrade procedure, so that projects, issues, labels, milestones, users, permissions, and other benchmark entities remain semantically aligned. Under this setup, a transfer failure is not explained by missing GitLab functionality or changed task data; it is instead evidence that a previously learned skill no longer matches the evolved execution context.

Figure~\ref{fig:app_b_gitlab_pair} shows the GitLab source-target pair on task-relevant project-management pages. The screenshots are included as construction evidence rather than as grading targets, illustrating the interface changes under which the same developer-platform operations must still be executed.

\begin{figure*}[htbp]
    \centering
    \begin{subfigure}{0.48\textwidth}
        \centering
        \includegraphics[width=\linewidth]{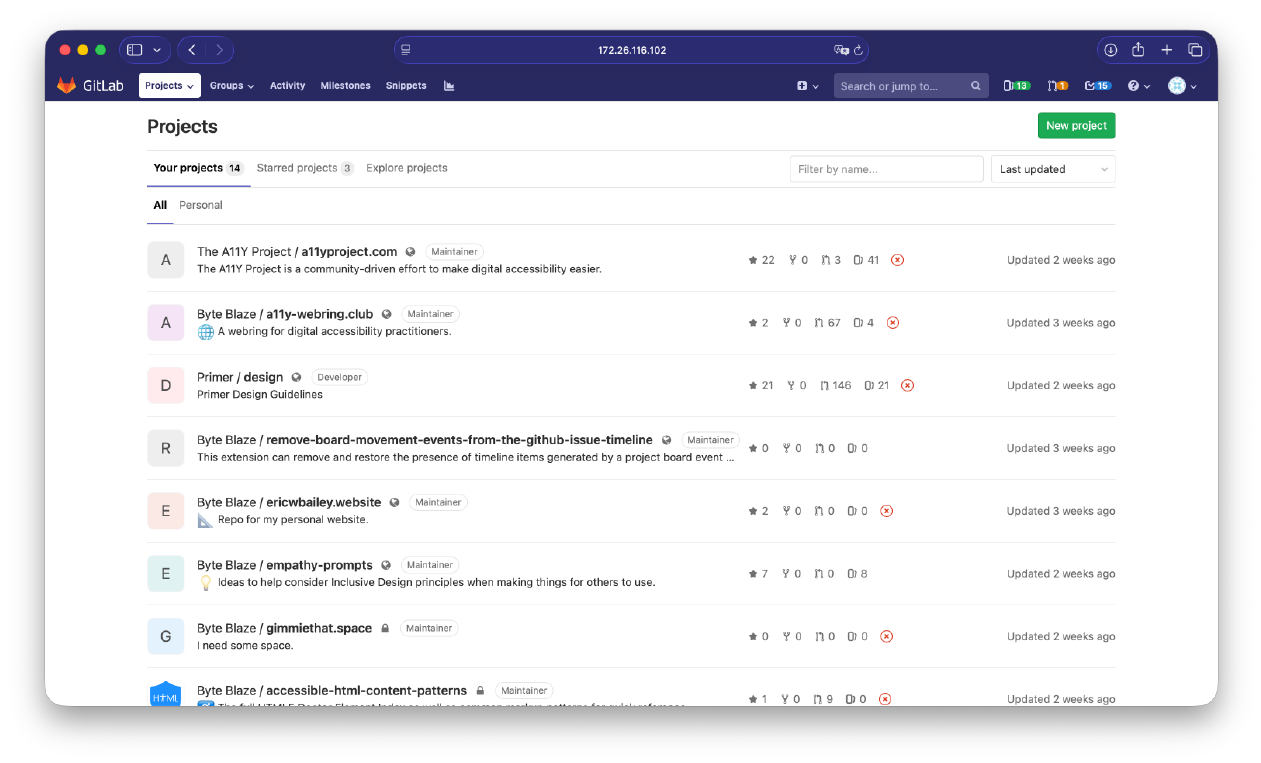}
        \caption{GitLab v12.0.0 source variant.}
        \label{fig:app_b_gitlab_source}
    \end{subfigure}
    \hfill
    \begin{subfigure}{0.48\textwidth}
        \centering
        \includegraphics[width=\linewidth]{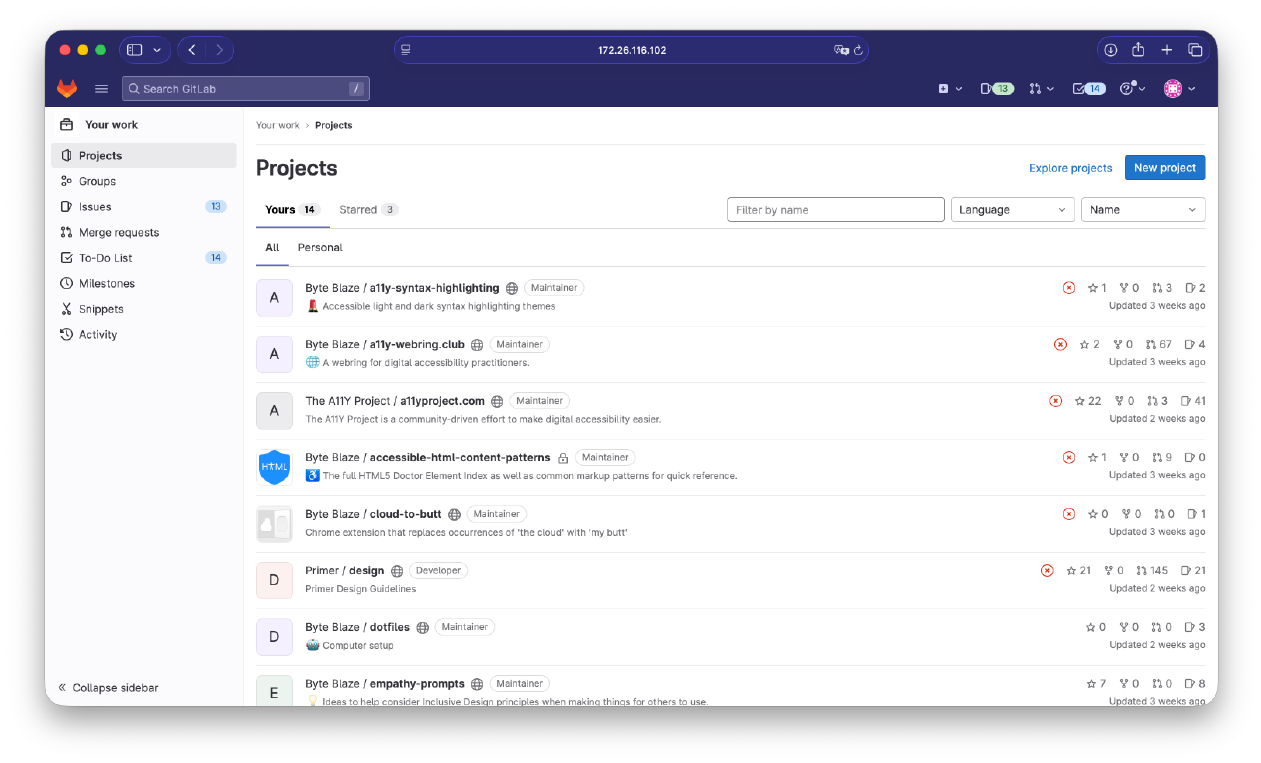}
        \caption{GitLab v16.0.6 target variant.}
        \label{fig:app_b_gitlab_target}
    \end{subfigure}
    \caption{
    GitLab source and target variants.
    }
    \label{fig:app_b_gitlab_pair}
\end{figure*}

\paragraph{Magento.}
Magento serves as the e-commerce administration environment. Its task space includes products, categories, customers, orders, inventory, attributes, and multi-field administrative forms. Compared with GitLab, the main pressure here comes less from nested project contexts and more from table- and form-heavy record manipulation: agents must search, filter, inspect, edit, and save structured business objects. Magento is suitable for our website-level criteria because it is a mature open-source e-commerce system that supports private deployment and deterministic reset. Its evolution is also interaction-centric. The product catalog is not treated as the evolving object; instead, the benchmark focuses on how the administrative interface for operating on the same catalog entities changes across major versions.

For the source variant, we use \textbf{v1.9.3}~\cite{magento193ReleaseNotes}; for the target variant, we use \textbf{v2.4.6}~\cite{adobe2023commerce246release}. This pair captures a major break in the Magento administration experience. Magento 2 reorganizes menus, grids, form layouts, product-editing workflows, and save flows relative to Magento 1, producing a strong test of version-driven skill transfer~\cite{adobe2024datamigrationtool,adobe2023commerce246release}. To keep the comparison focused on interface evolution, both deployments are initialized from a shared canonical benchmark store state and materialized through version-specific setup procedures~\cite{adobe2024datamigrationtool}. This preserves task-critical products, categories, customers, orders, inventory records, and attributes at the semantic level. Thus, the e-commerce operation required by the instruction remains the same, while the interaction surface through which the operation is carried out changes substantially.

Figure~\ref{fig:app_b_magento_pair} presents the Magento source-target pair on task-relevant administration pages. The visual comparison highlights the administrative UI migration that agents must handle while operating over aligned catalog-management data.

\begin{figure*}[htbp]
    \centering
    \begin{subfigure}{0.48\textwidth}
        \centering
        \includegraphics[width=\linewidth]{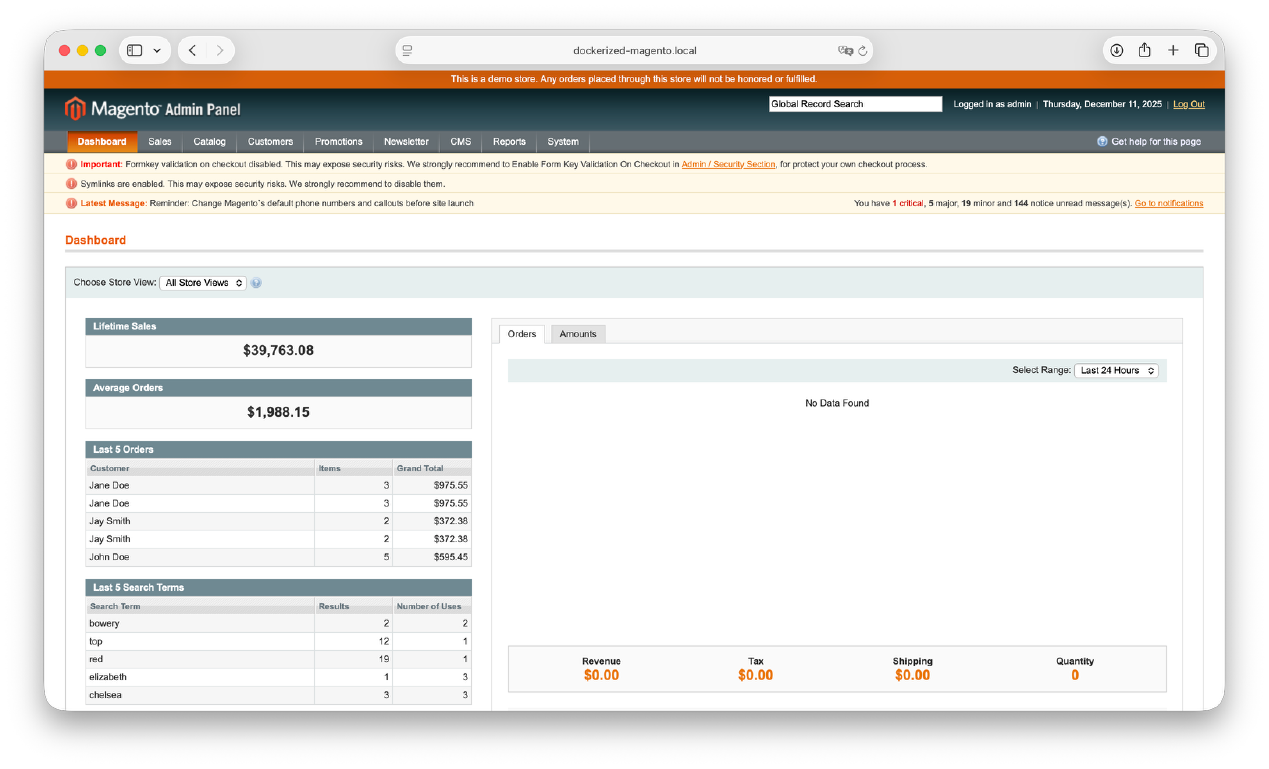}
        \caption{Magento v1.9.3 source variant.}
        \label{fig:app_b_magento_source}
    \end{subfigure}
    \hfill
    \begin{subfigure}{0.48\textwidth}
        \centering
        \includegraphics[width=\linewidth]{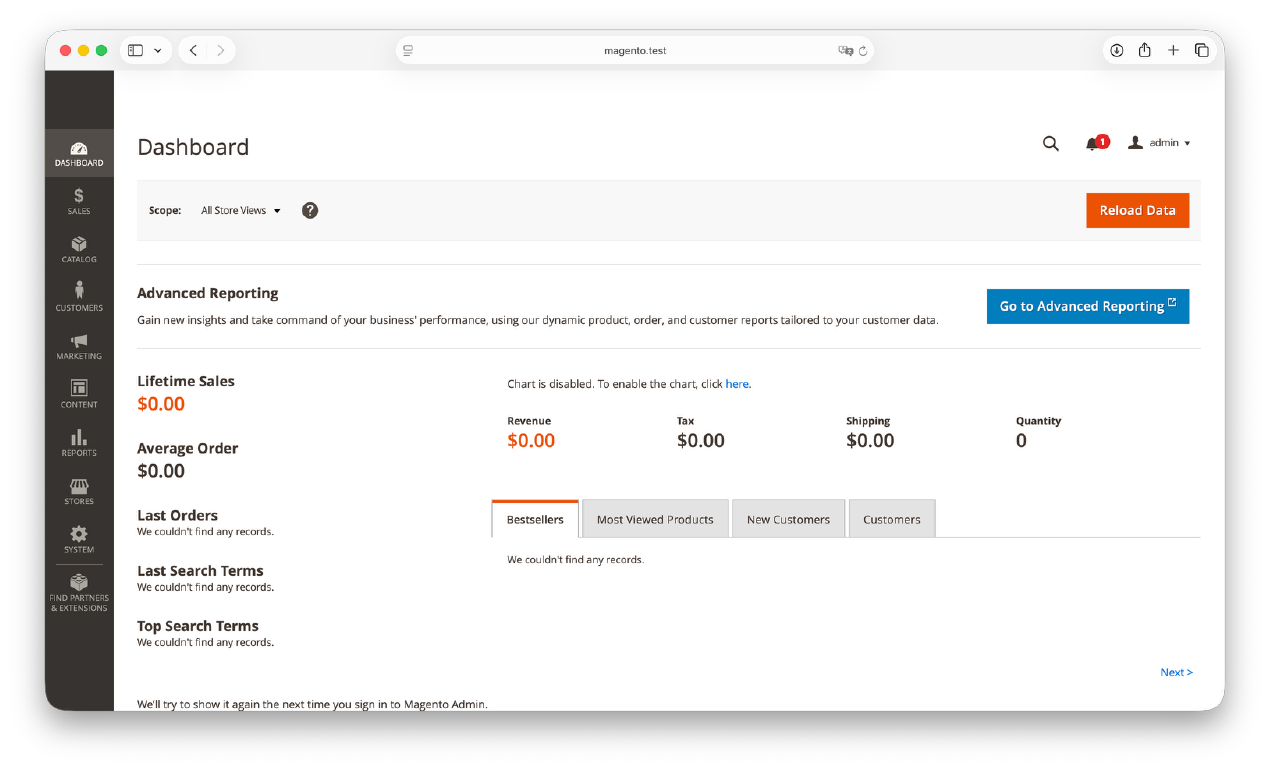}
        \caption{Magento v2.4.6 target variant.}
        \label{fig:app_b_magento_target}
    \end{subfigure}
    \caption{
    Magento source and target variants.
    }
    \label{fig:app_b_magento_pair}
\end{figure*}

\paragraph{WordPress.}
WordPress adds a complementary content-management setting. The tasks involve posts, pages, menus, media, comments, and public-facing content verification. Unlike GitLab and Magento, WordPress is not used primarily as a back-end version migration case. Instead, it isolates presentation-layer evolution: the same content-management state is evaluated under different themes. This makes WordPress useful for separating UI and data evolution. The website itself satisfies our deployment criterion as a mature open-source system that can be locally hosted and deterministically reset, while the chosen theme pair changes the observable layout, DOM structure, styling, and front-end navigation without changing the underlying seeded content.

We use the \textbf{News Portal}~\cite{wordpressNewsPortalTheme} and \textbf{Twenty Twenty-Four}~\cite{wordpress2023twentytwentyfour} themes under the same WordPress-based content-management deployment. The contrast between these themes provides a substantial front-end shift: page layout, navigation presentation, DOM organization, and styling assumptions differ markedly. No cross-version database migration is required in this case. The same WordPress instance and seeded database can be evaluated under both themes, which gives exact task-level data alignment. Posts, pages, menus, comments, and grading-relevant content remain fixed; only the presentation layer changes. WordPress therefore tests whether learned skills remain executable when front-end structure evolves while the back-end data state is held constant.

Figure~\ref{fig:app_b_wordpress_pair} shows the WordPress theme pair. In contrast to the GitLab and Magento examples, the figure illustrates theme-level presentation drift rather than a software-version migration.

\begin{figure*}[htbp]
    \centering
    \begin{subfigure}{0.48\textwidth}
        \centering
        \includegraphics[width=\linewidth]{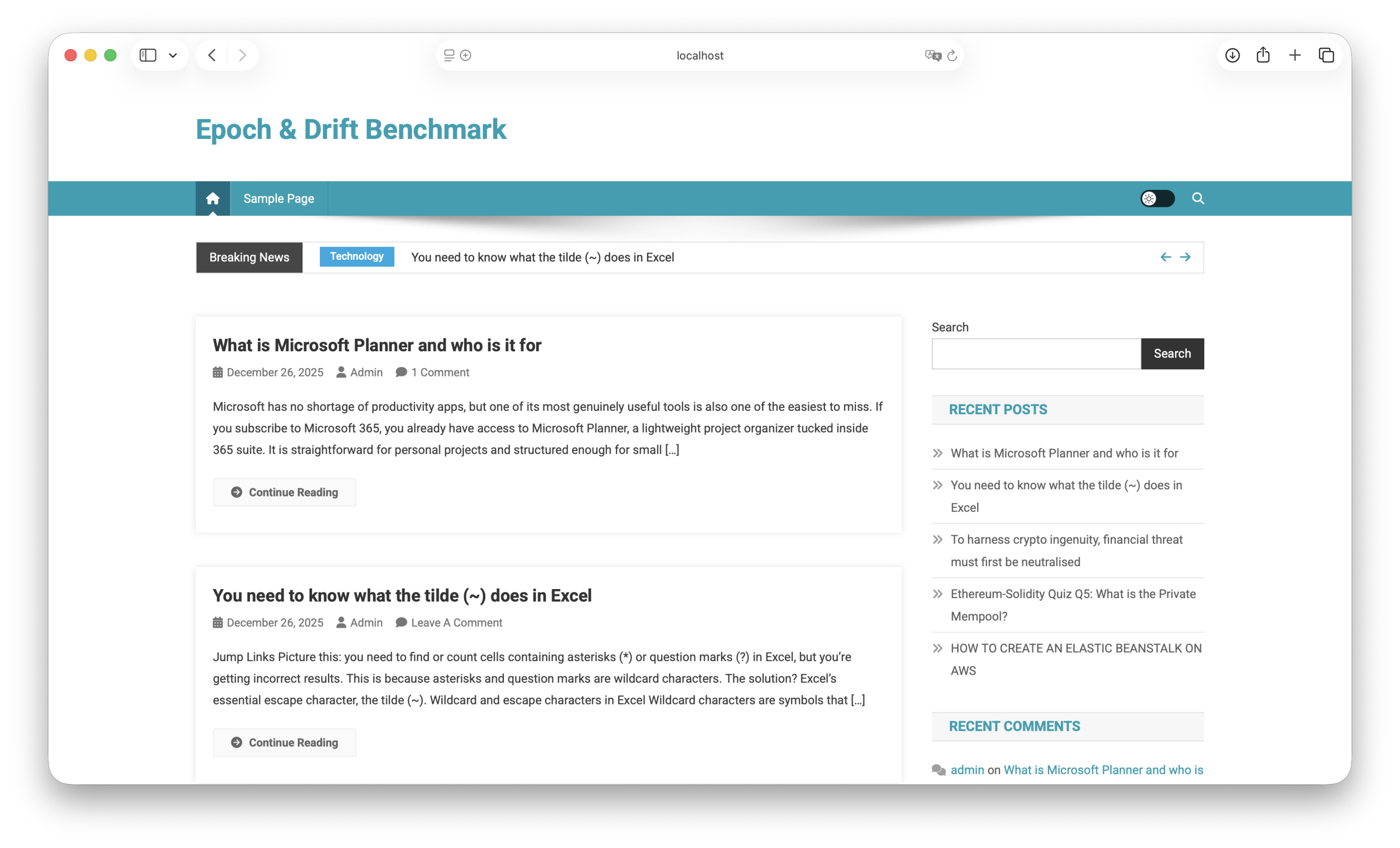}
        \caption{WordPress News Portal source variant.}
        \label{fig:app_b_wordpress_source}
    \end{subfigure}
    \hfill
    \begin{subfigure}{0.48\textwidth}
        \centering
        \includegraphics[width=\linewidth]{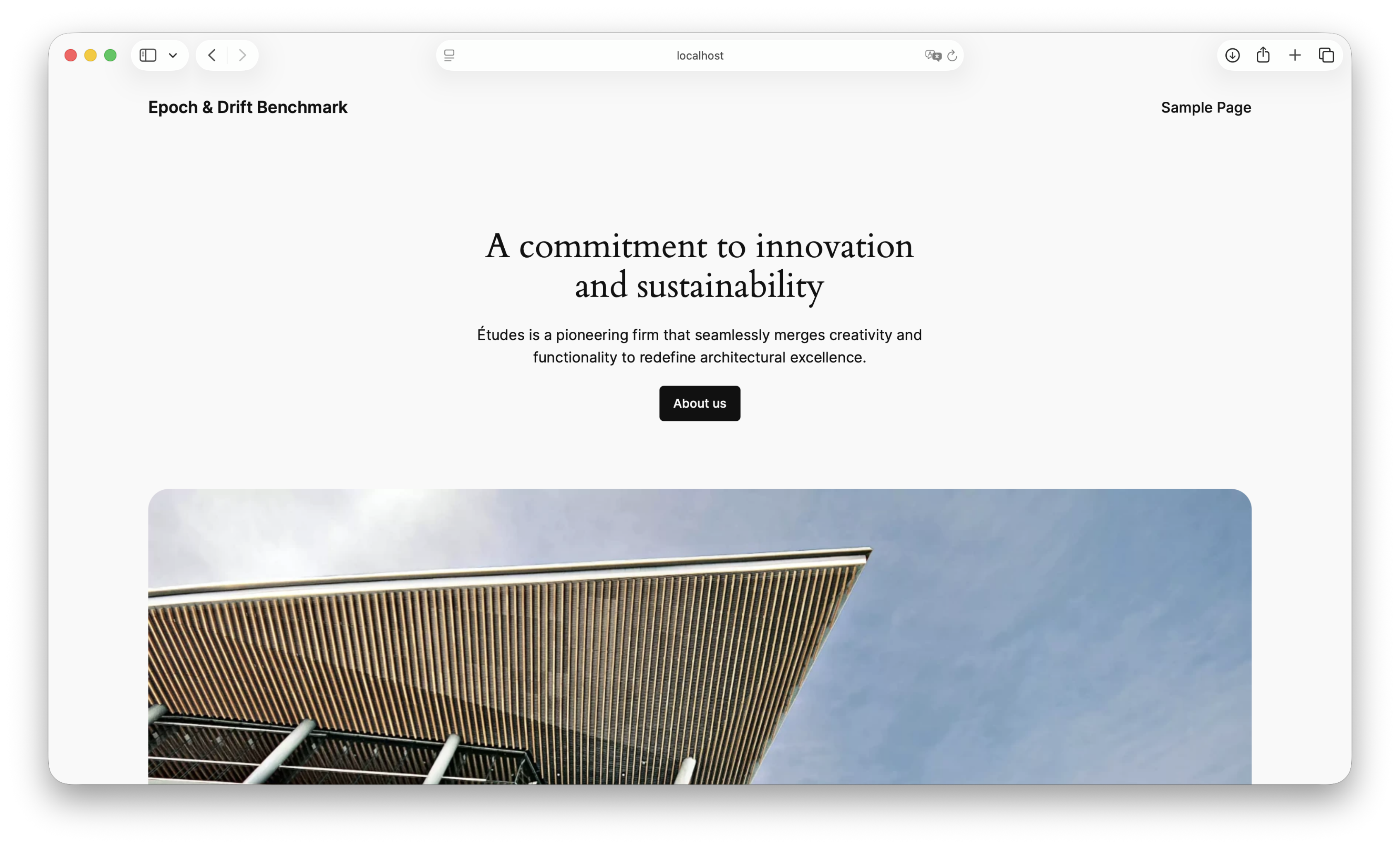}
        \caption{WordPress Twenty Twenty-Four target variant.}
        \label{fig:app_b_wordpress_target}
    \end{subfigure}
    \caption{
    WordPress source and target variants.
    }
    \label{fig:app_b_wordpress_pair}
\end{figure*}

\subsection{Shift Perturbation Implementation}
\label{sec:app_shift_perturbation_impl}

Shifted variants are generated through controlled interventions that preserve the underlying task semantics. Contextual Shift is realized by the version or theme pairings in Section~\ref{sec:app_b_env_pairing}, while Perceptual Shift and Execution Shift are implemented as runtime modules that alter grounding evidence and action feasibility, respectively.

\subsubsection{Perceptual Shift Implementation}
\label{sec:app_perceptual_shift_impl}

Perceptual Shift is implemented as a deterministic observable-surface transformation applied after page load.
The goal is not to change the task, the underlying data, or the available functionality, but to alter the interface evidence through which an agent grounds its next operation.
This design targets agents whose learned skills rely on stable visual hierarchy, DOM structure, attributes, or locator cues.

Concretely, the perturbation module applies two classes of transformations.
The first class modifies non-semantic visual presentation, such as typography, spacing, contrast, button styling, and small geometric distortions.
These changes alter the rendered appearance of the page while preserving the functional role and textual content of task-relevant elements.
The second class modifies DOM-surface signals, including locator attributes, class names, and lightweight tag-level structure.
For example, stable locator cues may be removed or perturbed, synthetic classes may be added, and simple inline tags may be replaced by semantically similar alternatives.
These transformations are designed to change the agent-facing grounding evidence rather than the website state itself.

To ensure reproducibility, all stochastic choices are generated from a fixed seed.
Thus, the same task instance produces the same shifted interface across repeated runs.
The perturbation is also re-applied to dynamically inserted elements, so the perceptual pressure persists throughout the episode rather than only at initial page load.
This allows Perceptual Shift to serve as a controlled stress test for grounding robustness while preserving task intent, task-critical entities, functional roles, navigable paths, and grader semantics.

\subsubsection{Execution Shift Implementation}
\label{sec:app_execution_shift_impl}

Execution Shift is implemented as a controlled runtime interruption mechanism.
Unlike Perceptual Shift, which changes the observable surface used for grounding, Execution Shift changes whether an otherwise valid action can be immediately realized in the browser.

The main execution-shift mechanism injects a modal blocker during the task episode.
After the relevant page is loaded, a full-screen overlay is inserted into the interface and intercepts user interactions with the underlying page.
As a result, task-relevant controls may still be present and visually accessible, but actions directed at them cannot take effect until the interruption is resolved.
The agent must therefore detect the blocker, dismiss it, and then resume the original task flow.

The interruption controller is activated only within the execution window of a benchmark task.
Start and end signals delimit when perturbations are allowed to affect the environment, preventing interference with task initialization, grading, or unrelated browser traffic.
Once the modal blocker is dismissed, the controller records the recovery event and stops re-injecting the same blocker for that episode.
This makes the perturbation reproducible while still requiring the agent to perform recovery before continuing the task.

\subsubsection{Injection Regimes for Execution Shift}
\label{sec:app_injection_regimes}

The interruption controller exposes the \emph{injection schedule} as a configurable dimension of the benchmark rather than as a fixed property of Execution Shift. SkillShift-Bench implements the three regimes summarized in Table~\ref{tab:app_injection_regimes}. Under \texttt{once}, a server-side episode flag and a client-side browser-storage flag jointly guarantee that a dismissed blocker is not reinstated. Under \texttt{multiple}, both persistent guards are removed, so the injection predicate evaluates to true on every document load and the blocker reappears after each navigation. Under \texttt{random}, injection is decided by a Bernoulli draw at each document load and the overlay is revealed after a randomized in-page delay, so the agent cannot anticipate when the interruption will occur.

\begin{table}[htbp]
\centering
\small
\setlength{\tabcolsep}{4pt}
\renewcommand{\arraystretch}{1.5}
\begin{tabularx}{0.8\linewidth}{@{}lXX@{}}
\toprule
\textbf{Regime} & \textbf{Behavior} & \textbf{Mechanism} \\
\midrule
\texttt{once} & Injected once per episode; not reinstated after dismissal & Server-side episode flag plus client-side storage guard \\
\texttt{multiple} & Reinjected on every page navigation & Persistent guards removed; injection predicate always true \\
\texttt{random} & Injected with fixed probability at a randomized in-page delay & Per-load Bernoulli draw plus client-side delay timer \\
\bottomrule
\end{tabularx}
\caption{Injection regimes available for Execution Shift.}
\label{tab:app_injection_regimes}
\end{table}

All Execution Shift conditions reported in the main paper use \texttt{once}. This is the most conservative of the three schedules: a single successful recovery restores an unobstructed page for the remainder of the episode, so the measured degradation is a lower bound on the pressure that runtime interruption can exert. Appendix~\ref{sec:app_reinjection} reports how both the fixed-library baseline and MPCR behave under the two more aggressive schedules.

\subsection{Task Validation Protocol}
\label{sec:app_task_validation}

Every task family in SkillShift-Bench is manually verified before it enters the shift-generation stage. The validation covers four properties. \emph{Validity} checks that the instruction describes a well-formed operation on the seeded website state. \emph{Correctness} checks that the automatic scorer accepts exactly the intended final state and rejects near-miss states, such as a partially filled form or an unsaved edit. \emph{Functional availability} checks that every capability required by the instruction is present in both the source and the target variant, so that a failure cannot be attributed to a removed feature. \emph{Solvability} checks that at least one executable solution path exists under each of the six evaluation conditions, including the perceptual and execution perturbations; task families for which a perturbation removes the last remaining solution path are revised or discarded.

We emphasize that this protocol establishes task solvability, not a human performance reference. SkillShift-Bench does not report a human baseline, and the numbers in Section~\ref{sec:experiments} should not be read against an implicit human ceiling. The protocol instead supports the benchmark's attribution claim: because a solution exists in every condition, a drop in success between aligned conditions reflects the agent's inability to invoke, ground, or execute a learned skill rather than the absence of a feasible path.

\subsection{Shift-Induced Failure Modes}
\label{sec:app_failure_modes}

This section examines how a previously working skill breaks under each shift. To isolate the effect of a single shift, we analyze \emph{regression} sets: tasks that the fixed-library ASI agent with GPT-5 mini solves in the reference condition but fails under the shifted condition, with the reference condition being \texttt{Src} for Perceptual and Execution Shift and the \texttt{Src}--\texttt{Tgt} pair for Contextual Shift. Both the quantitative sets and the illustrative cases below are drawn from Magento. All cited trajectories, step counts, and runtime errors are taken from the recorded execution logs.

\begin{table}[htbp]
\centering
\small
\setlength{\tabcolsep}{6.5pt}
\renewcommand{\arraystretch}{1.15}
\begin{tabular}{@{}lccc@{}}
\toprule
\textbf{Regression set} & \textbf{n} & \textbf{Steps} & \textbf{Trunc.} \\
\midrule
Contextual & 13 & 5.4 $\to$ 6.2 & 6/13 \\
Perceptual & 10 & 5.9 $\to$ 9.7 & 9/10 \\
Execution & 29 & 5.4 $\to$ 9.5 & 26/29 \\
\bottomrule
\end{tabular}
\caption{Regression sets used for failure-mode analysis, on Magento with GPT-5 mini. ``Steps'' reports mean trajectory length in the reference condition and in the shifted condition; ``Trunc.'' counts shifted runs that exhaust the step budget.}
\label{tab:app_regression_sets}
\end{table}

Table~\ref{tab:app_regression_sets} already separates two behavioral regimes. Under Perceptual and Execution Shift, failing trajectories roughly double in length and almost always exhaust the step budget, indicating unrecoverable looping. Under Contextual Shift, trajectories grow only slightly and truncate in fewer than half of the cases, indicating premature termination on a wrong route or a stale answer.

\paragraph{Contextual Shift breaks skill invocation.}
When routes, control placement, and workflow layout evolve, the learned skill is invoked against a target that has moved. In the task ``total payment of the last five completed orders,'' the source trajectory navigates the order grid and returns the correct total in five steps. Under Contextual Shift the agent instead fires the reporting step immediately, in a single action, and returns a stale figure without ever re-grounding on the relocated order page. The skill remains task-relevant, but it is invoked against a page state that the learned procedure no longer locates, and the run ends early on a wrong answer rather than looping. This is the signature that Table~\ref{tab:app_regression_sets} records as short trajectories with comparatively low truncation.

\paragraph{Perceptual Shift breaks skill grounding.}
The perturbation preserves task-critical objects and functionality while altering DOM structure, attributes, and locator signals, which appear in the logs as injected synthetic class markers. The learned action remains plausible but can no longer be grounded to the intended element. In the task ``order identifier of the newest pending order,'' the clean trajectory invokes two of ASI's named learned skills, for order navigation and status filtering. Under Perceptual Shift the status filter that was previously rendered as a selection element is rendered as a table header, so the selection action fails repeatedly with a type error and the run loops until the step budget is exhausted. What the shift violates is the grounding assumption encoded in the skill itself, namely that a particular control type sits at a particular node.

\paragraph{Execution Shift breaks execution feasibility.}
A reproducible modal blocker intercepts interactions, so semantically correct and correctly targeted skill steps become unexecutable in the current page state. In the task ``reviews mentioning \emph{decent},'' which the clean condition solves in three to four steps, the review-link click times out repeatedly and an explicit navigation call times out as well, because the overlay keeps intercepting pointer events. The agent registers that something is wrong, stalling to ask whether it may proceed, but never resolves the interruption; once the overlay reshuffles the document, the previously valid identifier for the search box can no longer be found, producing a ten-step loop with no answer. Awareness of the obstruction is therefore not sufficient: the failure lies in restoring feasibility, not in identifying what blocks it.

\paragraph{Implications.}
The three shifts do not lower success uniformly. They induce \emph{stage-distinct failure signatures} that match the diagnostic intent of each mode: wrong-route or early-termination behavior for invocation, un-actionable-locator loops for grounding, and interception-induced stalls for execution feasibility. This is the structure that a feasibility-aware, current-page-conditioned reranker is designed to exploit at transfer time, because in all three cases the shifted skill remains task-relevant while its default action is no longer the feasible one in the current page state.

This analysis also qualifies the Skill Failure Rate diagnostic defined in Section~\ref{sec:evaluation_diagnostics}. SFR aggregates two distinct causes: a skill that is retrieved and executed but returns an incorrect result, and a skill whose steps are correct but cannot be realized in the current page state. The qualitative decomposition above separates them, and SFR should be read as a supplementary diagnostic rather than as a primary metric. SFR is additionally uninformative when skills are almost never invoked, since a near-zero denominator makes the ratio unstable; we therefore report skill-avoidance behavior through ASC in such regimes and refrain from interpreting SFR values computed over a handful of invocations.

\section{Dynamic-Regret Analysis of SWOP}
\label{app:problem-setup}

This appendix is organized as a theorem-first proof appendix, while the proof itself is ordered so that each proof step consumes previously stated mathematical interfaces. Subsection~\ref{app:problem-setup:basic-setup} defines the SWOP model, objective, sliding-window estimator, and algorithm. Subsection~\ref{app:swop-main-results} states the fixed-window dynamic-regret theorem and the oracle-tuned corollary used by the informal statement in the main text. Subsections~\ref{app:bellman-supporting-facts} and~\ref{app:concentration-filtration-details} collect the KL-Bellman, boundedness, concentration, filtration, and auxiliary martingale facts used by the proof. Subsection~\ref{app:fixed-window-proof} proves the fixed-window theorem from near-optimism, certificate decomposition, summed-bonus control, martingale control, and drift counting. Subsection~\ref{app:oracle-window-proof} proves the oracle-window balancing statement.

\subsection{Model, Objective, and Algorithm}
\label{app:problem-setup:basic-setup}

\subsubsection{Nonstationary Information-State MDP Sequence}
\label{app:problem-setup:mdp-sequence}

We consider a finite-horizon information state MDP sequence that varies
with episode. There are in total \(K\) episodes, and the horizon length
of each episode is fixed as \(H\).

The episode-\(k\) environment is
\begin{equation}
\label{eq:mdp-sequence}
\mathcal M^{(k)}=
\Bigl(
\Xi,
\{\Omega(\xi)\}_{\xi\in\Xi},
\{r_h^{(k)}\}_{h=1}^H,
\{P_h^{(k)}\}_{h=1}^H,
H
\Bigr),\qquad k=1,\dots,K,
\end{equation}
where
\begin{itemize}
\item
  \(\Xi\) is a finite information state space, \(|\Xi| = S_\xi\);
\item
  for each information state \(\xi\in\Xi\), \(\Omega(\xi)\) is the
  fixed feasible action set, and it does not vary with episode;
\item
  define
\[
   A:= \max_{\xi\in\Xi} |\Omega(\xi)|;
   \]
\item
  for each stage \(h\in[H]:=\{1,\dots,H\}\), episode \(k\), state
  \(\xi\) and action \(\omega\in\Omega(\xi)\), we have
\[
   P_h^{(k)}(\cdot\mid \xi,\omega)\in \Delta(\Xi),\qquad
   r_h^{(k)}(\xi,\omega)\in [0,1].
   \]
\end{itemize}

For the theoretical analysis, we use a fixed episode-invariant candidate universe \(\Omega(\xi)\). This convention keeps the KL-regularized Bellman operators and comparators defined on a common support across episodes. Changes in local executability are represented through the non-stationary rewards and transitions rather than through a changing action support.

This fixed-domain abstraction is representation-agnostic. Executable
programs, API functions, or tools can be represented as actions in
\(\Omega(\xi)\) when they are exposed to the executor as callable
candidates. In contrast, prompt-level workflow memories, such as AWM
workflows~\cite{wang2025agentworkflowmemory}, are represented through the behavioral prior or contextual
state information rather than as additional executable actions.

\subsubsection{KL-Regularized Objective}
\label{app:problem-setup:kl-objective}

Following regularized MDP and control-as-inference formulations~\cite{geist2019regularizedmdp,levine2018probabilisticinference}, we take a reference prior on the fixed action domain
\[
\pi_{\rm ref}(\cdot\mid \xi)\in \Delta(\Omega(\xi)),
\qquad \forall \xi\in\Xi,
\]
and assume that there exists \(\pi_{\min}>0\) such that
\[
\pi_{\rm ref}(\omega\mid \xi)\ge \pi_{\min},
\qquad \forall \xi\in\Xi,\ \forall \omega\in\Omega(\xi).
\]
Fix the regularization parameter \(\alpha>0\). For any local randomized
policy \(\pi_h(\cdot\mid \xi)\in\Delta(\Omega(\xi))\), define its
one-step KL-regularized return at stage \(h\), episode \(k\) as
\begin{equation}
\label{eq:kl-one-step-return}
\tilde r_h^{(k)}(\xi,\pi_h):=
\langle \pi_h(\cdot\mid \xi), r_h^{(k)}(\xi,\cdot)\rangle
-\alpha\,\mathrm{KL}\!\left(
\pi_h(\cdot\mid \xi)\,\|\,\pi_{\rm ref}(\cdot\mid \xi)
\right).
\end{equation}

Equivalently, if the concrete action
\(\omega_h\sim \pi_h(\cdot\mid \xi_h)\) is taken, then the
within-episode KL-regularized return is
\[
\sum_{h=1}^H
\left(
r_h^{(k)}(\xi_h,\omega_h)
-\alpha\log\frac{\pi_h(\omega_h\mid \xi_h)}{\pi_{\rm ref}(\omega_h\mid \xi_h)}
\right).
\]

Under this fixed-support analytical convention, masking is absorbed into
the construction of the candidate universe and the non-stationary executability
model. The KL penalty is therefore evaluated on the common support
\(\Omega(\xi)\).

\textbf{Interpretation.}
\(\pi_{\rm ref}\) can be understood as a behavioral prior provided by a
large language model (LLM), human knowledge, or prompt-level learned
workflows. This includes workflow-memory methods such as AWM, where
learned procedures guide action selection through contextual prompting
rather than by adding new executable actions. The agent optimizes a
trade-off between environment reward and the cost of deviating from the
prior.

\subsubsection{Policy Class and Episode Value}
\label{app:problem-setup:policy-value}

Let \(\Pi\) be the class of all stage-dependent Markov randomized
policies satisfying the support constraint, namely
\[
\pi = \{\pi_h(\cdot\mid \xi)\}_{h=1}^H,
\qquad
\mathrm{supp}(\pi_h(\cdot\mid \xi))\subseteq \Omega(\xi),
\ \forall h,\xi.
\]
For any policy \(\pi\in\Pi\), define its KL-regularized value in episode
\(k\) by
\begin{equation}
\label{eq:episode-policy-value}
V_1^{\pi,(k)}(\xi):=
\mathbb E^\pi\!\left[
\sum_{h=1}^H
\left(
r_h^{(k)}(\xi_h,\omega_h)
-\alpha\log\frac{\pi_h(\omega_h\mid \xi_h)}{\pi_{\rm ref}(\omega_h\mid \xi_h)}
\right)
\;\middle|\;
\xi_1=\xi
\right].
\end{equation}

The stage-wise policy value functions used below are defined by the
following backward recursion:
\[
V_{H+1}^{\pi,(k)}(\xi):=0,\qquad \forall \xi\in\Xi,
\]
\[
Q_h^{\pi,(k)}(\xi,\omega)
:=
r_h^{(k)}(\xi,\omega)
+
P_h^{(k)}(\cdot\mid\xi,\omega)^\top V_{h+1}^{\pi,(k)},
\]
and
\[
V_h^{\pi,(k)}(\xi)
=
\sum_{\omega\in\Omega(\xi)}
\pi_h(\omega\mid\xi)
\left[
Q_h^{\pi,(k)}(\xi,\omega)
-
\alpha
\log
\frac{\pi_h(\omega\mid\xi)}
{\pi_{\rm ref}(\omega\mid\xi)}
\right],
\qquad h=H,\dots,1.
\]
The definition in \eqref{eq:episode-policy-value} is the \(h=1\) special
case of this recursion.

Let the initial distribution be \(\mu_1\in\Delta(\Xi)\). In the main
regret theorem below we use the pathwise initial states \(\xi_1^k\)
directly, rather than introducing separate \(\mu_1\)-averaged shorthand
values.

\subsubsection{Variation Budgets}
\label{app:problem-setup:variation-budgets}

Following variation-budget analyses of non-stationary bandits and episodic MDPs~\cite{besbes2014nonstationarybandits,cheung2020nonstationarymdp,mao2021nonstationaryepisodic}, we allow the environment to change only between episodes. Define the
reward drift budget and the transition drift budget as
\begin{align}
B_r
&:=
\sum_{k=1}^{K-1}
\max_{h\in[H],\,\xi\in\Xi,\,\omega\in\Omega(\xi)}
\left|
r_h^{(k+1)}(\xi,\omega)-r_h^{(k)}(\xi,\omega)
\right|, \label{eq:variation-budget-reward}\\
B_p
&:=
\sum_{k=1}^{K-1}
\max_{h\in[H],\,\xi\in\Xi,\,\omega\in\Omega(\xi)}
\left\|
P_h^{(k+1)}(\cdot\mid \xi,\omega)
-
P_h^{(k)}(\cdot\mid \xi,\omega)
\right\|_1. \label{eq:variation-budget-transition}
\end{align}

Sometimes the two can also be combined into a total variation quantity as
\begin{equation}
\label{eq:total-variation-budget}
V_K:= B_r + H B_p,
\end{equation}
because in finite-horizon Bellman recursion, transition perturbations
are usually amplified through downstream value functions by an \(O(H)\)
factor. Subsequent theorem statements can be written either separately
as functions of \(B_r,B_p\) or uniformly as functions of \(V_K\).

\paragraph{Primitive adjacent drifts.}
For later Bellman-stability and sliding-window arguments, define
\[
\Delta_r^{(k)}
:=
\max_{h\in[H],\,\xi\in\Xi,\,\omega\in\Omega(\xi)}
\left|
r_h^{(k+1)}(\xi,\omega)-r_h^{(k)}(\xi,\omega)
\right|,
\qquad k=1,\ldots,K-1,
\]
and
\[
\Delta_p^{(k)}
:=
\max_{h\in[H],\,\xi\in\Xi,\,\omega\in\Omega(\xi)}
\left\|
P_h^{(k+1)}(\cdot\mid\xi,\omega)
-
P_h^{(k)}(\cdot\mid\xi,\omega)
\right\|_1,
\qquad k=1,\ldots,K-1.
\]
Then
\[
B_r=\sum_{k=1}^{K-1}\Delta_r^{(k)},
\qquad
B_p=\sum_{k=1}^{K-1}\Delta_p^{(k)}.
\]

\paragraph{Logarithmic notation.}
Throughout the appendix, \(\widetilde O(\cdot)\) hides only absolute
constants and polylogarithmic factors in
\[
S_\xi,\ A,\ H,\ K,\ \delta^{-1},
\]
and does not hide polynomial dependence on \(S_\xi,A,H,K\) or on the
variation budgets \(B_r,B_p,V_K\).

\subsubsection{Dynamic Comparator and Dynamic Regret}
\label{app:problem-setup:dynamic-regret}

As in dynamic-regret analyses for non-stationary decision processes~\cite{besbes2014nonstationarybandits,cheung2020nonstationarymdp,mao2021nonstationaryepisodic}, in the episodic setting, the dynamic benchmark compares each episode with
that episode's own optimal KL-regularized policy~\cite{geist2019regularizedmdp}. The KL-regularized
episodic dynamic regret used in the main theorem is
\begin{equation}
\label{eq:dynamic-regret}
\mathrm{DynReg}(K):=
\sum_{k=1}^K
\Big(
V_1^{\star,(k)}(\xi_1^k) - V_1^{\hat\pi^k,(k)}(\xi_1^k)
\Big).
\end{equation}
This is the standard value-based episodic dynamic regret, evaluated along the
realized initial states and the algorithmic history. No separate
cumulative-value shorthand or expected-regret object is used in the proof of
Theorem~\ref{thm:dynamic-regret-fixed-window}.

The theorem is stated in the standard variation-dependent form used for
non-stationary decision processes. This form is more informative than an
unconditional rate, because it separates statistical estimation error from the
intrinsic cost of environmental variation. Sublinear regret follows in the
low-variation regimes stated in Corollary~\ref{cor:oracle-tuned-window}. In the
common fixed-complexity regime where \(H,S_\xi,A\) are treated as problem
constants, the oracle-tuned window gives the concise condition \(V_K=o(K)\).
The corresponding growing-dimension condition is stated explicitly in
Corollary~\ref{cor:oracle-tuned-window}.

\subsubsection{Sliding-Window Empirical Model and Bonus}
\label{app:problem-setup:sliding-window-model}

Following sliding-window optimism for non-stationary reinforcement learning~\cite{cheung2020nonstationarymdp,mao2021nonstationaryepisodic}, to perform provable learning in a nonstationary environment, we use a
sliding window of length \(W\) and fix throughout
\[
1\le W\le K.
\]
For episode \(k\), define the window
\begin{equation}
\label{eq:sliding-window}
\mathcal W_k:= \{\max\{1,k-W\},\dots,k-1\}.
\end{equation}

At the beginning of episode \(k\) the algorithm uses only the data from
the most recent episodes in the window \(\mathcal W_k\) to
construct the window empirical model. The basic idea is:

\begin{enumerate}
\item
  Use only the most recent window data in order to track environment
  changes;
\item
  Estimate rewards and transitions within the window;
\item
  Construct an upper-confidence bonus;
\item
  Perform optimistic KL-regularized backward planning on the ``empirical
  model + bonus'';
\item
  Output the corresponding softmax policy and execute it.
\end{enumerate}

In this way, the environment inside the window can be regarded as
``approximately stationary'', while old data outside the window are
actively discarded to reduce the bias caused by drift.

For each \((h,\xi,\omega)\), define the number of visits inside the
window by
\begin{equation}
\label{eq:window-count}
N_h^k(\xi,\omega):=
\sum_{k'\in\mathcal W_k}
\mathbf 1\{(\xi_h^{k'},\omega_h^{k'})=(\xi,\omega)\}.
\end{equation}

The empirical reward mean uses the unique zero-count default value:
\[
\hat r_h^k(\xi,\omega):=
\begin{cases}
\displaystyle
\frac{1}{N_h^k(\xi,\omega)}
\sum_{k'\in\mathcal W_k}
r_h^{k'}\,
\mathbf 1\{(\xi_h^{k'},\omega_h^{k'})=(\xi,\omega)\},
& N_h^k(\xi,\omega)\ge 1,\\[1.2em]
0, & N_h^k(\xi,\omega)=0.
\end{cases}
\]

The empirical transition distribution uses a pre-fixed default distribution
\[
P_h^\circ(\cdot\mid\xi,\omega)\in\Delta(\Xi)
\]
as the unique cold-start convention, and is defined by
\[
\hat P_h^k(\xi'\mid\xi,\omega):=
\begin{cases}
\displaystyle
\frac{1}{N_h^k(\xi,\omega)}
\sum_{k'\in\mathcal W_k}
\mathbf 1\{(\xi_h^{k'},\omega_h^{k'},\xi_{h+1}^{k'})=(\xi,\omega,\xi')\},
& N_h^k(\xi,\omega)\ge 1,\\[1.2em]
P_h^\circ(\xi'\mid\xi,\omega),
& N_h^k(\xi,\omega)=0.
\end{cases}
\]

Here \(r_h^{k'}\) denotes the immediate reward observed at stage \(h\) of episode \(k'\).

The only UCB bonus used throughout is defined as
\begin{equation}
\label{eq:ucb-bonus}
b_h^k(\xi,\omega)
=
c_b H\sqrt{
\frac{S_\xi\log(C_{\log}S_\xi AHK/\delta)}
{(N_h^k(\xi,\omega)+1)\wedge W}
},
\end{equation}
where \(C_{\log}\ge e^2\) is an absolute constant and \(c_b>0\) is chosen
according to the parameter convention in
Subsection~\ref{app:main-proof-setup:empirical-optimistic-model}.
Throughout, assume \(\delta\in(0,1)\), \(1\le W\le K\), and
\(S_\xi,A,H,K\ge1\). We choose \(C_{\log}\) sufficiently large for the
reward and transition concentration union bounds, and choose \(c_b\) so
that
\[
c_b\sqrt{S_\xi\log(C_{\log}S_\xi AHK/\delta)}\ge 2.
\]
This displayed condition is used to uniformly absorb the maximal bonus
in cold-start cases; the explicit concentration-absorption lower bound on
\(c_b\) is fixed in
Subsection~\ref{app:main-proof-setup:empirical-optimistic-model}.
Drift does not enter the bonus; it is handled only through \(D_k\) from
Subsection~\ref{app:main-proof-setup:window-drift} and the final regret
decomposition.

\subsubsection{Optimistic Soft Planning}
\label{app:algorithm:optimistic-planning}

Define the optimistic \(Q\)-backup by
\begin{equation}
\label{eq:optimistic-q}
\widetilde Q_h^k(\xi,\omega):=
\min\left\{
H-h+1,\;
\hat r_h^k(\xi,\omega)
+
b_h^k(\xi,\omega)
+
\hat P_h^k(\cdot\mid\xi,\omega)^\top \widetilde V_{h+1}^k
\right\},
\end{equation}
with terminal condition
\[
\widetilde V_{H+1}^k(\xi)=0.
\]
Then define the optimistic value by
\begin{equation}
\label{eq:optimistic-value}
\widetilde V_h^k(\xi):=
\alpha\log
\sum_{\omega\in\Omega(\xi)}
\pi_{\rm ref}(\omega\mid\xi)
\exp\!\left(
\frac{\widetilde Q_h^k(\xi,\omega)}{\alpha}
\right).
\end{equation}

The corresponding output policy is defined by
\begin{equation}
\label{eq:optimistic-policy}
\hat\pi_h^k(\omega\mid\xi)
=
\frac{
\pi_{\rm ref}(\omega\mid\xi)
\exp(\widetilde Q_h^k(\xi,\omega)/\alpha)
}{
\sum_{\omega'\in\Omega(\xi)}
\pi_{\rm ref}(\omega'\mid\xi)
\exp(\widetilde Q_h^k(\xi,\omega')/\alpha)
}.
\end{equation}
This is the natural form of sliding-window UCB-style optimism under KL regularization~\cite{cheung2020nonstationarymdp,geist2019regularizedmdp}. Algorithm~\ref{alg:swop} summarizes the resulting SWOP procedure across episodes, combining sliding-window estimation, optimistic soft planning, policy execution, and trajectory update.

\begin{algorithm}[t]
\caption{SWOP: Sliding Window Optimistic Prior-Guided Planning}
\label{alg:swop}
\begin{algorithmic}[1]
\Require Window size $W\in\{1,\ldots,K\}$, regularization parameter $\alpha$, confidence level $\delta$, reference prior $\pi_{\rm ref}$, and default transition kernels $P_h^\circ(\cdot\mid \xi,\omega)$.
\For{$k=1,2,\ldots,K$}
    \State Set $\mathcal W_k=\{\max\{1,k-W\},\ldots,k-1\}$.
    \For{each $(h,\xi,\omega)$}
        \State Compute $N_h^k(\xi,\omega)$, $\hat r_h^k(\xi,\omega)$, $\hat P_h^k(\cdot\mid\xi,\omega)$, and $b_h^k(\xi,\omega)$ as in Subsubsection~\ref{app:problem-setup:sliding-window-model}.
    \EndFor
    \State Set $\widetilde V_{H+1}^k(\xi)=0$ for all $\xi\in\Xi$.
    \For{$h=H,H-1,\ldots,1$}
        \For{each $\xi\in\Xi$ and $\omega\in\Omega(\xi)$}
            \State Compute
            \[
            \widetilde Q_h^k(\xi,\omega)=
            \min\left\{H-h+1,\hat r_h^k(\xi,\omega)+b_h^k(\xi,\omega)+\hat P_h^k(\cdot\mid\xi,\omega)^\top\widetilde V_{h+1}^k\right\}.
            \]
        \EndFor
        \For{each $\xi\in\Xi$}
            \State Compute $\widetilde V_h^k(\xi)$ by the optimistic soft Bellman backup and set $\hat\pi_h^k(\cdot\mid\xi)$ to the induced softmax policy.
        \EndFor
    \EndFor
    \State Execute $\hat\pi^k=\{\hat\pi_h^k\}_{h=1}^H$ in episode $k$ and store the trajectory.
\EndFor
\end{algorithmic}
\end{algorithm}

\subsection{Main Theoretical Guarantees}
\label{app:swop-main-results}

This subsection states the two main formal outputs of Appendix~\ref{app:problem-setup}. The first result is the fixed-window high-probability dynamic-regret bound. The second result is the oracle-tuned fixed-window upper envelope corresponding to the informal statement in the main text. The proofs are given in Subsections~\ref{app:fixed-window-proof} and~\ref{app:oracle-window-proof}.

\subsubsection{Dynamic-Regret Bound for a Fixed Window}
\label{app:main-proof:fixed-window-bound}

The following result builds on the non-stationary MDP sliding-window optimism perspective and the KL-regularized MDP formulation cited above. The theorem itself is the paper's SWOP-specific analysis under the fixed-domain abstraction and proof setup developed in this appendix.

\begin{theorem}[SWOP Dynamic Regret under Sliding-Window KL-Regularized Optimistic Planning]
\label{thm:dynamic-regret-fixed-window}

Fix \(\alpha>0\), confidence level \(\delta\in(0,1)\), and a window length
\(W\in\{1,\ldots,K\}\) chosen before the algorithm is run. Suppose the
following standing assumptions hold.

\begin{enumerate}
\item The state space \(\Xi\) is finite with \(|\Xi|=S_\xi\), and the action
set \(\Omega(\xi)\) is fixed over episodes for each \(\xi\in\Xi\). Let
\(A:=\max_{\xi\in\Xi}|\Omega(\xi)|\). Throughout,
\(K,H,S_\xi,A\ge1\), \(\mu_1\in\Delta(\Xi)\), and the initial state of
each episode satisfies \(\xi_1^k\sim\mu_1\) independently of the completed
past history.
\item For every \(k\in[K]\), \(h\in[H]\), \(\xi\in\Xi\), and
\(\omega\in\Omega(\xi)\),
\[
r_h^{(k)}(\xi,\omega)\in[0,1],
\qquad
P_h^{(k)}(\cdot\mid\xi,\omega)\in\Delta(\Xi).
\]
The adjacent drifts \(\Delta_r^{(k)}\), \(\Delta_p^{(k)}\), the budgets
\(B_r,B_p\), and \(V_K=B_r+HB_p\) are defined in
Subsection~\ref{app:problem-setup:variation-budgets}.
\item The reference prior has full support on the fixed action domain:
\[
\pi_{\rm ref}(\omega\mid\xi)\ge\pi_{\min}>0,
\qquad
\forall \xi\in\Xi,\ \omega\in\Omega(\xi).
\]
All KL-regularized values, policies, and Bellman operators use the same
\(\alpha\) and \(\pi_{\rm ref}\).
\item The sampling and filtration conventions in
Subsection~\ref{app:problem-setup:sampling-filtration} hold, including the
environment predictability convention. In particular, at the beginning of
episode \(k\), the episode-\(k\) model objects are fixed or predictable,
all algorithmic planning objects are measurable with respect to the
completed history of episodes \(1,\ldots,k-1\), and
\(\xi_1^k\sim\mu_1\) is independent of that past.
\item Algorithm~\ref{alg:swop} uses fixed default transition kernels
\[
P_h^\circ(\cdot\mid\xi,\omega)\in\Delta(\Xi),
\qquad h\in[H],\ \xi\in\Xi,\ \omega\in\Omega(\xi),
\]
and the unified bonus
\[
b_h^k(\xi,\omega)
=
c_b H\sqrt{
\frac{S_\xi\log(C_{\log}S_\xi AHK/\delta)}
{(N_h^k(\xi,\omega)+1)\wedge W}
},
\]
where \(C_{\log}\ge e^2\) is chosen sufficiently large for
Lemmas~\ref{lem:reward-concentration}--\ref{lem:total-concentration-event},
and \(c_b\) satisfies the concentration-absorption and cold-start
conditions in
Subsection~\ref{app:main-proof-setup:empirical-optimistic-model}.
\end{enumerate}

The regret is the value-based dynamic regret defined in
\eqref{eq:dynamic-regret}. Then there exist absolute constants
\(C_1,C_2,C_3>0\) such that, with probability at least \(1-\delta\), the
pathwise random variable \(\mathrm{DynReg}(K)\) satisfies
\begin{equation}
\label{eq:fixed-window-regret-bound}
\mathrm{DynReg}(K)
\le
C_1 H^2 K S_\xi\sqrt{\frac{A\log(C_{\log}S_\xi AHK/\delta)}{W}}
+
C_2 W H(B_r+HB_p)
+
C_3 H^2\sqrt{K H\log(C_{\log}/\delta)}.
\end{equation}
Equivalently, up to logarithmic factors,
\[
\mathrm{DynReg}(K)
=
\widetilde O\!\left(
H^2 K S_\xi\sqrt{\frac{A}{W}}
+
W H(B_r+HB_p)
+
H^2\sqrt{KH}
\right).
\]

\end{theorem}

\subsubsection{Oracle-Tuned Window Corollary}
\label{app:main-proof:oracle-tuned-window}

\begin{corollary}[Oracle-Tuned Fixed-Window Balancing]
\label{cor:oracle-tuned-window}

Recall that
\[
V_K:=B_r+HB_p.
\]
Assume that the window length is selected before deployment as a deterministic
function of the variation budget \(V_K\). More precisely, for any admissible
\(W_{\mathrm{orc}}\in\{1,\dots,K\}\) fixed before the algorithm is run,
Theorem~\ref{thm:dynamic-regret-fixed-window} gives, with probability at least
\(1-\delta\),
\[
\mathrm{DynReg}(K)
\le
\widetilde O\!\left(
H^2KS_\xi\sqrt{\frac{A}{W_{\mathrm{orc}}}}
+
W_{\mathrm{orc}}HV_K
+
H^2\sqrt{KH}
\right).
\]
Equivalently, an oracle with knowledge of \(V_K\) may choose the deterministic
upper-envelope minimizer
\[
W_{\mathrm{orc}}
\in
\arg\min_{1\le W\le K}
\left\{
H^2KS_\xi\sqrt{\frac{A}{W}}
+
WHV_K
\right\}
\]
before the run, yielding
\[
\mathrm{DynReg}(K)
\le
\widetilde O\!\left(
\min_{1\le W\le K}
\left\{
H^2KS_\xi\sqrt{\frac{A}{W}}
+
WHV_K
\right\}
+
H^2\sqrt{KH}
\right).
\]
For a uniform display over all variation regimes, use the smoothed quantity
\(V_K+1\) and set
\begin{equation}
\label{eq:oracle-tuned-window}
W_\star:=
\left(
\frac{H^2S_\xi^2 A K^2}{(V_K+1)^2}
\right)^{1/3},
\end{equation}
with the admissible deterministic window
\[
W_{\mathrm{orc}}:=
\max\!\left\{
1,\;
\min\!\left\{
K,\;
\left\lceil W_\star\right\rceil
\right\}
\right\}.
\]
In the untruncated case \(1\le W_\star\le K\), running the algorithm with this
fixed window gives
\begin{equation}
\label{eq:oracle-tuned-bound}
\mathrm{DynReg}(K)
\le
\widetilde O\!\left(
H^{5/3}S_\xi^{2/3}A^{1/3}
K^{2/3}(V_K+1)^{1/3}
+
H^2\sqrt{KH}
\right),
\end{equation}
with probability at least \(1-\delta\).

In the common fixed-complexity regime where \(H,S_\xi,A\) are treated as
problem constants, the tuned-window bound gives
\[
\mathrm{DynReg}(K)=\widetilde O\!\left(K^{2/3}(V_K+1)^{1/3}\right)+o(K),
\]
and hence sublinear dynamic regret whenever \(V_K=o(K)\). If these problem
parameters are allowed to grow with \(K\), the corresponding sufficient
sublinearity conditions are
\[
H^5S_\xi^2A(V_K+1)=o(K),
\qquad
H^2\sqrt{KH}=o(K).
\]

\end{corollary}

\subsubsection{Proof Roadmap and Dependency Graph}
\label{app:swop-proof-roadmap}

The proof of Theorem~\ref{thm:dynamic-regret-fixed-window} is organized around five mathematical interfaces.

\begin{enumerate}
\item \textbf{Near-optimism under window drift.} Proposition~\ref{prop:near-optimism-window-drift} shows that the clipped optimistic value upper-bounds the episode-wise optimal value up to the local window-drift term:
\[
V_h^{\star,(k)}(\xi)-\widetilde V_h^k(\xi)
\le
C_{\mathrm{opt}}(H-h+1)D_k.
\]
\item \textbf{Certificate decomposition.} Lemma~\ref{lem:pathwise-certificate-bound} controls the episode-level gap between the optimistic value and the value of the executed policy by the realized bonuses, a martingale remainder, and the same drift quantity. The fixed-window theorem consumes the following pathwise interface:
\[
\widetilde V_1^k(\xi_1^k)-V_1^{\hat\pi^k,(k)}(\xi_1^k)
\le
2\sum_{h=1}^H b_h^k(\xi_h^k,\omega_h^k)
+
M_k
+
C_{\mathrm{cert}}H D_k.
\]
\item \textbf{Summed-bonus control.} Proposition~\ref{prop:summed-bonus-bound} controls the cumulative realized uncertainty bonus under overlapping sliding-window counts.
\item \textbf{Martingale control.} Proposition~\ref{prop:unified-martingale-control} controls the unified pathwise remainder \(\sum_k M_k\) with high probability.
\item \textbf{Window drift counting.} Equation~\eqref{eq:window-drift-sum} converts episode-wise window drift into the variation-budget contribution \(W(B_r+HB_p)\).
\end{enumerate}

Combining these five interfaces yields the three terms in the fixed-window theorem: statistical estimation, non-stationary drift, and martingale noise.

\subsection{Supporting KL-Regularized Bellman Facts}
\label{app:bellman-supporting-facts}

This subsection collects the KL-regularized Bellman algebra, log-sum-exp monotonicity, and boundedness facts used by the main proof. It also includes additional Bellman stability facts for context; those stability facts are not invoked in the proof of Theorem~\ref{thm:dynamic-regret-fixed-window}, whose drift control uses Lemma~\ref{lem:conditional-mean-model-drift} and \eqref{eq:window-drift-sum}.

\subsubsection{Soft Bellman Recursion}
\label{app:kl-bellman:recursion}

This subsection uses the standard KL-regularized Bellman and softmax form for regularized MDPs~\cite{geist2019regularizedmdp,levine2018probabilisticinference}, specialized to the fixed-domain episode-\(k\) model.

For each episode \(k\), define the stage-\(h\) optimal regularized value
function \(V_{h}^{\star,(k)}:\Xi\to\mathbb R\), with terminal condition
\[
V_{H+1}^{\star,(k)}(\xi)=0,\qquad \forall \xi\in\Xi.
\]
The corresponding soft \(Q\)-function is defined by
\begin{equation}
\label{eq:soft-q-true}
Q_h^{\star,(k)}(\xi,\omega):=
r_h^{(k)}(\xi,\omega)
+
P_h^{(k)}(\cdot\mid \xi,\omega)^\top V_{h+1}^{\star,(k)}.
\end{equation}
Then the KL-regularized Bellman recursion is
\begin{equation}
\label{eq:kl-bellman-recursion}
V_h^{\star,(k)}(\xi)
=
\alpha \log
\sum_{\omega\in\Omega(\xi)}
\pi_{\rm ref}(\omega\mid \xi)
\exp\!\left(
\frac{Q_h^{\star,(k)}(\xi,\omega)}{\alpha}
\right),
\qquad h=H,\dots,1.
\end{equation}
For any episode \(k\), stage \(h\), and bounded function
\(f:\Xi\to\mathbb R\), define
\[
(\mathcal T_h^{(k)}f)(\xi)
:=
\alpha\log
\sum_{\omega\in\Omega(\xi)}
\pi_{\rm ref}(\omega\mid\xi)
\exp\!\left(
\frac{
r_h^{(k)}(\xi,\omega)
+
P_h^{(k)}(\cdot\mid\xi,\omega)^\top f
}{\alpha}
\right).
\]
With this notation,
\[
V_{H+1}^{\star,(k)}\equiv 0,
\qquad
V_h^{\star,(k)}=\mathcal T_h^{(k)}V_{h+1}^{\star,(k)}.
\]

The corresponding stage-wise greedy regularized policy has the following
closed-form softmax form:
\begin{equation}
\label{eq:softmax-policy}
\pi_{h}^{\star,(k)}(\omega\mid \xi)
=
\frac{
\pi_{\rm ref}(\omega\mid \xi)
\exp\!\left(
Q_h^{\star,(k)}(\xi,\omega)/\alpha
\right)
}{
\sum_{\omega'\in\Omega(\xi)}
\pi_{\rm ref}(\omega'\mid \xi)
\exp\!\left(
Q_h^{\star,(k)}(\xi,\omega')/\alpha
\right)
},
\qquad \omega\in\Omega(\xi).
\end{equation}
This closed-form expression is a key advantage of KL regularization: the
optimal policy is naturally absolutely continuous with respect to
\(\pi_{\rm ref}\) as a soft policy, thereby avoiding the discontinuity
and tie-breaking technical issues caused by hard argmax selection.

\subsubsection{KL-Bellman Optimality}
\label{app:kl-bellman:optimality}

\begin{proposition}[KL-Bellman optimality and softmax form]
\label{prop:kl-bellman-optimality}

For any fixed episode \(k\), stage \(h\), and state \(\xi\), let \(q=(q_\omega)_{\omega\in\Omega(\xi)}\in\mathbb R^{|\Omega(\xi)|}\). Then
\[
\max_{\pi\in\Delta(\Omega(\xi))}
\left\{
\sum_{\omega}\pi(\omega)q_\omega
-\alpha\,\mathrm{KL}\bigl(\pi\|\pi_{\rm ref}(\cdot\mid\xi)\bigr)
\right\}
=
\alpha\log\sum_{\omega\in\Omega(\xi)}
\pi_{\rm ref}(\omega\mid\xi)\exp(q_\omega/\alpha),
\]
and the unique maximizer is
\[
\pi^q(\omega)
=
\frac{\pi_{\rm ref}(\omega\mid\xi)\exp(q_\omega/\alpha)}
{\sum_{\omega'}\pi_{\rm ref}(\omega'\mid\xi)\exp(q_{\omega'}/\alpha)}.
\]
Therefore, by taking
\[
q_\omega
=
r_h^{(k)}(\xi,\omega)
+
P_h^{(k)}(\cdot\mid\xi,\omega)^\top V_{h+1}^{\star,(k)}
\]
in the above equality and applying backward induction from \(h=H\) we obtain the KL-regularized Bellman recursion and the optimal soft policy in this section. 

\end{proposition}
\begin{proof}

Let \(p_\omega=\pi_{\rm ref}(\omega\mid\xi)\). By assumption \(p_\omega>0\). Let
\[
Z(q):=\sum_{\omega}p_\omega\exp(q_\omega/\alpha),
\qquad
\pi^q(\omega):=\frac{p_\omega\exp(q_\omega/\alpha)}{Z(q)}.
\]
For any \(\pi\in\Delta(\Omega(\xi))\), direct expansion of the KL gives
\[
\alpha\log Z(q)
-
\left\{\langle\pi,q\rangle-\alpha\mathrm{KL}(\pi\|p)\right\}
=
\alpha\mathrm{KL}(\pi\|\pi^q)
\ge0.
\]
Therefore the maximum value is \(\alpha\log Z(q)\), and since \(\mathrm{KL}(\pi\|\pi^q)=0\) if and only if \(\pi=\pi^q\), the maximizer is unique. 

The terminal layer \(V_{H+1}^{\star,(k)}\equiv0\) has been fixed. If \(V_{h+1}^{\star,(k)}\) has been defined, then the optimal local choice at stage \(h\) is exactly the maximization problem above with \(q_\omega=Q_h^{\star,(k)}(\xi,\omega)\) substituted in; its maximum value gives \(V_h^{\star,(k)}(\xi)\), and its unique maximizing policy gives \(\pi_h^{\star,(k)}(\cdot\mid\xi)\). The backward induction is complete. This proves the claim.
\end{proof}

\subsubsection{Log-Sum-Exp and Bellman Non-Expansiveness}
\label{app:tools:nonexpansiveness}

\begin{lemma}[Log-Sum-Exp Backup Is \(\ell_\infty\)-Non-Expansive in \(Q\)]
\label{lem:log-sum-exp-nonexpansive}

Fix any state \(\xi\) and define the map
\[
\Phi_\xi(q):=
\alpha\log\sum_{\omega\in\Omega(\xi)}
\pi_{\rm ref}(\omega\mid\xi)\exp(q_\omega/\alpha),
\qquad q\in\mathbb R^{|\Omega(\xi)|}.
\]
Then for any \(q,q'\in\mathbb R^{|\Omega(\xi)|}\), 
\[
|\Phi_\xi(q)-\Phi_\xi(q')|
\le
\|q-q'\|_\infty.
\]

\end{lemma}
\begin{proof}

Let
\[
\delta:= \|q-q'\|_\infty.
\]
Then for each action \(\omega\), 
\[
q_\omega \le q'_\omega + \delta.
\]
Hence
\[
\exp(q_\omega/\alpha)
\le
e^{\delta/\alpha}\exp(q'_\omega/\alpha).
\]
Multiplying both sides by \(\pi_{\rm ref}(\omega\mid\xi)\) and summing gives
\[
\sum_{\omega}\pi_{\rm ref}(\omega\mid\xi)\exp(q_\omega/\alpha)
\le
e^{\delta/\alpha}
\sum_{\omega}\pi_{\rm ref}(\omega\mid\xi)\exp(q'_\omega/\alpha).
\]
Taking \(\alpha\log\) gives
\[
\Phi_\xi(q)\le \Phi_\xi(q')+\delta.
\]
Swapping \(q,q'\) gives similarly
\[
\Phi_\xi(q')\le \Phi_\xi(q)+\delta.
\]
Thus
\[
|\Phi_\xi(q)-\Phi_\xi(q')|\le \delta=\|q-q'\|_\infty.
\]
This proves the claim. 

\end{proof}

\begin{lemma}[Log-Sum-Exp Monotonicity and Translation]
\label{lem:log-sum-exp-monotone-translation}
Fix any state \(\xi\) and let
\[
\Phi_\xi(q):=
\alpha\log\sum_{\omega\in\Omega(\xi)}
\pi_{\rm ref}(\omega\mid\xi)\exp(q_\omega/\alpha).
\]
Then the following two properties hold.
\begin{enumerate}
\item If \(q_\omega\ge q'_\omega\) for all \(\omega\in\Omega(\xi)\), then
\[
\Phi_\xi(q)\ge \Phi_\xi(q').
\]
\item For any scalar \(c\in\mathbb R\),
\[
\Phi_\xi(q+c\mathbf 1)=\Phi_\xi(q)+c.
\]
Consequently, if \(q_\omega\ge q'_\omega-\varepsilon\) for all
\(\omega\in\Omega(\xi)\), then
\[
\Phi_\xi(q)\ge \Phi_\xi(q')-\varepsilon.
\]
\end{enumerate}
\end{lemma}
\begin{proof}
The first claim follows because the exponential function is increasing
and all weights \(\pi_{\rm ref}(\omega\mid\xi)\) are nonnegative. For the
second claim,
\[
\Phi_\xi(q+c\mathbf 1)
=
\alpha\log\sum_{\omega}
\pi_{\rm ref}(\omega\mid\xi)\exp((q_\omega+c)/\alpha)
=
\alpha\log\left(
e^{c/\alpha}
\sum_{\omega}\pi_{\rm ref}(\omega\mid\xi)e^{q_\omega/\alpha}
\right)
=
\Phi_\xi(q)+c.
\]
If \(q_\omega\ge q'_\omega-\varepsilon\) for all \(\omega\), then by
monotonicity and translation,
\[
\Phi_\xi(q)\ge \Phi_\xi(q'-\varepsilon\mathbf 1)
=
\Phi_\xi(q')-\varepsilon.
\]
This proves the claim.
\end{proof}

\begin{lemma}[Soft Bellman Operator Is Non-Expansive in Future Values]
\label{lem:soft-bellman-nonexpansive}

For any fixed \(k,h\), and any two functions \(f,g:\Xi\to\mathbb R\), 
\[
\|\mathcal T_h^{(k)}f-\mathcal T_h^{(k)}g\|_\infty
\le
\|f-g\|_\infty.
\]

\end{lemma}
\begin{proof}

Fix \(\xi\in\Xi\). Define the action-value vectors
\[
q_\omega:= r_h^{(k)}(\xi,\omega)+P_h^{(k)}(\cdot\mid\xi,\omega)^\top f,
\qquad
q'_\omega:= r_h^{(k)}(\xi,\omega)+P_h^{(k)}(\cdot\mid\xi,\omega)^\top g.
\]
By Lemma~\ref{lem:log-sum-exp-nonexpansive},
\[
|(\mathcal T_h^{(k)}f)(\xi)-(\mathcal T_h^{(k)}g)(\xi)|
\le
\max_{\omega}
\left|
P_h^{(k)}(\cdot\mid\xi,\omega)^\top(f-g)
\right|.
\]
Because \(P_h^{(k)}(\cdot\mid\xi,\omega)\) is a probability distribution, 
\[
\left|
P_h^{(k)}(\cdot\mid\xi,\omega)^\top(f-g)
\right|
\le
\|f-g\|_\infty.
\]
Taking the supremum over all \(\xi\) gives
\[
\|\mathcal T_h^{(k)}f-\mathcal T_h^{(k)}g\|_\infty
\le
\|f-g\|_\infty.
\]
This proves the claim.
\end{proof}

\subsubsection{Boundedness of Soft Value Functions}
\label{app:tools:value-boundedness}

\begin{lemma}[Finite-Horizon Boundedness of Soft Value Functions]
\label{lem:soft-value-boundedness}

For any episode \(k\), stage \(h\), state \(\xi\), 
\[
0\le V_h^{\star,(k)}(\xi)\le H-h+1.
\]
Therefore
\[
\|V_h^{\star,(k)}\|_\infty \le H-h+1.
\]

\end{lemma}
\begin{proof}

The terminal condition \(V_{H+1}^{\star,(k)}\equiv 0\) clearly holds. We now use backward induction. 

Assume that \(h+1\) has already been established
\[
0\le V_{h+1}^{\star,(k)}(\xi)\le H-h,\qquad \forall \xi.
\]
Then for any \((\xi,\omega)\), 
\[
0\le r_h^{(k)}(\xi,\omega)+P_h^{(k)}(\cdot\mid\xi,\omega)^\top V_{h+1}^{\star,(k)}
\le 1+(H-h)=H-h+1.
\]
Since \(\alpha>0\) and \(\pi_{\rm ref}\) has full support, Proposition~\ref{prop:kl-bellman-optimality} gives
\[
V_h^{\star,(k)}(\xi)
=
\max_{\pi_h(\cdot\mid\xi)}
\left\{
\langle \pi_h, Q_h^{\star,(k)}(\xi,\cdot)\rangle
-\alpha \mathrm{KL}(\pi_h\|\pi_{\rm ref}(\cdot\mid\xi))
\right\}.
\]
Because \(\mathrm{KL}\ge 0\), the above expression does not exceed \(\max_\omega Q_h^{\star,(k)}(\xi,\omega)\le H-h+1\). 

On the other hand, take \(\pi_h=\pi_{\rm ref}(\cdot\mid\xi)\), then the KL term is \(0\). Thus
\[
V_h^{\star,(k)}(\xi)
\ge
\sum_\omega \pi_{\rm ref}(\omega\mid\xi)Q_h^{\star,(k)}(\xi,\omega)
\ge 0.
\]
The induction is complete. This proves the claim. 

\end{proof}

\begin{remark}[Boundedness convention for optimal and executed values]
Lemma~\ref{lem:soft-value-boundedness} applies to the optimal
KL-regularized value \(V_h^{\star,(k)}\). The same nonnegative
\([0,H-h+1]\) envelope should not be used for an arbitrary policy value
\(V_h^{\pi,(k)}\), because the KL penalty may make the one-step
regularized reward negative. The main proof separately uses the result
obtained from optimistic clipping,
\[
0\le \widetilde Q_h^k(\xi,\omega)\le H-h+1,
\qquad
0\le \widetilde V_h^k(\xi)\le H-h+1.
\]
In the martingale-control proof, we only use a conservative envelope for
the particular executed policy \(\hat\pi^k\), derived from the clipped
optimistic softmax identity:
\[
\alpha\log\frac{\hat\pi_h^k(\omega\mid\xi)}
{\pi_{\rm ref}(\omega\mid\xi)}
=
\widetilde Q_h^k(\xi,\omega)-\widetilde V_h^k(\xi).
\]
Since both clipped optimistic quantities lie in \([0,H]\), this implies
\[
\left|
 r_h^{(k)}(\xi,\omega)
 -\alpha\log\frac{\hat\pi_h^k(\omega\mid\xi)}
 {\pi_{\rm ref}(\omega\mid\xi)}
\right|
\le 2H.
\]
Consequently,
\[
|V_h^{\hat\pi^k,(k)}(\xi)|
\le 2H(H-h+1)\le 2H^2.
\]
\end{remark}

\subsubsection{Bellman Stability under Model Perturbation}
\label{app:tools:bellman-stability}

The results in this subsection are included to connect the notation with standard non-stationary MDP variation arguments. They are not invoked in the proof of Theorem~\ref{thm:dynamic-regret-fixed-window}; the fixed-window regret proof uses the conditional-mean drift bound in Lemma~\ref{lem:conditional-mean-model-drift} together with the window counting bound \eqref{eq:window-drift-sum}.

\begin{lemma}[Bellman Backup Stability under One-Step Model Perturbation]
\label{lem:bellman-perturbation}

For any \(h\in[H]\), arbitrary function \(f:\Xi\to\mathbb R\), 
\[
\|\mathcal T_h^{(k+1)}f-\mathcal T_h^{(k)}f\|_\infty
\le
\Delta_r^{(k)} + \Delta_p^{(k)}\|f\|_\infty.
\]

\end{lemma}
\begin{proof}

Fix \(\xi\in\Xi\). Define the action-values
\[
\begin{aligned}
q_\omega
&:=
r_h^{(k+1)}(\xi,\omega)
+
P_h^{(k+1)}(\cdot\mid\xi,\omega)^\top f,\\
q'_\omega
&:=
r_h^{(k)}(\xi,\omega)
+
P_h^{(k)}(\cdot\mid\xi,\omega)^\top f.
\end{aligned}
\]
By Lemma~\ref{lem:log-sum-exp-nonexpansive},
\[
|(\mathcal T_h^{(k+1)}f)(\xi)-(\mathcal T_h^{(k)}f)(\xi)|
\le
\max_\omega |q_\omega-q'_\omega|.
\]
For each \(\omega\), 
\[
|q_\omega-q'_\omega|
\le
|r_h^{(k+1)}(\xi,\omega)-r_h^{(k)}(\xi,\omega)|
+
\left|
\big(P_h^{(k+1)}-P_h^{(k)}\big)(\cdot\mid\xi,\omega)^\top f
\right|.
\]
The first term is at most \(\Delta_r^{(k)}\). The second term follows from Hölder's inequality
\[
\left|
\big(P_h^{(k+1)}-P_h^{(k)}\big)^\top f
\right|
\le
\|P_h^{(k+1)}(\cdot\mid\xi,\omega)-P_h^{(k)}(\cdot\mid\xi,\omega)\|_1 \,\|f\|_\infty
\le
\Delta_p^{(k)}\|f\|_\infty.
\]
Thus, for each \(\xi\), 
\[
|(\mathcal T_h^{(k+1)}f)(\xi)-(\mathcal T_h^{(k)}f)(\xi)|
\le
\Delta_r^{(k)}+\Delta_p^{(k)}\|f\|_\infty.
\]
Taking the supremum over \(\xi\) gives the conclusion. This proves the claim. 

\end{proof}
\begin{theorem}[Soft Bellman Stability across Episodes]
\label{thm:soft-bellman-stability}

For any adjacent episodes \(k\) and \(k+1\), and any stage \(h\in[H]\), 
\[
\|V_h^{\star,(k+1)}-V_h^{\star,(k)}\|_\infty
\le
(H-h+1)\Delta_r^{(k)}
+
\frac{(H-h)(H-h+1)}{2}\Delta_p^{(k)}.
\]
In particular, a coarse form is
\[
\|V_h^{\star,(k+1)}-V_h^{\star,(k)}\|_\infty
\le
(H-h+1)\big(\Delta_r^{(k)}+H\Delta_p^{(k)}\big).
\]

\end{theorem}
\begin{proof}

Define
\[
\delta_h^{(k)}:=\|V_h^{\star,(k+1)}-V_h^{\star,(k)}\|_\infty.
\]
From the terminal condition we know
\[
\delta_{H+1}^{(k)}=0.
\]
For any \(h\le H\), 
\[
V_h^{\star,(k+1)}=\mathcal T_h^{(k+1)}V_{h+1}^{\star,(k+1)},
\qquad
V_h^{\star,(k)}=\mathcal T_h^{(k)}V_{h+1}^{\star,(k)}.
\]
Therefore
\[
\delta_h^{(k)}
\le
\underbrace{
\|\mathcal T_h^{(k+1)}V_{h+1}^{\star,(k+1)}-\mathcal T_h^{(k+1)}V_{h+1}^{\star,(k)}\|_\infty
}_{\text{(I)}}
+
\underbrace{
\|\mathcal T_h^{(k+1)}V_{h+1}^{\star,(k)}-\mathcal T_h^{(k)}V_{h+1}^{\star,(k)}\|_\infty
}_{\text{(II)}}.
\]
For term (I), By Lemma~\ref{lem:soft-bellman-nonexpansive}, 
\[
\text{(I)}\le \delta_{h+1}^{(k)}.
\]
For term (II), By Lemmas~\ref{lem:bellman-perturbation} and~\ref{lem:soft-value-boundedness}, 
\[
\text{(II)}
\le
\Delta_r^{(k)}+\Delta_p^{(k)}\|V_{h+1}^{\star,(k)}\|_\infty
\le
\Delta_r^{(k)}+(H-h)\Delta_p^{(k)}.
\]
Thus we obtain the recursion
\[
\delta_h^{(k)}
\le
\delta_{h+1}^{(k)}+\Delta_r^{(k)}+(H-h)\Delta_p^{(k)}.
\]
Iterating from \(h\) to \(H\), and using \(\delta_{H+1}^{(k)}=0\), gives
\[
\delta_h^{(k)}
\le
\sum_{j=h}^H \Delta_r^{(k)}
+
\sum_{j=h}^H (H-j)\Delta_p^{(k)}.
\]
The first term equals
\[
(H-h+1)\Delta_r^{(k)}.
\]
The second term equals
\[
\sum_{m=0}^{H-h} m \,\Delta_p^{(k)}
=
\frac{(H-h)(H-h+1)}{2}\Delta_p^{(k)}.
\]
Combining gives
\[
\delta_h^{(k)}
\le
(H-h+1)\Delta_r^{(k)}
+
\frac{(H-h)(H-h+1)}{2}\Delta_p^{(k)}.
\]
This proves the claim. 

\end{proof}
\begin{corollary}[Cumulative Value-Function Drift Bound]
\label{cor:cumulative-value-drift}

Summing at the initial stage \(h=1\) gives
\[
\sum_{k=1}^{K-1}
\|V_1^{\star,(k+1)}-V_1^{\star,(k)}\|_\infty
\le
H B_r + \frac{H(H-1)}{2} B_p.
\]
Therefore in big \(O\) notation, 
\[
\sum_{k=1}^{K-1}
\|V_1^{\star,(k+1)}-V_1^{\star,(k)}\|_\infty
=
O\big(H(B_r+H B_p)\big).
\]
This corollary is included as explanatory context; the proof of
Theorem~\ref{thm:dynamic-regret-fixed-window} uses the
conditional-mean drift bound below rather than this cumulative
value-drift comparison directly.
\end{corollary}
\begin{proof}

Taking \(h=1\) in Theorem~\ref{thm:soft-bellman-stability} gives, for
each \(k=1,\dots,K-1\),
\[
\|V_1^{\star,(k+1)}-V_1^{\star,(k)}\|_\infty
\le
H\Delta_r^{(k)}
+
\frac{H(H-1)}{2}\Delta_p^{(k)}.
\]
Summing over \(k=1,\dots,K-1\) and using
\(B_r=\sum_{k=1}^{K-1}\Delta_r^{(k)}\) and
\(B_p=\sum_{k=1}^{K-1}\Delta_p^{(k)}\) gives
\[
\sum_{k=1}^{K-1}
\|V_1^{\star,(k+1)}-V_1^{\star,(k)}\|_\infty
\le
H B_r+\frac{H(H-1)}{2}B_p.
\]
The displayed big \(O\) form follows directly. This proves the claim.

\end{proof}

\subsection{Concentration, Filtration, and Window-Model Details}
\label{app:concentration-filtration-details}

This subsection collects the sampling, filtration, concentration, and auxiliary martingale results used by the proof. The concentration statements are formulated for the sliding-window empirical model and its conditional-mean counterpart.

\subsubsection{Sampling Model and Filtration}
\label{app:problem-setup:sampling-filtration}

The following interactive sampling model is used throughout. Let
\[
\mathcal F_{k,h}^{-}
\]
denote, at episode \(k\), stage \(h\), when the current state
\(\xi_h^k\) and current action \(\omega_h^k\) have already been
determined, but the immediate reward \(r_h^k\) and next state
\(\xi_{h+1}^k\) have not yet been revealed, the history
\(\sigma\)--algebra. If
\[
(\xi_h^k,\omega_h^k)=(\xi,\omega),
\]
then the reward observation satisfies
\[
\mathbb E\!\left[r_h^k\mid \mathcal F_{k,h}^{-}\right]
=
r_h^{(k)}(\xi,\omega),
\qquad
r_h^k\in[0,1],
\]
and the next state satisfies
\[
\mathbb P\!\left(
\xi_{h+1}^k=\xi'
\mid
\mathcal F_{k,h}^{-}
\right)
=
P_h^{(k)}(\xi'\mid\xi,\omega),
\qquad \forall \xi'\in\Xi.
\]
Here conditional independence between reward noise and transition
sampling is not required; the subsequent proofs use only the above
conditional mean and conditional law. The window counts, empirical
model, bonus, optimistic value, and executed policy constructed by the
algorithm at the beginning of episode \(k\) depend only on the data from
episodes \(1,\dots,k-1\).

\paragraph{Reward notation convention.}
Throughout the appendix, \(r_h^{(k)}(\xi,\omega)\) denotes the conditional
mean reward function of the episode-\(k\) environment, whereas \(r_h^k\)
denotes the realized reward observation collected at stage \(h\) of
episode \(k\). Thus the superscript \((k)\) indexes the environment model,
while the superscript \(k\) without parentheses indexes the realized
trajectory.

\paragraph{Global filtration convention.} 
Throughout, a global increasing
filtration is assumed to exist, revealing randomness in the natural time
order of episodes and stages. At the beginning of each episode \(k\) the
initial state \(\xi_1^k\sim\mu_1\) under the given fixed \(\mu_1\) is
independent of the past history. At the beginning of
episode \(k\), the algorithm constructs \(\hat\pi^k\) and all planning
objects using only the completed history of episodes \(1,\ldots,k-1\).
At stage \(h\), given the current state \(\xi_h^k\), the action is
sampled according to
\[
\omega_h^k\sim \hat\pi_h^k(\cdot\mid\xi_h^k)
\]
and is already determined before the reward and next state are
revealed; the resulting pre-outcome \(\sigma\)--algebra is exactly the
above \(\mathcal F_{k,h}^{-}\). Conditional on \(\mathcal F_{k,h}^{-}\),
the reward is only required to satisfy the above conditional mean, and
the next state is only required to satisfy the above conditional law.
The pair \((\mathcal F_t,\mathcal G_t)\) used in Proposition~\ref{prop:unified-martingale-control} is the linearized version of this global filtration under the global time index \(t=(k-1)H+h\).

\paragraph{Environment predictability convention.}
Throughout the appendix, the nonstationary MDP sequence
\[
\{\mathcal M^{(k)}\}_{k=1}^K
\]
is either fixed before the interaction begins, or more generally is
predictable at the beginning of each episode. Namely, for every episode
\(k\), the collection
\[
\{r_h^{(k)}(\xi,\omega),P_h^{(k)}(\cdot\mid\xi,\omega):
h\in[H],\xi\in\Xi,\omega\in\Omega(\xi)\}
\]
is measurable with respect to the beginning-of-episode \(k\)
\(\sigma\)-algebra and is not allowed to depend on the current
within-episode reward or transition randomness. Conditional on the
pre-outcome \(\sigma\)-algebra \(\mathcal F_{k,h}^{-}\), the only
randomness used in the sampling model is the current reward observation
and next-state draw satisfying the displayed conditional mean and
conditional law.

\subsubsection{Martingale Auxiliary Inequality}
\label{app:tools:azuma-hoeffding}

\begin{lemma}[Azuma--Hoeffding Inequality]
\label{lem:azuma-hoeffding}
Let \((Y_t,\mathcal F_t)_{t=1}^n\) be a martingale difference sequence
such that \(|Y_t|\le B\) almost surely for every \(t\). Then, for every
\(u>0\),
\[
\mathbb P\!\left(
\left|\sum_{t=1}^n Y_t\right|\ge u
\right)
\le
2\exp\!\left(
-\frac{u^2}{2B^2n}
\right).
\]
\end{lemma}

\begin{proof}
For each \(\lambda\in\mathbb R\), Hoeffding's lemma applied conditionally on
\(\mathcal F_{t-1}\) gives
\[
\mathbb E\!\left[\exp(\lambda Y_t)\mid\mathcal F_{t-1}\right]
\le
\exp\!\left(\frac{\lambda^2B^2}{2}\right),
\]
because \(\mathbb E[Y_t\mid\mathcal F_{t-1}]=0\) and \(|Y_t|\le B\) almost
surely. Iterating this inequality over \(t=1,\ldots,n\) yields
\[
\mathbb E\exp\!\left(\lambda\sum_{t=1}^nY_t\right)
\le
\exp\!\left(\frac{n\lambda^2B^2}{2}\right).
\]
By Markov's inequality, for any \(u>0\) and \(\lambda>0\),
\[
\mathbb P\!\left(\sum_{t=1}^nY_t\ge u\right)
\le
\exp\!\left(-\lambda u+\frac{n\lambda^2B^2}{2}\right).
\]
Choosing \(\lambda=u/(nB^2)\) gives
\[
\mathbb P\!\left(\sum_{t=1}^nY_t\ge u\right)
\le
\exp\!\left(-\frac{u^2}{2nB^2}\right).
\]
Applying the same argument to \(-Y_t\) and taking a union bound proves the
claim.
\end{proof}

\subsubsection{Predictable-Skipping Concentration}
\label{app:tools:predictable-skipping}

This subsection handles the concentration problem caused by random hit
locations. The core requirement is: the hit indicator must be measurable
before the current reward / next-state outcome is revealed; the proof
must not condition on the complete hit-location vector.

\begin{lemma}[Predictable-Skipping Concentration with Stopped Filtration]
\label{lem:predictable-skipping-concentration}
Let \(M\ge1\). Suppose there exists a two-level filtration
\[
\mathcal H_0\subseteq \mathcal G_1\subseteq \mathcal H_1
\subseteq \mathcal G_2\subseteq \mathcal H_2\subseteq\cdots
\subseteq \mathcal G_M\subseteq\mathcal H_M,
\]
where \(\mathcal G_t\) denotes the pre-outcome information for the \(t\)-th candidate sample and \(\mathcal H_t\) denotes the corresponding post-outcome information. 
In particular, the displayed nesting implies that
\[
\mathcal H_s\subseteq \mathcal G_t,
\qquad 0\le s<t\le M,
\]
with the convention that \(\mathcal H_0\subseteq\mathcal G_t\) for all \(t\ge1\).
This is the temporal measurability relation used below.
\[
I_t\in\{0,1\},\qquad I_t\in\mathcal G_t,
\]
and suppose the random variable \(X_t\) satisfies
\[
X_t\in\mathcal H_t,
\qquad
X_t=I_tX_t,
\qquad
|X_t|\le B,
\qquad
\mathbb E[X_t\mid\mathcal G_t]=0.
\]
Define the \(n\)-th selected location
\[
\tau_n:=\inf\left\{t\in[M]:\sum_{s=1}^t I_s\ge n\right\},
\]
If the set is empty, set \(\tau_n=\infty\). Let
\[
Z_n:=
\begin{cases}
X_{\tau_n}, & \tau_n<\infty,\\
0, & \tau_n=\infty.
\end{cases}
\]
Thus the undefined quantity \(X_\infty\) is never invoked. Then there exists a stopped filtration \((\mathcal K_n)_{n=0}^M\), such that \((Z_n,\mathcal K_n)\) is \(B\)-bounded martingale difference sequence. Furthermore, for any \(\delta'\in(0,1)\), with probability at least \(1-\delta'\), simultaneously for all \(1\le n\le M\) 
\[
\left|\sum_{m=1}^n Z_m\right|
\le
2B\sqrt{n\log(2M/\delta')}.
\]
Therefore, if \(N:=\sum_{t=1}^M I_t\ge1\), then
\[
\left|
\frac1N\sum_{t=1}^M I_tX_t
\right|
=\left|\frac1N\sum_{m=1}^N Z_m\right|
\le
2B\sqrt{\frac{\log(2M/\delta')}{N}}.
\]

\end{lemma}
\begin{proof}

Let \(\mathsf\Omega\) denote the underlying probability sample space in
this proof, to avoid conflict with the action-set notation
\(\Omega(\xi)\).

First, because \(I_t\in\mathcal G_t\), the event
\[
\{\tau_n\le t\}
=
\left\{\sum_{s=1}^t I_s\ge n\right\}
\]
belongs to \(\mathcal G_t\). Therefore \(\tau_n\) is a stopping time with respect to the pre-outcome filtration \((\mathcal G_t)_{t=1}^M\), and \(\{\tau_n=t\}\in\mathcal G_t\). 

Define the stopped \(\sigma\)-algebra
\[
\mathcal K_n:=
\left\{
A:\ 
A\cap\{\tau_n=t\}\in\mathcal H_t\ \forall t\in[M],
\quad
A\cap\{\tau_n=\infty\}\in\mathcal H_M
\right\},
\]
and set \(\mathcal K_0:=\mathcal H_0\). Next verify \((\mathcal K_n)_{n=0}^M\) is a filtration. 

First, \(\mathcal K_n\) is a \(\sigma\)-algebra. Indeed, \(\mathsf\Omega\in\mathcal K_n\) and \(\varnothing\in\mathcal K_n\) clearly hold, because \(\{\tau_n=t\}\in\mathcal G_t\subseteq\mathcal H_t\) and \(\{\tau_n=\infty\}\in\mathcal G_M\subseteq\mathcal H_M\). If \(A\in\mathcal K_n\), then for each \(t\in[M]\), \(A^c\cap\{\tau_n=t\}=\{\tau_n=t\}\setminus(A\cap\{\tau_n=t\})\in\mathcal H_t\), and \(A^c\cap\{\tau_n=\infty\}\in\mathcal H_M\); if \(A_i\in\mathcal K_n\), then \((\bigcup_i A_i)\cap\{\tau_n=t\}=\bigcup_i(A_i\cap\{\tau_n=t\})\in\mathcal H_t\), and \((\bigcup_i A_i)\cap\{\tau_n=\infty\}\in\mathcal H_M\). Therefore \(\mathcal K_n\) is a \(\sigma\)-algebra. 

Next prove monotonicity. If \(A\in\mathcal K_{n-1}\), then for any \(t\in[M]\):
\begin{itemize}
\item When \(n=1\), \(A\in\mathcal H_0\subseteq\mathcal H_t\), and
\(\{\tau_1=t\}\in\mathcal G_t\subseteq\mathcal H_t\), so
\(A\cap\{\tau_1=t\}\in\mathcal H_t\); at the same time
\(A\cap\{\tau_1=\infty\}\in\mathcal H_M\).
\item When \(n\ge2\),
\[
A\cap\{\tau_n=t\}
=
\bigcup_{s=1}^{t-1}
\bigl(A\cap\{\tau_{n-1}=s\}\cap\{\tau_n=t\}\bigr).
\]
Since \(A\in\mathcal K_{n-1}\), \(A\cap\{\tau_{n-1}=s\}\in\mathcal H_s\subseteq\mathcal H_t\); at the same time \(\{\tau_n=t\}\in\mathcal G_t\subseteq\mathcal H_t\). Thus each union term belongs to \(\mathcal H_t\).
\end{itemize}
On the other hand, 
\[
A\cap\{\tau_n=\infty\}
=
\bigl(A\cap\{\tau_{n-1}=\infty\}\bigr)
\cup
\bigcup_{s=1}^{M}
\bigl(A\cap\{\tau_{n-1}=s\}\cap\{\tau_n=\infty\}\bigr)
\in\mathcal H_M,
\]
because \(A\cap\{\tau_{n-1}=\infty\}\in\mathcal H_M\), \(A\cap\{\tau_{n-1}=s\}\in\mathcal H_s\subseteq\mathcal H_M\), and \(\{\tau_n=\infty\}\in\mathcal G_M\subseteq\mathcal H_M\). 

Therefore \(A\in\mathcal K_n\), and hence \(\mathcal K_{n-1}\subseteq\mathcal K_n\). Thus \((\mathcal K_n)_{n=0}^M\) is a filtration. 

Next verify \(Z_n\) measurability. For any Borel set \(E\subseteq\mathbb R\), if \(0\notin E\), then
\[
\{Z_n\in E\}\cap\{\tau_n=\infty\}=\varnothing\in\mathcal H_M,
\]
If \(0\in E\), then
\[
\{Z_n\in E\}\cap\{\tau_n=\infty\}=\{\tau_n=\infty\}\in\mathcal H_M.
\]
At the same time, for each \(t\in[M]\), 
\[
\{Z_n\in E\}\cap\{\tau_n=t\}
=
\{X_t\in E\}\cap\{\tau_n=t\}
\in\mathcal H_t,
\]
because \(X_t\in\mathcal H_t\) and \(\{\tau_n=t\}\in\mathcal G_t\subseteq\mathcal H_t\). Thus \(\{Z_n\in E\}\in\mathcal K_n\), \(Z_n\) is measurable with respect to \(\mathcal K_n\). Also, by \(|X_t|\le B\) and \(Z_n=0\) on \(\{\tau_n=\infty\}\), \(|Z_n|\le B\). 

Next prove zero mean. Fix \(A\in\mathcal K_{n-1}\). We first prove that for each \(t\), 
\[
A\cap\{\tau_n=t\}\in\mathcal G_t.
\]
If \(n=1\), then \(A\in\mathcal K_0=\mathcal H_0\subseteq\mathcal G_t\), and \(\{\tau_1=t\}\in\mathcal G_t\). Thus the claim holds. If \(n\ge2\), then on \(\{\tau_n=t\}\) the \(n-1\) hits must occur at some \(s<t\), therefore
\[
A\cap\{\tau_n=t\}
=
\bigcup_{s=1}^{t-1}
\left(
A\cap\{\tau_{n-1}=s\}\cap\{\tau_n=t\}
\right).
\]
Since \(A\in\mathcal K_{n-1}\),
\[
A\cap\{\tau_{n-1}=s\}\in\mathcal H_s.
\]
Also, because \(s<t\), we have \(\mathcal H_s\subseteq\mathcal G_t\), and because \(\{\tau_n=t\}\in\mathcal G_t\), each term in the finite union belongs to \(\mathcal G_t\). Hence the above measurability holds.

Therefore
\[
\mathbb E[\mathbf 1_A Z_n]
=
\sum_{t=1}^M
\mathbb E\!\left[
\mathbf 1_{A\cap\{\tau_n=t\}}X_t
\right]
=
\sum_{t=1}^M
\mathbb E\!\left[
\mathbf 1_{A\cap\{\tau_n=t\}}
\mathbb E[X_t\mid\mathcal G_t]
\right]
=0.
\]
Thus \(\mathbb E[Z_n\mid\mathcal K_{n-1}]=0\), and \((Z_n,\mathcal K_n)\) is a \(B\)-bounded martingale difference sequence. The branch \(\tau_n=\infty\) contributes zero and therefore does not produce \(X_\infty\) or any nonmeasurable term.

For fixed \(n\), Lemma~\ref{lem:azuma-hoeffding} gives
\[
\mathbb P\!\left(
\left|\sum_{m=1}^n Z_m\right|\ge u
\right)
\le
2\exp\!\left(-\frac{u^2}{2B^2n}\right).
\]
Taking
\[
u=2B\sqrt{n\log(2M/\delta')}
\]
makes the right-hand side at most \(\delta'/M\). A union bound over \(n=1,\dots,M\) gives simultaneous control of all partial sums. Finally, if \(N\ge1\), taking \(n=N\) gives the empirical mean bound. If \(N=0\), no empirical average is formed; the subsequent \(N=0\) cases are covered by the cold-start bonus clipping mechanism. This proves the claim.
\end{proof}

\subsubsection{Reward and Transition Concentration}
\label{app:tools:reward-transition-concentration}

Before stating the concentration events, define the conditional-mean
model around which the empirical window estimates concentrate. For any
\((k,h,\xi,\omega)\), let
\[
N_h^k(\xi,\omega)
=
\sum_{j\in\mathcal W_k}
\mathbf 1\{(\xi_h^j,\omega_h^j)=(\xi,\omega)\}.
\]
If \(N_h^k(\xi,\omega)\ge 1\), define
\[
\breve r_h^k(\xi,\omega)
:=
\frac{1}{N_h^k(\xi,\omega)}
\sum_{j\in\mathcal W_k}
r_h^{(j)}(\xi,\omega)
\mathbf 1\{(\xi_h^j,\omega_h^j)=(\xi,\omega)\},
\]
and
\[
\breve P_h^k(\cdot\mid\xi,\omega)
:=
\frac{1}{N_h^k(\xi,\omega)}
\sum_{j\in\mathcal W_k}
P_h^{(j)}(\cdot\mid\xi,\omega)
\mathbf 1\{(\xi_h^j,\omega_h^j)=(\xi,\omega)\}.
\]
If \(N_h^k(\xi,\omega)=0\), set
\[
\breve r_h^k(\xi,\omega):=0,
\qquad
\breve P_h^k(\cdot\mid\xi,\omega)
:=
P_h^\circ(\cdot\mid\xi,\omega),
\]
where \(P_h^\circ(\cdot\mid\xi,\omega)\in\Delta(\Xi)\) is fixed before
learning begins. These default values are not used in the concentration
statements, which are only asserted for \(N_h^k(\xi,\omega)\ge1\), but
they keep the notation globally single-valued.

\begin{lemma}[Reward Concentration around the Conditional Mean]
\label{lem:reward-concentration}
There exists an absolute constant \(c_r>0\) such that, with probability at least \(1-\delta/4\), uniformly over all
\[
k\in[K],\quad h\in[H],\quad \xi\in\Xi,\quad \omega\in\Omega(\xi)
\]
quadruples satisfying \(N_h^k(\xi,\omega)\ge1\), we have
\[
\bigl|\hat r_h^k(\xi,\omega)-\breve r_h^k(\xi,\omega)\bigr|
\le
c_r\sqrt{\frac{\log(C_{\log}S_\xi AHK/\delta)}{(N_h^k(\xi,\omega)+1)\wedge W}}.
\]
The concentration event is stated for positive empirical counts. The zero-count regime is covered separately by the clipped cold-start bonus, which is the mechanism used in the near-optimism and certificate arguments. 

\end{lemma}
\begin{proof}

Fix \(k,h,\xi,\omega\), and denote the episodes in the window as
\[
u_1<u_2<\cdots<u_{m_k},\qquad \{u_1,\dots,u_{m_k}\}=\mathcal W_k,
\qquad m_k=|\mathcal W_k|\le W.
\]
If \(m_k=0\), then there are no quadruples with \(N_h^k(\xi,\omega)\ge1\), so the assertion is empty. We therefore assume \(m_k\ge1\). 

Let \(\mathcal G_s\) be the pre-outcome \(\sigma\)-algebra at stage \(h\) of episode \(u_s\), after the state and action have been determined but before the reward and next state have been revealed. Let
\[
\mathcal H_s:=\mathcal G_s\vee\sigma(r_h^{u_s},\xi_{h+1}^{u_s}).
\]
By the global filtration convention, \(\mathcal H_{s-1}\subseteq\mathcal G_s\). Define
\[
I_s:=\mathbf 1\{(\xi_h^{u_s},\omega_h^{u_s})=(\xi,\omega)\},
\qquad
X_s:=I_s\bigl(r_h^{u_s}-r_h^{(u_s)}(\xi,\omega)\bigr).
\]
Then \(I_s\in\mathcal G_s\), \(X_s\in\mathcal H_s\), \(X_s=I_sX_s\), \(|X_s|\le1\), and the sampling model gives
\[
\mathbb E[X_s\mid\mathcal G_s]
=
I_s\,
\mathbb E[
r_h^{u_s}-r_h^{(u_s)}(\xi,\omega)\mid\mathcal G_s]
=0.
\]
Let \(N:=N_h^k(\xi,\omega)\). If \(N\ge1\), then
\[
\hat r_h^k(\xi,\omega)-\breve r_h^k(\xi,\omega)
=
\frac1N\sum_{s=1}^{m_k} I_s\bigl(r_h^{u_s}-r_h^{(u_s)}(\xi,\omega)\bigr).
\]
The number of possible quadruples is at most
\[
K\cdot H\cdot\sum_{\xi\in\Xi}|\Omega(\xi)|
\le KHS_\xi A.
\]
Thus the following failure-probability allocation is sufficient for a union
bound over all \((k,h,\xi,\omega)\).

For fixed \((k,h,\xi,\omega)\), apply Lemma~\ref{lem:predictable-skipping-concentration} with \(B=1\), \(M=m_k\le W\) and
\[
\delta'=
\frac{\delta}{4KS_\xi AH}.
\]
By \(m_k\le W\), 
\[
\log(2m_k/\delta')\le\log(8WKS_\xi AH/\delta).
\]
We obtain, for \(N\ge1\),
\[
\bigl|\hat r_h^k(\xi,\omega)-\breve r_h^k(\xi,\omega)\bigr|
\le
2\sqrt{\frac{\log(8WKS_\xi AH/\delta)}{N}}.
\]
Taking a union bound over all \((k,h,\xi,\omega)\) and using \(1\le W\le K\), for an absolute constant \(c>0\), after choosing \(C_{\log}\) sufficiently large,
\[
\log(8WKS_\xi AH/\delta)
\le
c\,\log(C_{\log}S_\xi AHK/\delta).
\]
At the same time, when \(N\ge1\) and \(N\le W\), 
\[
(N+1)\wedge W\le 2N,
\qquad\text{therefore}\qquad
\frac1N\le \frac{2}{(N+1)\wedge W}.
\]
Absorbing absolute constants and taking the supremum over all \(N\ge1\) gives the stated bound. This proves the claim.

\end{proof}

This subsection establishes \(\ell_1\)-concentration for \(\hat P_h^k(\cdot\mid\xi,\omega)-\breve P_h^k(\cdot\mid\xi,\omega)\). Whenever it is necessary to control \((\hat P_h^k-\breve P_h^k)^\top f\), we use
\[
|(\hat P_h^k-\breve P_h^k)^\top f|
\le
\|f\|_\infty
\|\hat P_h^k-\breve P_h^k\|_1.
\]
We do not need a separate concentration result for the random test function \(\widetilde V\).

\begin{lemma}[Transition Concentration around the Conditional Mean]
\label{lem:transition-concentration}
There exists an absolute constant \(c_P>0\) such that, with probability at least \(1-\delta/4\), uniformly over all
\[
k\in[K],\quad h\in[H],\quad \xi\in\Xi,\quad \omega\in\Omega(\xi)
\]
quadruples satisfying \(N_h^k(\xi,\omega)\ge1\), we have
\[
\bigl\|
\hat P_h^k(\cdot\mid\xi,\omega)-\breve P_h^k(\cdot\mid\xi,\omega)
\bigr\|_1
\le
c_P\sqrt{\frac{S_\xi\log(C_{\log}S_\xi AHK/\delta)}{(N_h^k(\xi,\omega)+1)\wedge W}}.
\]
The concentration event is stated for positive empirical counts. The zero-count regime is covered by the same clipped cold-start bonus mechanism used in the subsequent optimism and certificate bounds. 

\end{lemma}
\begin{proof}

Fix \(k,h,\xi,\omega\), and follow the same candidate episode order \(u_1<\cdots<u_{m_k}\) and filtration construction as in Lemma~\ref{lem:reward-concentration}. If \(m_k=0\), then there are no quadruples with \(N_h^k(\xi,\omega)\ge1\), so the assertion is empty. We therefore assume \(m_k\ge1\). 

For each \(v\in\{-1,1\}^{S_\xi}\), define
\[
I_s:=\mathbf 1\{(\xi_h^{u_s},\omega_h^{u_s})=(\xi,\omega)\},
\]
\[
X_s^v:=
I_s\,v^\top\left(
e_{\xi_{h+1}^{u_s}}-P_h^{(u_s)}(\cdot\mid\xi,\omega)
\right).
\]
Then \(I_s\in\mathcal G_s\), \(X_s^v\in\mathcal H_s\), and
\[
|X_s^v|
\le
|v^\top e_{\xi_{h+1}^{u_s}}|
+
|v^\top P_h^{(u_s)}(\cdot\mid\xi,\omega)|
\le 2.
\]
On the event \(I_s=1\), the sampling model gives
\[
\mathbb E[
e_{\xi_{h+1}^{u_s}}-P_h^{(u_s)}(\cdot\mid\xi,\omega)
\mid\mathcal G_s]=0,
\]
whereas on the event \(I_s=0\), \(X_s^v=0\). Therefore
\[
\mathbb E[X_s^v\mid\mathcal G_s]=0.
\]
Let \(N:=N_h^k(\xi,\omega)\). If \(N\ge1\), then
\[
v^\top\!\left(
\hat P_h^k(\cdot\mid\xi,\omega)-\breve P_h^k(\cdot\mid\xi,\omega)
\right)
=
\frac1N\sum_{s=1}^{m_k}X_s^v.
\]
For fixed \(v\in\{-1,1\}^{S_\xi}\), the number of possible
quadruples \((k,h,\xi,\omega)\) is at most
\[
K\cdot H\cdot\sum_{\xi\in\Xi}|\Omega(\xi)|
\le KHS_\xi A.
\]
Thus the following failure-probability allocation is sufficient for a union
bound over all quadruples and all sign vectors
\(v\in\{-1,1\}^{S_\xi}\).

For fixed \((k,h,\xi,\omega,v)\), apply Lemma~\ref{lem:predictable-skipping-concentration} with \(B=2\), \(M=m_k\le W\) and
\[
\delta'=
\frac{\delta}{4KS_\xi AH\,2^{S_\xi}}.
\]
By \(m_k\le W\), 
\[
\log(2m_k/\delta')\le
\log(8WKS_\xi AH\,2^{S_\xi}/\delta).
\]
Therefore, for \(N\ge1\),
\[
\left|
v^\top\!\left(
\hat P_h^k(\cdot\mid\xi,\omega)-\breve P_h^k(\cdot\mid\xi,\omega)
\right)
\right|
\le
4\sqrt{
\frac{\log(8WKS_\xi AH\,2^{S_\xi}/\delta)}{N}
}.
\]
Taking the supremum over all \(v\in\{-1,1\}^{S_\xi}\) and applying a union bound over all \((k,h,\xi,\omega)\), together with the finite-dimensional dual representation
\[
\|x\|_1=\max_{v\in\{-1,1\}^{S_\xi}}v^\top x,
\qquad x\in\mathbb R^{S_\xi},
\]
and using \(1\le W\le K\), for an absolute constant \(c>0\), after choosing \(C_{\log}\) sufficiently large,
\[
\log(8WKS_\xi AH\,2^{S_\xi}/\delta)
\le
c\,S_\xi\log(C_{\log}S_\xi AHK/\delta),
\]
gives the simultaneous bound for all \(N_h^k(\xi,\omega)\ge1\)
\[
\bigl\|
\hat P_h^k(\cdot\mid\xi,\omega)-\breve P_h^k(\cdot\mid\xi,\omega)
\bigr\|_1
\le
c_P\sqrt{
\frac{S_\xi\log(C_{\log}S_\xi AHK/\delta)}
{N_h^k(\xi,\omega)}
}.
\]
Since \(N:=N_h^k(\xi,\omega)\ge1\) and \(N\le W\),
\[
\frac1N\le \frac{2}{(N+1)\wedge W}.
\]
Absorbing absolute constants gives the stated bound. This proves the claim.

\end{proof}

Define
\[
\mathcal E_r:=
\left\{
\begin{aligned}
&\forall k,h,\xi,\omega\text{ with }N_h^k(\xi,\omega)\ge1,\\
&\qquad
\bigl|\hat r_h^k(\xi,\omega)-\breve r_h^k(\xi,\omega)\bigr|
\le
c_r\sqrt{\frac{\log(C_{\log}S_\xi AHK/\delta)}{(N_h^k(\xi,\omega)+1)\wedge W}}
\end{aligned}
\right\},
\]
and
\[
\mathcal E_P:=
\left\{
\begin{aligned}
&\forall k,h,\xi,\omega\text{ with }N_h^k(\xi,\omega)\ge1,\\
&\qquad
\bigl\|
\hat P_h^k(\cdot\mid\xi,\omega)-\breve P_h^k(\cdot\mid\xi,\omega)
\bigr\|_1
\le
c_P\sqrt{\frac{S_\xi\log(C_{\log}S_\xi AHK/\delta)}{(N_h^k(\xi,\omega)+1)\wedge W}}
\end{aligned}
\right\}.
\]
Define
\[
\mathcal E_{\mathrm{conc}}:=\mathcal E_r\cap\mathcal E_P.
\]

\begin{lemma}[Total Concentration Event]
\label{lem:total-concentration-event}
If \(C_{\log}\) is sufficiently large, then
\[
\mathbb P(\mathcal E_{\mathrm{conc}})\ge1-\delta/2.
\]

\end{lemma}
\begin{proof}

By Lemmas~\ref{lem:reward-concentration} and~\ref{lem:transition-concentration}, 
\[
\mathbb P(\mathcal E_r^c)\le\delta/4,
\qquad
\mathbb P(\mathcal E_P^c)\le\delta/4.
\]
A union bound gives
\[
\mathbb P(\mathcal E_{\mathrm{conc}}^c)
\le
\mathbb P(\mathcal E_r^c)+\mathbb P(\mathcal E_P^c)
\le
\delta/2.
\]
This proves the claim. 
\end{proof}

\subsection{Proof of the Fixed-Window Theorem}
\label{app:fixed-window-proof}

\subsubsection{Current True Model}
\label{app:main-proof-setup:true-model}

This section has three formal responsibilities:

\begin{enumerate}
\item
  unify the model objects and notation appearing in all subsequent
  propositions;
\item
  specify the unified high-probability event and its allocation of
  failure probabilities;
\item
  fix the same bonus interface jointly used in subsequent Proposition~\ref{prop:near-optimism-window-drift}, Lemmas~\ref{lem:pathwise-bridge-identity}--\ref{lem:pathwise-certificate-bound}, Proposition~\ref{prop:summed-bonus-bound}, and Theorem~\ref{thm:dynamic-regret-fixed-window}.
\end{enumerate}

We first unify the objects used in the stability analysis and in the dynamic-regret theorem. The proof separates two roles that are useful in non-stationary analysis: \textbf{nonstationarity enters through a window-average conditional model, while statistical concentration enters through the empirical model}. This separation keeps the optimism and certificate arguments aligned with the sliding-window estimator.

Recall from Subsection~\ref{app:problem-setup:mdp-sequence} that the true
finite-horizon KL-regularized MDP in episode \(k\in[K]\) is
\[
\mathcal M^{(k)}=
\Bigl(
\Xi,\{\Omega(\xi)\}_{\xi\in\Xi},
\{r_h^{(k)}\}_{h=1}^H,
\{P_h^{(k)}\}_{h=1}^H,
H
\Bigr),
\]
where, for any \(h\in[H]\), \(\xi\in\Xi\), and \(\omega\in\Omega(\xi)\),
\[
r_h^{(k)}(\xi,\omega)\in[0,1],
\qquad
P_h^{(k)}(\cdot\mid \xi,\omega)\in\Delta(\Xi).
\]
Recall that the corresponding soft Bellman operator
\(\mathcal T_h^{(k)}\) was defined after the KL Bellman recursion in
Subsection~\ref{app:kl-bellman:recursion}. With that notation, the
optimal value functions satisfy
\[
V_{H+1}^{\star,(k)}\equiv 0,
\qquad
V_h^{\star,(k)}=\mathcal T_h^{(k)}V_{h+1}^{\star,(k)}.
\]

\subsubsection{Primitive Drift and Window Cumulative Drift}
\label{app:main-proof-setup:window-drift}

Recall that the primitive adjacent drifts \(\Delta_r^{(j)}\) and
\(\Delta_p^{(j)}\), \(j=1,\ldots,K-1\), are defined in
Subsection~\ref{app:problem-setup:variation-budgets} immediately after
the total variation budget.

For episode \(k\), define the sliding window
\[
\mathcal W_k:=\{\max\{1,k-W\},\dots,k-1\},
\qquad
m_k:=|\mathcal W_k|.
\]

The window cumulative drift quantity used uniformly in the main proof is
defined as
\begin{equation}
\label{eq:window-drift}
D_k:=
\sum_{j=\max\{1,k-W\}}^{k-1}
\bigl(\Delta_r^{(j)}+H\Delta_p^{(j)}\bigr).
\end{equation}

Throughout the appendix, an empty sum is zero. In particular, if \(k=1\),
then \(\mathcal W_1=\varnothing\) and \(D_1=0\).

The main analysis uniformly uses \(D_k\), because it satisfies a clean
summability relation
\begin{equation}
\label{eq:window-drift-sum}
\sum_{k=1}^K D_k\le W(B_r+HB_p).
\end{equation}

\textbf{Discrete counting note. }
The above display follows directly from discrete counting. Indeed,
\[
\sum_{k=1}^K D_k
=
\sum_{k=1}^K\sum_{j=\max\{1,k-W\}}^{k-1}
\bigl(\Delta_r^{(j)}+H\Delta_p^{(j)}\bigr)
=
\sum_{j=1}^{K-1} c_j \bigl(\Delta_r^{(j)}+H\Delta_p^{(j)}\bigr),
\]
where
\[
c_j:=
\#\{k\in[K]: j\in[\max\{1,k-W\},k-1]\}.
\]
For fixed \(j\), condition \(j\in[\max\{1,k-W\},k-1]\) is equivalent to
\[
j+1\le k\le \min\{K,j+W\},
\]
Thus
\[
0\le c_j\le W.
\]
Hence
\[
\sum_{k=1}^K D_k
\le
W\sum_{j=1}^{K-1}\bigl(\Delta_r^{(j)}+H\Delta_p^{(j)}\bigr)
=
W(B_r+HB_p).
\]

\subsubsection{Window Reference Model}
\label{app:main-proof-setup:window-reference-model}

\begin{remark}[Conditional-Mean Reference Model]
\label{rem:window-reference-model}
The main proof below depends only on the conditional-mean model
\(\breve r_h^k,\breve P_h^k\), the concentration event
\(\mathcal E_{\mathrm{conc}}\), and the window drift quantity \(D_k\).
No additional window-average model is used in the proof of
Theorem~\ref{thm:dynamic-regret-fixed-window}.
\end{remark}

\subsubsection{Conditional-Mean Model}

\label{app:main-proof-setup:conditional-mean-model}

Recall that the conditional-mean reward and transition model
\(\breve r_h^k,\breve P_h^k\) was defined before the concentration
events in Subsection~\ref{app:tools:reward-transition-concentration}.
The role of this model is that the concentration of
\(\hat r_h^k,\hat P_h^k\) holds with respect to
\(\breve r_h^k,\breve P_h^k\), and the subsequent proof uses this
conditional-mean reference model together with the drift quantity \(D_k\).

\subsubsection{Empirical Model and Optimistic Planning Objects}
\label{app:main-proof-setup:empirical-optimistic-model}

The empirical model is,
\[
\hat r_h^k(\xi,\omega):=
\begin{cases}
\displaystyle
\frac{1}{N_h^k(\xi,\omega)}
\sum_{j\in\mathcal W_k}
r_h^j\,\mathbf 1\{(\xi_h^j,\omega_h^j)=(\xi,\omega)\},
& N_h^k(\xi,\omega)\ge 1,\\[1.2em]
0, & N_h^k(\xi,\omega)=0,
\end{cases}
\]
\[
\hat P_h^k(\xi'\mid\xi,\omega):=
\begin{cases}
\displaystyle
\frac{1}{N_h^k(\xi,\omega)}
\sum_{j\in\mathcal W_k}
\mathbf 1\{(\xi_h^j,\omega_h^j,\xi_{h+1}^j)=(\xi,\omega,\xi')\},
& N_h^k(\xi,\omega)\ge 1,\\[1.2em]
P_h^\circ(\xi'\mid\xi,\omega), & N_h^k(\xi,\omega)=0.
\end{cases}
\]
At the same time define
\[
\widetilde V_{H+1}^k(\xi)=0,\qquad \forall \xi\in\Xi,
\]
and recursively for \(h=H,\dots,1\)
\[
\widetilde Q_h^k(\xi,\omega):=
\min\Bigl\{
H-h+1,\;
\hat r_h^k(\xi,\omega)
+b_h^k(\xi,\omega)
+\hat P_h^k(\cdot\mid\xi,\omega)^\top \widetilde V_{h+1}^k
\Bigr\},
\]
\[
\widetilde V_h^k(\xi):=
\alpha\log\sum_{\omega\in\Omega(\xi)}
\pi_{\rm ref}(\omega\mid\xi)
\exp\!\left(\frac{\widetilde Q_h^k(\xi,\omega)}{\alpha}\right),
\]
\[
\hat\pi_h^k(\omega\mid\xi):=
\frac{
\pi_{\rm ref}(\omega\mid\xi)\exp(\widetilde Q_h^k(\xi,\omega)/\alpha)
}{
\sum_{\omega'}\pi_{\rm ref}(\omega'\mid\xi)\exp(\widetilde Q_h^k(\xi,\omega')/\alpha)
}.
\]

The bonus is
\[
b_h^k(\xi,\omega)
=
c_b H\sqrt{
\frac{S_\xi\log(C_{\log}S_\xi AHK/\delta)}{(N_h^k(\xi,\omega)+1)\wedge W}
}
\]
where \(C_{\log}\ge e^2\) is an absolute constant, \(c_b>0\) is an absolute
constant. Throughout the main proof, assume
\[
\delta\in(0,1),\qquad 1\le W\le K,\qquad S_\xi,A,H,K\ge1.
\]
Let \(c_r\) and \(c_P\) denote respectively the reward and transition concentration constants in Lemmas~\ref{lem:reward-concentration} and~\ref{lem:transition-concentration}, and let
\[
c_{\mathrm{conc}}:=c_r+c_P.
\]
Choose \(C_{\log}\) and \(c_b\) such that
\[
c_b\ge 2(c_r+c_P+1)\ge c_{\mathrm{conc}},
\qquad
c_b\sqrt{S_\xi\log(C_{\log}S_\xi AHK/\delta)}\ge 2.
\]
Hence, when \(N_h^k(\xi,\omega)=0\),
\[
b_h^k(\xi,\omega)
=c_b H\sqrt{S_\xi\log(C_{\log}S_\xi AHK/\delta)}
\ge 2H.
\]
All subsequent cold-start concentration absorption only invoke this
parameter convention, and do not change the bonus separately. Here
\((+1)\) is used to remove cold-start singularity, while \(\wedge W\) is
kept consistent with the effective sample size of the sliding window.

\textbf{Concentration absorption convention. } For any
\(n\in\{1,\dots,W\}\), because \(H\ge1\) and \(S_\xi\ge1\),
\[
\begin{aligned}
&c_r\sqrt{\frac{\log(C_{\log}S_\xi AHK/\delta)}{n}}
+
c_P H\sqrt{\frac{S_\xi\log(C_{\log}S_\xi AHK/\delta)}{n}}\\
&\qquad\le
(c_r+c_P)H\sqrt{\frac{S_\xi\log(C_{\log}S_\xi AHK/\delta)}{n}}
\le
b_h^k(\xi,\omega)
\end{aligned}
\]
whenever \(n=(N_h^k(\xi,\omega)+1)\wedge W\). Thus the statistical
reward and transition deviations are explicitly absorbed by the single
bonus term; zero-count cases use only the preceding cold-start condition.

\textbf{Unified interface convention. } From this section onward, after
Proposition~\ref{prop:near-optimism-window-drift}, Lemmas~\ref{lem:pathwise-bridge-identity}--\ref{lem:pathwise-certificate-bound}, Proposition~\ref{prop:summed-bonus-bound}, and Theorem~\ref{thm:dynamic-regret-fixed-window}
and their proofs
\[
b_h^k(\xi,\omega)
\]
all refer to the same window-truncated bonus definition above;
subsequent sections do not switch back to any other bonus notation.
The near-optimism result, certificate decomposition, summed-bonus bound, and aggregate dynamic-regret theorem all hold for the same algorithmic object.

\begin{lemma}[Boundedness of Clipped Optimistic Values]
\label{lem:optimistic-value-boundedness}
For every episode \(k\), stage \(h\in[H]\), state \(\xi\in\Xi\), and
action \(\omega\in\Omega(\xi)\), the optimistic planning objects defined in
this subsection satisfy
\[
0\le \widetilde Q_h^k(\xi,\omega)\le H-h+1,
\qquad
0\le \widetilde V_h^k(\xi)\le H-h+1.
\]
\end{lemma}
\begin{proof}
The terminal condition gives \(\widetilde V_{H+1}^k\equiv0\). Suppose
\(0\le \widetilde V_{h+1}^k(\xi')\le H-h\) for all \(\xi'\in\Xi\). Since
\(\hat r_h^k(\xi,\omega)\ge0\), \(b_h^k(\xi,\omega)\ge0\), and
\(\hat P_h^k(\cdot\mid\xi,\omega)\in\Delta(\Xi)\), the unclipped backup is
nonnegative. Therefore the clipping definition gives
\[
0\le \widetilde Q_h^k(\xi,\omega)\le H-h+1.
\]
Let
\[
m:=\min_{\omega\in\Omega(\xi)}\widetilde Q_h^k(\xi,\omega),
\qquad
M:=\max_{\omega\in\Omega(\xi)}\widetilde Q_h^k(\xi,\omega).
\]
Since \(\pi_{\rm ref}(\cdot\mid\xi)\) is a probability distribution and
\(\alpha>0\),
\[
e^{m/\alpha}
\le
\sum_{\omega\in\Omega(\xi)}
\pi_{\rm ref}(\omega\mid\xi)
e^{\widetilde Q_h^k(\xi,\omega)/\alpha}
\le
e^{M/\alpha}.
\]
Taking \(\alpha\log\) gives
\[
\min_{\omega\in\Omega(\xi)}\widetilde Q_h^k(\xi,\omega)
\le
\widetilde V_h^k(\xi)
\le
\max_{\omega\in\Omega(\xi)}\widetilde Q_h^k(\xi,\omega).
\]
Thus \(0\le \widetilde V_h^k(\xi)\le H-h+1\). Backward induction proves the
claim.
\end{proof}

\begin{lemma}[Unified Bonus Absorption]
\label{lem:unified-bonus-absorption}
Under the parameter convention in this subsection, on
\(\mathcal E_{\mathrm{conc}}\), uniformly over all
\((k,h,\xi,\omega)\), if \(N_h^k(\xi,\omega)\ge1\), then
\[
|\hat r_h^k(\xi,\omega)-\breve r_h^k(\xi,\omega)|
+
H
\|\hat P_h^k(\cdot\mid\xi,\omega)
-
\breve P_h^k(\cdot\mid\xi,\omega)\|_1
\le
b_h^k(\xi,\omega).
\]
If \(N_h^k(\xi,\omega)=0\), then
\[
b_h^k(\xi,\omega)\ge 2H.
\]
\end{lemma}
\begin{proof}

For \(N_h^k(\xi,\omega)\ge1\), the event \(\mathcal E_r\) controls the
reward deviation and \(\mathcal E_P\) controls the transition deviation.
The concentration absorption convention then absorbs their sum into the
single bonus \(b_h^k(\xi,\omega)\). For \(N_h^k(\xi,\omega)=0\), the
cold-start convention gives
\[
b_h^k(\xi,\omega)
=c_bH\sqrt{S_\xi\log(C_{\log}S_\xi AHK/\delta)}
\ge2H.
\]
This proves the claim.

\end{proof}

\textbf{Measurability note. }
For each episode \(k\), the planning objects
\[
\hat r_h^k,\ \hat P_h^k,\ b_h^k,\ \widetilde Q_h^k,\ \widetilde V_h^k,\ \hat\pi_h^k
\]
are all constructed only from episode \(k\) window history before the
beginning of episode \(k\), using only the window history \(\mathcal W_k\subseteq\{1,\dots,k-1\}\).
Therefore, for any \(h\), the random function
\[
\widetilde V_{h+1}^k
\]
is measurable with respect to the beginning-of-episode \(k\) \(\sigma\)-algebra. This measurability is used only in
the martingale filtration definitions in the executed-policy bridge;
the transition concentration argument no longer treats
\(\widetilde V_{h+1}^k\) as a test function.

\subsubsection{Model Difference Bounds}
\label{app:main-proof-setup:model-difference-bounds}

The proof of Theorem~\ref{thm:dynamic-regret-fixed-window} uses the
conditional-mean drift bound below as the only model-comparison input in
this subsection.

\begin{remark}[Only the Conditional-Mean Drift Bound Is Used]
\label{rem:window-reference-model-drift}
The proof of Theorem~\ref{thm:dynamic-regret-fixed-window} uses only
Lemma~\ref{lem:conditional-mean-model-drift}. No separate drift bound for
an auxiliary window-average model is invoked.
\end{remark}

\begin{lemma}[Conditional-Mean Model Drift]

\label{lem:conditional-mean-model-drift}
If $N_h^k(\xi,\omega)\ge 1$, then
\[
|\breve r_h^k(\xi,\omega)-r_h^{(k)}(\xi,\omega)|\le D_k,
\]
and
\[
\|\breve P_h^k(\cdot\mid\xi,\omega)-P_h^{(k)}(\cdot\mid\xi,\omega)\|_1
\le \frac{D_k}{H}.
\]

\end{lemma}
\begin{proof}

Let
\[
N:=N_h^k(\xi,\omega)\ge 1.
\]
The episodes in the window hitting $(h,\xi,\omega)$ have indices
\[
j_1<j_2<\cdots<j_N,\qquad j_m\in\mathcal W_k.
\]
By definition, 
\[
\breve r_h^k(\xi,\omega)
=
\frac1N\sum_{m=1}^N r_h^{(j_m)}(\xi,\omega).
\]
Therefore
\[
\breve r_h^k(\xi,\omega)-r_h^{(k)}(\xi,\omega)
=
\frac1N\sum_{m=1}^N
\bigl(r_h^{(j_m)}(\xi,\omega)-r_h^{(k)}(\xi,\omega)\bigr).
\]
Taking absolute values and using the triangle inequality, 
\[
|\breve r_h^k(\xi,\omega)-r_h^{(k)}(\xi,\omega)|
\le
\frac1N\sum_{m=1}^N
|r_h^{(j_m)}(\xi,\omega)-r_h^{(k)}(\xi,\omega)|.
\]
For each $m$, because $j_m\in\mathcal W_k$, by telescoping the primitive drift definitions, 
\[
|r_h^{(j_m)}(\xi,\omega)-r_h^{(k)}(\xi,\omega)|
\le
\sum_{\ell=j_m}^{k-1}\Delta_r^{(\ell)}
\le
\sum_{\ell=\max\{1,k-W\}}^{k-1}\Delta_r^{(\ell)}.
\]
Thus
\[
|\breve r_h^k(\xi,\omega)-r_h^{(k)}(\xi,\omega)|
\le
\sum_{\ell=\max\{1,k-W\}}^{k-1}\Delta_r^{(\ell)}
\le D_k.
\]
Similarly, 
\[
\breve P_h^k(\cdot\mid\xi,\omega)
=
\frac1N\sum_{m=1}^N P_h^{(j_m)}(\cdot\mid\xi,\omega),
\]
Hence
\[
\|\breve P_h^k(\cdot\mid\xi,\omega)-P_h^{(k)}(\cdot\mid\xi,\omega)\|_1
\le
\frac1N\sum_{m=1}^N
\|P_h^{(j_m)}(\cdot\mid\xi,\omega)-P_h^{(k)}(\cdot\mid\xi,\omega)\|_1.
\]
Again by telescoping, 
\[
\|P_h^{(j_m)}(\cdot\mid\xi,\omega)-P_h^{(k)}(\cdot\mid\xi,\omega)\|_1
\le
\sum_{\ell=j_m}^{k-1}\Delta_p^{(\ell)}
\le
\sum_{\ell=\max\{1,k-W\}}^{k-1}\Delta_p^{(\ell)}.
\]
Thus
\[
\|\breve P_h^k(\cdot\mid\xi,\omega)-P_h^{(k)}(\cdot\mid\xi,\omega)\|_1
\le
\sum_{\ell=\max\{1,k-W\}}^{k-1}\Delta_p^{(\ell)}
\le
\frac{D_k}{H},
\]
because
\[
D_k=\sum_{\ell=\max\{1,k-W\}}^{k-1}
\bigl(\Delta_r^{(\ell)}+H\Delta_p^{(\ell)}\bigr)
\ge
H\sum_{\ell=\max\{1,k-W\}}^{k-1}\Delta_p^{(\ell)}.
\]

\textbf{Boundary remark for \(N_h^k(\xi,\omega)=0\).} When $N_h^k(\xi,\omega)=0$, $\breve r_h^k(\xi,\omega)$ and $\breve P_h^k(\cdot\mid\xi,\omega)$ follow the fixed convention values; the subsequent near-optimism and certificate decomposition arguments cover this regime through optimistic clipping and maximal bonus, while Lemma~\ref{lem:conditional-mean-model-drift} is used for the positive-count regime.

This proves the claim. 
\end{proof}

\subsubsection{Near-Optimism under Sliding-Window Estimation}
\label{app:main-proof:near-optimism}

\begin{lemma}[Single-Step Optimistic Domination, Including Zero-Count Cases]
\label{lem:single-step-optimistic-domination}

On the event \(\mathcal E_{\mathrm{conc}}\), under the parameter
convention of
Subsection~\ref{app:main-proof-setup:empirical-optimistic-model}, there
exists an absolute constant \(C_{\mathrm{loc}}>0\) such that, for any
\[
k\in[K],\quad h\in[H],\quad \beta_{h+1}\ge 0,
\]
if
\[
\widetilde V_{h+1}^k(\xi')
\ge
V_{h+1}^{\star,(k)}(\xi')-\beta_{h+1}D_k,
\qquad \forall \xi'\in\Xi,
\]
then, for every \(\xi\in\Xi\) and every \(\omega\in\Omega(\xi)\), we have
\[
\widetilde Q_h^k(\xi,\omega)
\ge
Q_h^{\star,(k)}(\xi,\omega)
-
\bigl(\beta_{h+1}+C_{\mathrm{loc}}\bigr)D_k.
\]

\end{lemma}
\begin{proof}

Fix \(k,h,\xi,\omega\), and write
\[
N:=N_h^k(\xi,\omega),
\qquad
U_h^k(\xi,\omega):=
\hat r_h^k(\xi,\omega)
+
b_h^k(\xi,\omega)
+
\hat P_h^k(\cdot\mid\xi,\omega)^\top \widetilde V_{h+1}^k .
\]
By the induction hypothesis and because
\(\hat P_h^k(\cdot\mid\xi,\omega)\in\Delta(\Xi)\),
\[
\hat P_h^k(\cdot\mid\xi,\omega)^\top \widetilde V_{h+1}^k
\ge
\hat P_h^k(\cdot\mid\xi,\omega)^\top V_{h+1}^{\star,(k)}
-\beta_{h+1}D_k .
\]
Hence
\[
U_h^k(\xi,\omega)
\ge
\hat r_h^k(\xi,\omega)
+
b_h^k(\xi,\omega)
+
\hat P_h^k(\cdot\mid\xi,\omega)^\top V_{h+1}^{\star,(k)}
-\beta_{h+1}D_k .
\]
First suppose \(N\ge1\). We compare the last display with
\[
Q_h^{\star,(k)}(\xi,\omega)
=
r_h^{(k)}(\xi,\omega)
+
P_h^{(k)}(\cdot\mid\xi,\omega)^\top V_{h+1}^{\star,(k)} .
\]
Decompose the reward and transition errors as
\[
\hat r_h^k-r_h^{(k)}
=
(\hat r_h^k-\breve r_h^k)+(\breve r_h^k-r_h^{(k)})
\]
and
\[
\hat P_h^k{}^\top V_{h+1}^{\star,(k)}
-
P_h^{(k)\top}V_{h+1}^{\star,(k)}
=
(\hat P_h^k-\breve P_h^k)^\top V_{h+1}^{\star,(k)}
+
(\breve P_h^k-P_h^{(k)})^\top V_{h+1}^{\star,(k)} .
\]
On \(\mathcal E_{\mathrm{conc}}\), Lemma~\ref{lem:unified-bonus-absorption}
gives
\[
\bigl|\hat r_h^k-\breve r_h^k\bigr|
+
H\|\hat P_h^k-\breve P_h^k\|_1
\le
b_h^k(\xi,\omega).
\]
Moreover, by Lemma~\ref{lem:conditional-mean-model-drift},
\[
\bigl|\breve r_h^k-r_h^{(k)}\bigr|\le D_k,
\qquad
\|\breve P_h^k-P_h^{(k)}\|_1\le \frac{D_k}{H}.
\]
By Lemma~\ref{lem:soft-value-boundedness},
\(\|V_{h+1}^{\star,(k)}\|_\infty\le H-h\le H\). Therefore, we obtain
\[
\begin{aligned}
&\hat r_h^k(\xi,\omega)-r_h^{(k)}(\xi,\omega)
+
\bigl(\hat P_h^k(\cdot\mid\xi,\omega)
-
P_h^{(k)}(\cdot\mid\xi,\omega)\bigr)^\top
V_{h+1}^{\star,(k)}
\\
&\quad\ge
-\bigl|\hat r_h^k-\breve r_h^k\bigr|
-
H\|\hat P_h^k-\breve P_h^k\|_1
-
\bigl|\breve r_h^k-r_h^{(k)}\bigr|
-
H\|\breve P_h^k-P_h^{(k)}\|_1
\\
&\quad\ge
-b_h^k(\xi,\omega)-2D_k .
\end{aligned}
\]
Substituting this bound into the lower bound for \(U_h^k(\xi,\omega)\)
yields
\[
U_h^k(\xi,\omega)
\ge
Q_h^{\star,(k)}(\xi,\omega)
-
\bigl(\beta_{h+1}+2\bigr)D_k .
\]
Choose an absolute constant \(C_{\mathrm{loc}}\ge2\). Since
Lemma~\ref{lem:soft-value-boundedness} also gives
\(Q_h^{\star,(k)}(\xi,\omega)\le H-h+1\), and
\[
\widetilde Q_h^k(\xi,\omega)
=
\min\{H-h+1,U_h^k(\xi,\omega)\},
\]
we get
\[
\widetilde Q_h^k(\xi,\omega)
\ge
Q_h^{\star,(k)}(\xi,\omega)
-
\bigl(\beta_{h+1}+C_{\mathrm{loc}}\bigr)D_k .
\]
It remains to consider \(N=0\). In this case, neither concentration nor
Lemma~\ref{lem:conditional-mean-model-drift} is used. By the zero-count
default convention,
\[
\hat r_h^k(\xi,\omega)=0,
\qquad
\hat P_h^k(\cdot\mid\xi,\omega)=P_h^\circ(\cdot\mid\xi,\omega).
\]
By Lemma~\ref{lem:optimistic-value-boundedness},
\(\widetilde V_{h+1}^k\ge0\), and hence
\[
U_h^k(\xi,\omega)
=
b_h^k(\xi,\omega)
+
P_h^\circ(\cdot\mid\xi,\omega)^\top \widetilde V_{h+1}^k
\ge
b_h^k(\xi,\omega).
\]
By Lemma~\ref{lem:unified-bonus-absorption}, in the zero-count case
\[
b_h^k(\xi,\omega)\ge2H\ge H-h+1 .
\]
Therefore
\[
U_h^k(\xi,\omega)\ge H-h+1,
\]
and so
\[
\widetilde Q_h^k(\xi,\omega)
=
\min\{H-h+1,U_h^k(\xi,\omega)\}
=
H-h+1 .
\]
By Lemma~\ref{lem:soft-value-boundedness},
\(Q_h^{\star,(k)}(\xi,\omega)\le H-h+1\). Therefore
\[
\widetilde Q_h^k(\xi,\omega)
\ge
Q_h^{\star,(k)}(\xi,\omega)
\ge
Q_h^{\star,(k)}(\xi,\omega)
-
\bigl(\beta_{h+1}+C_{\mathrm{loc}}\bigr)D_k .
\]
Combining the cases \(N\ge1\) and \(N=0\) proves the claim.

\end{proof}
\begin{proposition}[Near-Optimism under Window Drift]
\label{prop:near-optimism-window-drift}

On the event \(\mathcal E_{\mathrm{conc}}\), there exists an absolute constant \(C_{\mathrm{opt}}>0\) such that, uniformly over all
\[
k\in[K],\quad h\in[H],\quad \xi\in\Xi,
\]
we have
\begin{equation}
\label{eq:near-optimism-bound}
\widetilde V_h^k(\xi)
\ge
V_h^{\star,(k)}(\xi)
-
C_{\mathrm{opt}}(H-h+1)D_k.
\end{equation}

\end{proposition}
\begin{proof}

We prove the claim by backward induction on \(h\). The boundary case \(h=H+1\) satisfies 
\(\widetilde V_{H+1}^k(\xi)=V_{H+1}^{\star,(k)}(\xi)=0\). 

Fix
\[
C_{\mathrm{opt}}:=C_{\mathrm{loc}}.
\]
Assume that the claim has already been established for \(h+1\le H+1\), namely
\[
\widetilde V_{h+1}^k(\xi')
\ge
V_{h+1}^{\star,(k)}(\xi')
-
C_{\mathrm{opt}}(H-h)D_k,
\qquad \forall \xi'\in\Xi.
\]
By Lemma~\ref{lem:single-step-optimistic-domination}, taking \(\beta_{h+1}=C_{\mathrm{opt}}(H-h)\), for any \((\xi,\omega)\) 
\[
\widetilde Q_h^k(\xi,\omega)
\ge
Q_h^{\star,(k)}(\xi,\omega)
-
\bigl(C_{\mathrm{opt}}(H-h)+C_{\mathrm{loc}}\bigr)D_k.
\]
Since \(C_{\mathrm{opt}}=C_{\mathrm{loc}}\), the right-hand side is
\[
Q_h^{\star,(k)}(\xi,\omega)
-C_{\mathrm{opt}}(H-h+1)D_k.
\]
By Lemma~\ref{lem:log-sum-exp-monotone-translation}, applied with
\[
q_\omega=\widetilde Q_h^k(\xi,\omega),
\qquad
q'_\omega=Q_h^{\star,(k)}(\xi,\omega),
\qquad
\varepsilon=C_{\mathrm{opt}}(H-h+1)D_k,
\]
we obtain
\[
\widetilde V_h^k(\xi)
\ge
V_h^{\star,(k)}(\xi)
-
C_{\mathrm{opt}}(H-h+1)D_k.
\]
This closes the induction and proves the claim.
\end{proof}

\subsubsection{Certificate Decomposition}
\label{app:main-proof:certificate-decomposition}

This subsection first establishes a pathwise bridge identity with a unified remainder and then derives the certificate upper bound on the event \(\mathcal E_{\mathrm{conc}}\). The martingale-control argument uses the same conditional-mean identity, and Theorem~\ref{thm:dynamic-regret-fixed-window} invokes the resulting certificate bound on \(\mathcal E_{\mathrm{tot}}\).

\begin{lemma}[Pathwise bridge identity and unified remainder]
\label{lem:pathwise-bridge-identity}

Fix an arbitrary episode \(k\in[K]\). Define
\[
\Gamma_h^k(\xi):=
\widetilde V_h^k(\xi)-V_h^{\hat\pi^k,(k)}(\xi),
\qquad
\Gamma_{H+1}^k\equiv 0.
\]
Along the true trajectory \((\xi_h^k,\omega_h^k,\xi_{h+1}^k)\), define
\[
G_h^k(\xi,\omega):=
\widetilde Q_h^k(\xi,\omega)
-
r_h^{(k)}(\xi,\omega)
-
P_h^{(k)}(\cdot\mid\xi,\omega)^\top V_{h+1}^{\hat\pi^k,(k)},
\]
\[
a_h^k:=
\Gamma_h^k(\xi_h^k)-G_h^k(\xi_h^k,\omega_h^k),
\]
\[
t_h^k:=
P_h^{(k)}(\cdot\mid\xi_h^k,\omega_h^k)^\top \Gamma_{h+1}^k
-
\Gamma_{h+1}^k(\xi_{h+1}^k),
\]
and
\[
\varepsilon_h^k:=a_h^k+t_h^k,
\qquad
M_k:=\sum_{h=1}^H \varepsilon_h^k.
\]
Then, for every \(h\in[H]\) and every \(\xi\in\Xi\),
\[
\Gamma_h^k(\xi)
=
\sum_{\omega\in\Omega(\xi)}
\hat\pi_h^k(\omega\mid\xi)G_h^k(\xi,\omega).
\]
Consequently, along the realized trajectory,
\[
a_h^k
=
\sum_{\omega\in\Omega(\xi_h^k)}
\hat\pi_h^k(\omega\mid\xi_h^k)G_h^k(\xi_h^k,\omega)
-
G_h^k(\xi_h^k,\omega_h^k).
\]
Moreover,
\[
\Gamma_h^k(\xi_h^k)
=
G_h^k(\xi_h^k,\omega_h^k)+a_h^k,
\]
and
\[
\Gamma_h^k(\xi_h^k)
=
\Bigl[
\widetilde Q_h^k(\xi_h^k,\omega_h^k)
-
r_h^{(k)}(\xi_h^k,\omega_h^k)
-
P_h^{(k)}(\cdot\mid\xi_h^k,\omega_h^k)^\top \widetilde V_{h+1}^k
\Bigr]
+
\Gamma_{h+1}^k(\xi_{h+1}^k)
+
\varepsilon_h^k.
\]

\end{lemma}
\begin{proof}

For any state \(\xi\), apply
Proposition~\ref{prop:kl-bellman-optimality} to the vector
\(q_\omega=\widetilde Q_h^k(\xi,\omega)\),
\(\omega\in\Omega(\xi)\). By the definition of \(\hat\pi_h^k\), this
policy is the corresponding softmax maximizer, and the Fenchel equality
in Proposition~\ref{prop:kl-bellman-optimality} gives
\[
\widetilde V_h^k(\xi)
=
\sum_{\omega\in\Omega(\xi)}
\hat\pi_h^k(\omega\mid\xi)
\left[
\widetilde Q_h^k(\xi,\omega)
-
\alpha\log\frac{\hat\pi_h^k(\omega\mid\xi)}{\pi_{\rm ref}(\omega\mid\xi)}
\right].
\]
The policy value satisfies
\[
V_h^{\hat\pi^k,(k)}(\xi)
=
\sum_{\omega\in\Omega(\xi)}
\hat\pi_h^k(\omega\mid\xi)
\Bigl(
r_h^{(k)}(\xi,\omega)
+
P_h^{(k)}(\cdot\mid\xi,\omega)^\top V_{h+1}^{\hat\pi^k,(k)}
-
\alpha\log\frac{\hat\pi_h^k(\omega\mid\xi)}{\pi_{\rm ref}(\omega\mid\xi)}
\Bigr).
\]
Subtracting the two displays gives
\[
\Gamma_h^k(\xi)
=
\sum_{\omega}
\hat\pi_h^k(\omega\mid\xi)\,G_h^k(\xi,\omega).
\]
For the action term, the definition of \(a_h^k\) gives
\[
\Gamma_h^k(\xi_h^k)=G_h^k(\xi_h^k,\omega_h^k)+a_h^k.
\]
On the other hand, 
\[
G_h^k(\xi_h^k,\omega_h^k)
=
\widetilde Q_h^k(\xi_h^k,\omega_h^k)
-
r_h^{(k)}(\xi_h^k,\omega_h^k)
-
P_h^{(k)}(\cdot\mid\xi_h^k,\omega_h^k)^\top \widetilde V_{h+1}^k
+
P_h^{(k)}(\cdot\mid\xi_h^k,\omega_h^k)^\top\Gamma_{h+1}^k.
\]
By the definition of \(t_h^k\),
\[
P_h^{(k)}(\cdot\mid\xi_h^k,\omega_h^k)^\top\Gamma_{h+1}^k
=
\Gamma_{h+1}^k(\xi_{h+1}^k)+t_h^k,
\]
we obtain the identity. This proves the claim. 

\end{proof}
\begin{lemma}[On-Event Pathwise Certificate Bound]
\label{lem:pathwise-certificate-bound}

On the event \(\mathcal E_{\mathrm{conc}}\), there exists an absolute constant \(C_{\mathrm{cert}}>0\) such that, uniformly over all \(k\in[K]\) and \(h\in[H]\), we have
\begin{equation}
\label{eq:certificate-decomposition-step}
\Gamma_h^k(\xi_h^k)
\le
2\,b_h^k(\xi_h^k,\omega_h^k)
+
\Gamma_{h+1}^k(\xi_{h+1}^k)
+
\varepsilon_h^k
+
C_{\mathrm{cert}}D_k.
\end{equation}
Therefore,
\begin{equation}
\label{eq:certificate-decomposition-bound}
\widetilde V_1^k(\xi_1^k)-V_1^{\hat\pi^k,(k)}(\xi_1^k)
\le
2\sum_{h=1}^H b_h^k(\xi_h^k,\omega_h^k)
+
M_k
+
C_{\mathrm{cert}}H D_k.
\end{equation}

\end{lemma}
\begin{proof}

By the optimistic backup and Lemma~\ref{lem:pathwise-bridge-identity}, 
\[
\Gamma_h^k(\xi_h^k)
\le
\hat r_h^k-r_h^{(k)}
+
b_h^k
+
(\hat P_h^k-P_h^{(k)})^\top \widetilde V_{h+1}^k
+
\Gamma_{h+1}^k(\xi_{h+1}^k)
+
\varepsilon_h^k,
\]
where all unindexed terms are evaluated at \((h,\xi_h^k,\omega_h^k)\). 

Let
\[
N:=N_h^k(\xi_h^k,\omega_h^k).
\]
If \(N\ge1\), decompose
\[
\hat r_h^k-r_h^{(k)}
=
(\hat r_h^k-\breve r_h^k)+(\breve r_h^k-r_h^{(k)}),
\]
\[
(\hat P_h^k-P_h^{(k)})^\top \widetilde V_{h+1}^k
=
(\hat P_h^k-\breve P_h^k)^\top \widetilde V_{h+1}^k
+
(\breve P_h^k-P_h^{(k)})^\top \widetilde V_{h+1}^k.
\]
On the event \(\mathcal E_{\mathrm{conc}}\), by
Lemmas~\ref{lem:unified-bonus-absorption}
and~\ref{lem:conditional-mean-model-drift}, and by
Lemma~\ref{lem:optimistic-value-boundedness},
\(\|\widetilde V_{h+1}^k\|_\infty\le H\), for \(N\ge1\),
\[
\hat r_h^k-r_h^{(k)}
+
(\hat P_h^k-P_h^{(k)})^\top \widetilde V_{h+1}^k
\le
b_h^k(\xi_h^k,\omega_h^k)+2D_k.
\]
If \(N=0\), Lemma~\ref{lem:conditional-mean-model-drift} is not invoked.
Instead, by the default convention \(\hat r_h^k(\xi_h^k,\omega_h^k)=0\)
and \(r_h^{(k)}\in[0,1]\), we get
\(\hat r_h^k-r_h^{(k)} = -r_h^{(k)} \le 0\).
At the same time, \(\hat P_h^k\) and
\(P_h^{(k)}\) are probability distributions, and by Lemma~\ref{lem:optimistic-value-boundedness},
\(0\le\widetilde V_{h+1}^k\le H\), therefore
\[
(\hat P_h^k-P_h^{(k)})^\top \widetilde V_{h+1}^k
\le
(\hat P_h^k)^\top \widetilde V_{h+1}^k
\le H.
\]
Hence
\[
\hat r_h^k-r_h^{(k)}
+
(\hat P_h^k-P_h^{(k)})^\top \widetilde V_{h+1}^k
\le
0+H
=
H,
\]
while Lemma~\ref{lem:unified-bonus-absorption} gives
\(b_h^k\ge2H\ge H\).
Therefore the cold-start contribution is directly absorbed
by the same bonus term.

Therefore, in both cases, and using \(D_k\ge0\) for the zero-count case,
the unified concentration interface gives, for example with \(C_0=2\),
\[
\hat r_h^k-r_h^{(k)}
+
(\hat P_h^k-P_h^{(k)})^\top \widetilde V_{h+1}^k
\le
b_h^k(\xi_h^k,\omega_h^k)+C_0D_k.
\]
Substituting back into the initial inequality and taking
\(C_{\mathrm{cert}}\ge C_0\) gives the certificate bound. Summing over
\(h=1,\dots,H\), telescoping, and using \(\Gamma_{H+1}^k\equiv0\) gives
the episode-wise bound. This proves the claim.
\end{proof}

\subsubsection{Summed Bonus Bound}
\label{app:main-proof:summed-bonus}

\begin{proposition}[Summed Bonus Bound under Overlapping Sliding-Window Counts]
\label{prop:summed-bonus-bound}

For any \(1\le W\le K\), consider the sliding-window bonus defined in Subsection~\ref{app:main-proof-setup:empirical-optimistic-model}.
\[
b_h^k(\xi,\omega)
=
c_b H\sqrt{
\frac{S_\xi\log(C_{\log}S_\xi AHK/\delta)}{(N_h^k(\xi,\omega)+1)\wedge W}
}.
\]
Then there exists an absolute constant \(C_{\mathrm{bon}}>0\) such that
\begin{equation}
\label{eq:summed-bonus-bound}
\sum_{k=1}^K\sum_{h=1}^H
b_h^k(\xi_h^k,\omega_h^k)
\le
C_{\mathrm{bon}}\,H^2 K S_\xi
\sqrt{\frac{A\log(C_{\log}S_\xi AHK/\delta)}{W}}.
\end{equation}
Up to absolute constants,
\[
\sum_{k=1}^K\sum_{h=1}^H
b_h^k(\xi_h^k,\omega_h^k)
=
\widetilde O\!\left(
H^2 K S_\xi\sqrt{\frac{A}{W}}
\right).
\]

\end{proposition}
\begin{proof}

We use an overlapping-window counting argument and keep the \(\sqrt{S_\xi}\) factor from the bonus explicitly. 

Define
\[
\mathcal Z:=\{(h,\xi,\omega): h\in[H],\ \xi\in\Xi,\ \omega\in\Omega(\xi)\}.
\]
Then
\[
|\mathcal Z|\le H S_\xi A.
\]
Partition the episode axis into non-overlapping blocks of length \(W\)
\[
\mathcal B_1,\mathcal B_2,\dots,\mathcal B_B,
\qquad
B:=\left\lceil \frac{K}{W}\right\rceil,
\]
where
\[
\mathcal B_b:=\{(b-1)W+1,\dots,\min\{bW,K\}\}.
\]
For each fixed triple \(z=(h,\xi,\omega)\in\mathcal Z\) and block \(b\in[B]\), define the number of occurrences of \(z\) in the block
\[
M_{b,z}:=
\sum_{k\in\mathcal B_b}
\mathbf 1\{(\xi_h^k,\omega_h^k)=(\xi,\omega)\}.
\]
If \(M_{b,z}\ge1\), write its occurrence times as
\[
k_{b,z,1}<k_{b,z,2}<\cdots<k_{b,z,M_{b,z}},
\qquad k_{b,z,j}\in\mathcal B_b.
\]
For \(i<j\), the two occurrence times \(k_{b,z,i}\) and
\(k_{b,z,j}\) belong to the same block \(\mathcal B_b\), whose length is
at most \(W\). Hence
\[
1\le k_{b,z,j}-k_{b,z,i}\le W-1 .
\]
Therefore
\[
k_{b,z,i}\in
\{\max\{1,k_{b,z,j}-W\},\ldots,k_{b,z,j}-1\}
=\mathcal W_{k_{b,z,j}} .
\]
Thus, when the \(j\)-th occurrence of \(z\) in block \(b\) happens, the
previous \(j-1\) occurrences of the same \(z\) in that block are all
counted in the sliding-window count. Consequently,
\[
N_h^{k_{b,z,j}}(\xi,\omega)\ge j-1,
\qquad
(N_h^{k_{b,z,j}}(\xi,\omega)+1)\wedge W\ge j .
\]
Therefore
\[
b_h^{k_{b,z,j}}(\xi,\omega)
\le
c_b H\sqrt{\frac{S_\xi\log(C_{\log}S_\xi AHK/\delta)}{j}}.
\]
Summing over the occurrences of \(z\) inside block \(b\) and using
\(\sum_{j=1}^m j^{-1/2}\le 2\sqrt m\), we obtain
\[
\sum_{k\in\mathcal B_b}
b_h^k(\xi_h^k,\omega_h^k)
\mathbf 1\{(\xi_h^k,\omega_h^k)=(\xi,\omega)\}
\le
2c_b H
\sqrt{S_\xi\log(C_{\log}S_\xi AHK/\delta)}
\sqrt{M_{b,z}}.
\]
Summing this deterministic inequality over all blocks \(b\) and triples
\(z\in\mathcal Z\), and then applying Cauchy--Schwarz, gives
\[
\sum_{k=1}^K\sum_{h=1}^H b_h^k(\xi_h^k,\omega_h^k)
\le
2c_b H
\sqrt{S_\xi\log(C_{\log}S_\xi AHK/\delta)}
\sum_{b=1}^B\sum_{z\in\mathcal Z}\sqrt{M_{b,z}}
\]
and
\[
\sum_{b=1}^B\sum_{z\in\mathcal Z}\sqrt{M_{b,z}}
\le
\sqrt{B|\mathcal Z|}
\sqrt{\sum_{b=1}^B\sum_{z\in\mathcal Z}M_{b,z}}.
\]
Since
\[
\sum_{b=1}^B\sum_{z\in\mathcal Z}M_{b,z}=KH,
\qquad
B\le \left\lceil\frac KW\right\rceil\le \frac{2K}{W},
\qquad
|\mathcal Z|\le HS_\xi A,
\]
we have
\[
\sum_{b=1}^B\sum_{z\in\mathcal Z}\sqrt{M_{b,z}}
\le
\sqrt{\frac{2K}{W}\cdot HS_\xi A}\sqrt{KH}
=
\sqrt 2\, H K\sqrt{\frac{S_\xi A}{W}}.
\]
Substituting this bound yields
\[
\sum_{k=1}^K\sum_{h=1}^H b_h^k(\xi_h^k,\omega_h^k)
\le
2\sqrt 2\,c_b\,
H^2 K S_\xi
\sqrt{\frac{A\log(C_{\log}S_\xi AHK/\delta)}{W}}.
\]
Absorbing the absolute factor \(2\sqrt2\,c_b\) into
\(C_{\mathrm{bon}}\) proves the claim.

\end{proof}
\begin{remark}[Use of the Summed-Bonus Bound]
\label{rem:summed-bonus-bound}

This proposition is a pathwise combinatorial bound controlling the realized bonus on the true trajectory. It is used pathwise in Theorem~\ref{thm:dynamic-regret-fixed-window}.
\end{remark}

\subsubsection{Pathwise Martingale Control}
\label{app:main-proof:martingale-control}

\begin{proposition}[High-Probability Control of the Unified Remainder]
\label{prop:unified-martingale-control}

Assume the sampling, global-filtration, environment-predictability, and
algorithmic-measurability conventions from
Subsection~\ref{app:problem-setup:sampling-filtration} and
Subsection~\ref{app:main-proof-setup:empirical-optimistic-model}. In
particular, at the beginning of episode \(k\), the model
\(\mathcal M^{(k)}\), the empirical objects
\(\hat r^k,\hat P^k,b^k\), the optimistic planning objects
\(\widetilde Q^k,\widetilde V^k\), and the policy \(\hat\pi^k\) are
measurable with respect to the beginning-of-episode \(\sigma\)-algebra,
and the action is sampled according to
\[
\omega_h^k\sim \hat\pi_h^k(\cdot\mid\xi_h^k)
\]
before the reward and next state are revealed.

Using the remainder notation from Lemma~\ref{lem:pathwise-bridge-identity}, define
\[
\varepsilon_h^k=a_h^k+t_h^k,
\qquad
M_k=\sum_{h=1}^H\varepsilon_h^k.
\]
Flatten the episode-stage indices
\[
t=(k-1)H+h,\qquad t=1,\dots,KH.
\]
Let \((k(t),h(t))\) denote the inverse episode-stage map, and define
\[
\eta_t:=\varepsilon_{h(t)}^{k(t)}.
\]

Define the flattened filtration as follows. For \(t=(k-1)H+h\), let
\(\mathcal F_{t-1}\) contain all randomness completed strictly before the
stage-\((k,h)\) outcome is generated, together with the
beginning-of-episode \(k\) planning objects and the current state
\(\xi_h^k\). Let
\[
\mathcal G_t:=\mathcal F_{t-1}\vee\sigma(\omega_h^k),
\]
and let
\[
\mathcal F_t:=\mathcal G_t\vee\sigma(r_h^k,\xi_{h+1}^k).
\]
If \(h=H\) and \(k<K\), we enlarge \(\mathcal F_t\) further by including
\(\xi_1^{k+1}\). This enlargement preserves the increasing-filtration
property. Since \(\xi_1^{k+1}\) is independent of the completed history
and \(t_h^k\) depends only on the current transition outcome
\(\xi_{h+1}^k\), this enlargement does not affect the conditional
zero-mean calculation for \(t_h^k\).

Then there exists an absolute constant \(C_{\mathrm{mg}}>0\), such that the event
\begin{equation}
\label{eq:martingale-remainder-bound}
\mathcal E_{\mathrm{mg}}:=
\left\{
\left|
\sum_{k=1}^K M_k
\right|
\le
C_{\mathrm{mg}}H^2\sqrt{KH\log(C_{\log}/\delta)}
\right\}
\end{equation}
satisfies
\[
\mathbb P(\mathcal E_{\mathrm{mg}})\ge1-\delta/2.
\]

\end{proposition}
\begin{proof}

We prove the high-probability bound by showing that the flattened sequence
\(\eta_t\) is a bounded martingale difference sequence, and then applying
Lemma~\ref{lem:azuma-hoeffding}.

Fix \(t=(k-1)H+h\). By the measurability note in Subsection~\ref{app:main-proof-setup:empirical-optimistic-model}, \(\widetilde Q_h^k,\widetilde V_h^k,\hat\pi_h^k\) are fixed at the beginning of episode \(k\); after the current state \(\xi_h^k\) is observed, \(\xi_h^k\) is \(\mathcal F_{t-1}\)-measurable. By the environment predictability convention in Subsection~\ref{app:problem-setup:sampling-filtration}, the current model objects
\[
\{r_h^{(k)}(\xi,\omega),P_h^{(k)}(\cdot\mid\xi,\omega):
 h\in[H],\xi\in\Xi,\omega\in\Omega(\xi)\}
\]
are also fixed, or at least measurable with respect to the beginning-of-episode \(k\) \(\sigma\)-algebra. Since \(\hat\pi^k\) is constructed at the beginning of episode \(k\), the policy value functions \(V_h^{\hat\pi^k,(k)}\) are deterministic functions of these beginning-of-episode objects. Hence, after \(\xi_h^k\) is included in \(\mathcal F_{t-1}\), the map
\[
\omega\mapsto G_h^k(\xi_h^k,\omega)
\]
is \(\mathcal F_{t-1}\)-measurable. Lemma~\ref{lem:pathwise-bridge-identity} gives
\[
\Gamma_h^k(\xi_h^k)
=
\sum_{\omega\in\Omega(\xi_h^k)}
\hat\pi_h^k(\omega\mid\xi_h^k)
G_h^k(\xi_h^k,\omega),
\]
and therefore \(\Gamma_h^k(\xi_h^k)\) is \(\mathcal F_{t-1}\)-measurable. Since
\[
\omega_h^k\mid\mathcal F_{t-1}
\sim
\hat\pi_h^k(\cdot\mid\xi_h^k),
\]
and
\[
a_h^k
=
\Gamma_h^k(\xi_h^k)-G_h^k(\xi_h^k,\omega_h^k),
\]
it follows that
\[
\mathbb E[a_h^k\mid\mathcal F_{t-1}]=0.
\]
On the other hand, 
\[
t_h^k
=
P_h^{(k)}(\cdot\mid\xi_h^k,\omega_h^k)^\top\Gamma_{h+1}^k
-
\Gamma_{h+1}^k(\xi_{h+1}^k).
\]
Conditioning on \(\mathcal G_t\), \((\xi_h^k,\omega_h^k)\) has been fixed, while sampling model gives
\[
\xi_{h+1}^k\mid\mathcal G_t
\sim
P_h^{(k)}(\cdot\mid\xi_h^k,\omega_h^k).
\]
Therefore
\[
\mathbb E[t_h^k\mid\mathcal G_t]=0.
\]
By tower property, 
\[
\mathbb E[\eta_t\mid\mathcal F_{t-1}]
=
\mathbb E[a_h^k+t_h^k\mid\mathcal F_{t-1}]
=
\mathbb E[a_h^k\mid\mathcal F_{t-1}]
+
\mathbb E[\mathbb E[t_h^k\mid\mathcal G_t]\mid\mathcal F_{t-1}]
=0.
\]
Thus \((\eta_t,\mathcal F_t)_{t=1}^{KH}\) is MDS. 

We next prove bounded increments. By clipping, 
\[
0\le \widetilde Q_h^k(\xi,\omega)\le H-h+1\le H,
\qquad
0\le \widetilde V_h^k(\xi)\le H-h+1\le H.
\]
By the definition of the optimistic softmax policy in
\eqref{eq:optimistic-policy} and the optimistic value in
\eqref{eq:optimistic-value},
\[
\alpha\log
\frac{\hat\pi_h^k(\omega\mid\xi)}{\pi_{\rm ref}(\omega\mid\xi)}
=
\widetilde Q_h^k(\xi,\omega)-\widetilde V_h^k(\xi).
\]
Therefore
\[
\left|
\alpha\log
\frac{\hat\pi_h^k(\omega\mid\xi)}{\pi_{\rm ref}(\omega\mid\xi)}
\right|
\le H,
\]
and regularized reward with respect to
\[
\left|
r_h^{(k)}(\xi,\omega)
-
\alpha\log
\frac{\hat\pi_h^k(\omega\mid\xi)}{\pi_{\rm ref}(\omega\mid\xi)}
\right|
\le
1+H\le2H.
\]
gives
\[
|V_h^{\hat\pi^k,(k)}(\xi)|
\le
2H(H-h+1)
\le
2H^2.
\]
This \(H^2\) envelope follows from the clipped \(Q,V\)-range and the softmax
identity above. The bound is stated in a form sufficient for the
variation-dependent sublinearity result used in the main text.
Hence
\[
|\Gamma_h^k(\xi)|
\le
|\widetilde V_h^k(\xi)|
+
|V_h^{\hat\pi^k,(k)}(\xi)|
\le
H+2H^2
\le
3H^2.
\]
At the same time,
\[
|G_h^k(\xi,\omega)|
\le
|\widetilde Q_h^k(\xi,\omega)|
+
|r_h^{(k)}(\xi,\omega)|
+
\left|P_h^{(k)}(\cdot\mid\xi,\omega)^\top V_{h+1}^{\hat\pi^k,(k)}\right|
\le
H+1+2H^2
\le
4H^2,
\qquad H\ge1.
\]
because
\[
a_h^k
=
\Gamma_h^k(\xi_h^k)-G_h^k(\xi_h^k,\omega_h^k)
=
\mathbb E_{\omega\sim\hat\pi_h^k(\cdot\mid\xi_h^k)}G_h^k(\xi_h^k,\omega)
-G_h^k(\xi_h^k,\omega_h^k),
\]
we get
\[
|a_h^k|
\le
2\sup_{\omega\in\Omega(\xi_h^k)}|G_h^k(\xi_h^k,\omega)|
\le
8H^2.
\]
Similarly,
\[
|t_h^k|
\le
\left|
P_h^{(k)}(\cdot\mid\xi_h^k,\omega_h^k)^\top\Gamma_{h+1}^k
\right|
+
|\Gamma_{h+1}^k(\xi_{h+1}^k)|
\le
2\|\Gamma_{h+1}^k\|_\infty
\le
6H^2.
\]
Therefore
\[
|\eta_t|=|\varepsilon_h^k|
\le
|a_h^k|+|t_h^k|
\le14H^2.
\]
We subsequently use an absolute-constant \(H^2\) boundedness. Thus we write
\[
|\eta_t|\le C_\eta H^2
\]
with \(C_\eta\) later absorbed into \(C_{\mathrm{mg}}\). 

Applying Lemma~\ref{lem:azuma-hoeffding} to the martingale difference sequence gives 
\[
\mathbb P\!\left(
\left|\sum_{t=1}^{KH}\eta_t\right|\ge u
\right)
\le
2\exp\left(
-\frac{u^2}{2C_\eta^2H^4KH}
\right).
\]
Take
\[
u=C_{\mathrm{mg}}H^2\sqrt{KH\log(C_{\log}/\delta)}
\]
and \(C_{\mathrm{mg}}\) large enough, gives
\[
\mathbb P(\mathcal E_{\mathrm{mg}})\ge1-\delta/2.
\]
Finally, 
\[
\sum_{t=1}^{KH}\eta_t
=
\sum_{k=1}^K\sum_{h=1}^H\varepsilon_h^k
=
\sum_{k=1}^K M_k.
\]
This proves the claim.
\end{proof}

\subsubsection{Unified High-Probability Event}
\label{app:main-proof-setup:unified-event}

The analysis uses the following high-probability events. Recall that
\[
\mathcal E_{\mathrm{conc}}=\mathcal E_r\cap\mathcal E_P,
\]
where \(\mathcal E_r\) and \(\mathcal E_P\) are the reward and transition
concentration events defined in
Subsection~\ref{app:tools:reward-transition-concentration}. By
Lemma~\ref{lem:total-concentration-event},
\[
\mathbb P(\mathcal E_{\mathrm{conc}})\ge 1-\delta/2.
\]
The martingale event is
\[
\mathcal E_{\mathrm{mg}}
:=
\left\{
\left|\sum_{k=1}^K M_k\right|
\le
C_{\mathrm{mg}}H^2\sqrt{KH\log(C_{\log}/\delta)}
\right\},
\]
where \(M_k\) is the unified remainder from
Lemma~\ref{lem:pathwise-bridge-identity}. By
Proposition~\ref{prop:unified-martingale-control},
\[
\mathbb P(\mathcal E_{\mathrm{mg}})\ge 1-\delta/2.
\]
Finally define
\[
\mathcal E_{\mathrm{tot}}
:=
\mathcal E_{\mathrm{conc}}\cap\mathcal E_{\mathrm{mg}}.
\]
A union bound gives
\[
\mathbb P(\mathcal E_{\mathrm{tot}})\ge1-\delta.
\]
All pathwise regret bounds in Theorem~\ref{thm:dynamic-regret-fixed-window}
are proved on \(\mathcal E_{\mathrm{tot}}\).

\subsubsection{Proof of Theorem~\ref{thm:dynamic-regret-fixed-window}}
\label{app:main-proof:fixed-window-proof}

\begin{proof}[Proof of Theorem~\ref{thm:dynamic-regret-fixed-window}]
By Subsection~\ref{app:main-proof-setup:unified-event},
\(\mathbb P(\mathcal E_{\mathrm{tot}})\ge1-\delta\). We fix an outcome in
\(\mathcal E_{\mathrm{tot}}\) and proceed pathwise.

For each episode \(k\), at the realized initial state \(\xi_1^k\) decompose: 
\[
V_1^{\star,(k)}(\xi_1^k)-V_1^{\hat\pi^k,(k)}(\xi_1^k)
=
\bigl(V_1^{\star,(k)}(\xi_1^k)-\widetilde V_1^k(\xi_1^k)\bigr)
+
\bigl(\widetilde V_1^k(\xi_1^k)-V_1^{\hat\pi^k,(k)}(\xi_1^k)\bigr).
\]
By Proposition~\ref{prop:near-optimism-window-drift}, 
\[
V_1^{\star,(k)}(\xi_1^k)-\widetilde V_1^k(\xi_1^k)
\le
C_{\mathrm{opt}}HD_k.
\]
By Lemma~\ref{lem:pathwise-certificate-bound}, 
\[
\widetilde V_1^k(\xi_1^k)-V_1^{\hat\pi^k,(k)}(\xi_1^k)
\le
2\sum_{h=1}^H b_h^k(\xi_h^k,\omega_h^k)
+
M_k
+
C_{\mathrm{cert}}HD_k.
\]
Here \(M_k=\sum_{h=1}^H\varepsilon_h^k\) is the unified remainder from Lemma~\ref{lem:pathwise-bridge-identity} controlled by Proposition~\ref{prop:unified-martingale-control}. Summing over episodes and using the pathwise definition of \(\mathrm{DynReg}(K)\) gives
\[
\mathrm{DynReg}(K)
\le
2\sum_{k=1}^K\sum_{h=1}^H b_h^k(\xi_h^k,\omega_h^k)
+
\sum_{k=1}^K M_k
+
C_\star H\sum_{k=1}^K D_k,
\]
where \(C_\star:=C_{\mathrm{opt}}+C_{\mathrm{cert}}\). 

The three terms are controlled as follows:
\begin{enumerate}
\item Proposition~\ref{prop:summed-bonus-bound} gives the pathwise summed-bonus bound:
\[
2\sum_{k=1}^K\sum_{h=1}^H b_h^k(\xi_h^k,\omega_h^k)
\le
2C_{\mathrm{bon}}H^2KS_\xi
\sqrt{\frac{A\log(C_{\log}S_\xi AHK/\delta)}{W}}.
\]
\item Proposition~\ref{prop:unified-martingale-control} on \(\mathcal E_{\mathrm{mg}}\) gives
\[
\sum_{k=1}^K M_k
\le
C_{\mathrm{mg}}H^2\sqrt{KH\log(C_{\log}/\delta)}.
\]
This has the same logarithmic factor as the third term in
\eqref{eq:fixed-window-regret-bound}, so the martingale contribution is
absorbed after increasing the absolute constant.
\item Equation~\eqref{eq:window-drift-sum} gives
\[
\sum_{k=1}^K D_k
\le
W(B_r+HB_p).
\]
Therefore
\[
C_\star H\sum_{k=1}^K D_k
\le
C_\star WH(B_r+HB_p).
\]
\end{enumerate}
Renaming absolute constants gives the theorem statement. This proves the claim.
\end{proof}

\subsection{From Fixed-Window Regret to the Main-Text Informal Statement}
\label{app:oracle-window-proof}

This subsection proves Corollary~\ref{cor:oracle-tuned-window}. The result is a known-variation tuning statement: the window length is selected before the run as a deterministic function of the variation budget, and the fixed-window theorem is then applied to that algorithm instance.

\subsubsection{Proof of Corollary~\ref{cor:oracle-tuned-window}}
\label{app:main-proof:oracle-tuned-window-proof}

\begin{proof}[Proof of Corollary~\ref{cor:oracle-tuned-window}]
Fix an arbitrary window chosen before the algorithm is run,
\(W_{\mathrm{orc}}\in\{1,\dots,K\}\). Because the algorithm, empirical windows,
bonuses, and policy sequence are then fixed by this window choice,
Theorem~\ref{thm:dynamic-regret-fixed-window} applies directly to that algorithm
instance and gives
\[
\mathrm{DynReg}(K)
\le
\widetilde O\!\left(
H^2KS_\xi\sqrt{\frac{A}{W_{\mathrm{orc}}}}
+
W_{\mathrm{orc}}HV_K
+
H^2\sqrt{KH}
\right)
\]
with probability at least \(1-\delta\). Since the oracle may choose any
deterministic window before the run, choosing a window that minimizes the
deterministic upper-bound function gives the oracle upper envelope.

For a uniform display over all variation regimes, we use the smoothed quantity
\(V_K+1\). This affects only constant-level edge cases and yields the same
variation-dependent scaling in the non-stationary regimes of interest. In the
untruncated case, ignoring constants gives
\[
H^2KS_\xi\sqrt{\frac{A}{W}}
\asymp
WH(V_K+1),
\]
and hence
\[
HKS_\xi\sqrt A
\asymp
W^{3/2}(V_K+1).
\]
Therefore
\[
W_\star
=
\left(
\frac{H^2S_\xi^2 A K^2}{(V_K+1)^2}
\right)^{1/3}.
\]
When \(1\le W_\star\le K\), \(W_{\mathrm{orc}}\asymp W_\star\), substitution gives
\[
H^2KS_\xi\sqrt{\frac{A}{W_\star}}
\asymp
HW_\star(V_K+1)
\asymp
H^{5/3}S_\xi^{2/3}A^{1/3}K^{2/3}(V_K+1)^{1/3}.
\]
The martingale remainder \(H^2\sqrt{KH}\) gives the tuned-window bound.

For the growing-dimension statement, the first term in
\eqref{eq:oracle-tuned-bound} is \(o(K)\) whenever
\[
H^{5/3}S_\xi^{2/3}A^{1/3}K^{2/3}(V_K+1)^{1/3}=o(K).
\]
Cubing both sides and canceling \(K^2\) gives
\[
H^5S_\xi^2A(V_K+1)=o(K).
\]
Together with \(H^2\sqrt{KH}=o(K)\), this implies
\(\mathrm{DynReg}(K)=o(K)\). When \(H,S_\xi,A\) are fixed, the condition reduces
to \(V_K=o(K)\), and the martingale term is lower order.

This is a known-variation tuning statement. No simultaneous high-probability
guarantee over all window lengths is needed for the stated corollary, because
the algorithm is run with a single deterministic window chosen before observing
the trajectory. The minimization above is over the deterministic upper-bound
expression, not over realized pathwise regrets from multiple window choices. A
separate simultaneous guarantee over all window lengths could be obtained by a
standard union bound over \(W\in[K]\), at the cost of an additional logarithmic
factor. This proves the claim.
\end{proof}

\section{Deployable Executor Instantiation}
\label{sec:appendix_theory_method_mapping}

This appendix specifies the deployable instantiation of MPCR used in the fixed-library transfer experiments. MPCR is implemented as an executor-side wrapper around a base web agent and a frozen learned skill library.

\subsection{Correspondence between SWOP and MPCR}
\label{sec:mpcr_swop_correspondence}

SWOP and MPCR are two views of the same adaptation principle: stay anchored to what the learned library prefers, and let recent execution evidence decide which of those preferences is actionable right now. Table~\ref{tab:app_swop_mpcr_map} maps each element of the analysis onto the executor stage that realizes it.

\begin{table}[htbp]
\centering
\small
\setlength{\tabcolsep}{4pt}
\renewcommand{\arraystretch}{1.25}
\begin{tabularx}{0.92\linewidth}{@{}XX@{}}
\toprule
\textbf{SWOP (analysis)} & \textbf{MPCR (executor)} \\
\midrule
Reference prior $\pi_{\mathrm{ref}}$ & Fixed skill-library and base-executor preference \\
Sliding-window recent evidence & Current-page feasibility verification \\
Optimism for unresolved uncertainty & Soft scoring that downweights rather than filters \\
KL regularization toward the prior & Reranking constrained to the proposed shortlist \\
\bottomrule
\end{tabularx}
\caption{Componentwise correspondence between the SWOP analysis model and the MPCR executor.}
\label{tab:app_swop_mpcr_map}
\end{table}

This mapping is what fixes the executor's structure, and each row settles a concrete design question. The reference prior is the frozen skill library and base executor, so adaptation may not touch either; every reported gain is therefore attributable to test-time selection rather than to further skill acquisition. Recent evidence enters as current-page feasibility verification, which places adaptation inside a single episode, the timescale on which executability actually changes under environment evolution. Optimism is realized as soft scoring rather than hard filtering: a candidate whose target cannot be confirmed is downweighted but retained, so a partially observed page state cannot discard the action that would have recovered the task. The regularization term is then confined to reranking the proposed shortlist, which limits how far selection may move and keeps the executor from drifting away from the learned preference when the feasibility signal is noisy. The analysis therefore supplies the design: what to hold fixed, what to condition on, how to act when that evidence is inconclusive, and how far to move. Theorem~\ref{thm:swop-dynreg} characterizes the principle in the idealized model, and the evidence that it carries into browser execution is empirical.

Three results support the mapping, one per row. Appendix~\ref{sec:app_component_ablation} shows that its two active halves contribute separately and are not interchangeable: constrained reranking over the shortlist adds $4.31$ points of success over the unmodified proposal prior, and feasibility evidence adds a further $2.14$ points on top of it. Appendix~\ref{sec:app_budget_selection} confirms the regularization row directly, in that enlarging the candidate set dilutes the effective weight of the skill-library preference and monotonically suppresses skill invocation, which is precisely the behavior the regularization term exists to prevent. Appendix~\ref{sec:app_awm_generalization} shows that the same pressure-dependent gain pattern reappears when the mapping is instantiated over a structurally different learned library, indicating that it tracks a mechanism rather than a property of one skill format.

\subsection{Executor Interface and Execution Loop}
\label{sec:mpcr_overview_contract}

At each decision step, MPCR receives the task goal, the current browser state, a summary of progress, the primitive browser actions, the base executor, and the fixed learned skill library. It returns one executable next action. The base executor's top action is not executed directly; it enters the stage sequence of Table~\ref{tab:app_swop_mpcr_map} as one candidate among the proposed set. The selected action is executed in the browser, and the observed page change, error, blocker, or timeout is incorporated into the next decision step.

Algorithm~\ref{alg:mpcr} summarizes the full MPCR execution loop. The loop exposes the four operational stages used throughout this appendix: proposal, verification, scoring, and reranking. It also shows how browser feedback is folded into the next decision step and where fallback behavior is invoked.

\begin{algorithm}[t]
\caption{MPCR: Masked Prior-Guided Candidate Reranking}
\label{alg:mpcr}
\begin{algorithmic}[1]
\Require Task goal $g$, browser state $B_0$, fixed skill library $\mathcal{L}$, primitive action set $\mathcal{A}_{\mathrm{prim}}$, base executor $E$, candidate budget $N_{\mathrm{cand}}$, shortlist size $M$, maximum horizon $T$
\Ensure Executed trajectory $\tau$
\State $\tau \gets \emptyset$, $h_0 \gets \emptyset$
\For{$t = 1,\ldots,T$}
    \State $o_t \gets \textsc{Observe}(B_{t-1})$
    \If{$\textsc{TaskComplete}(g,o_t,h_{t-1})$}
        \State \Return $\tau$
    \EndIf

    \Statex \textit{// Phase 1: Propose}
    \State $C_t \gets \textsc{Propose}_{E}(g,h_{t-1},o_t,\mathcal{L},\mathcal{A}_{\mathrm{prim}},N_{\mathrm{cand}})$

    \If{$C_t$ is invalid or empty}
        \State $a_t \gets \textsc{SafeRecovery}(o_t,h_{t-1})$
    \Else
        \Statex \textit{// Phase 2: Verify}
        \ForAll{$c \in C_t$}
            \State $v_t(c) \gets \textsc{AnnotateFeasibility}(c,o_t,B_{t-1})$
        \EndFor

        \Statex \textit{// Phase 3: Score}
        \ForAll{$c \in C_t$}
            \State $s_t(c) \gets \textsc{FeasibilityScore}(c,v_t(c),g,h_{t-1})$
        \EndFor
        \State $S_t \gets \textsc{TopM}(C_t,s_t,M)$

        \Statex \textit{// Phase 4: Rank}
        \State $a_t \gets \textsc{PairwiseRank}(S_t,g,h_{t-1},\{v_t(c),s_t(c)\}_{c\in S_t})$
        \If{$a_t$ is invalid, ambiguous, or outside $S_t$}
            \State $a_t \gets \arg\max_{c \in S_t} s_t(c)$
        \EndIf
    \EndIf

    \State $(B_t,f_t) \gets \textsc{Execute}(B_{t-1},a_t)$
    \State $\tau \gets \tau \cup \{(o_t,C_t,a_t,f_t)\}$
    \State $h_t \gets \textsc{UpdateProgress}(h_{t-1},a_t,f_t)$
\EndFor
\State \Return $\tau$
\end{algorithmic}
\end{algorithm}

The remaining subsections unpack each component of Algorithm~\ref{alg:mpcr}. Section~\ref{sec:mpcr_propose} specifies the constrained proposal prompt and the structured candidate interface. Section~\ref{sec:mpcr_verify} describes the browser-state evidence used to annotate local feasibility. Section~\ref{sec:mpcr_score} explains how these annotations are converted into conservative shortlisting scores. Section~\ref{sec:mpcr_rank} specifies the constrained prior-guided reranking prompt. Section~\ref{sec:mpcr_recovery} describes the fallback rules that make the executor total under malformed generations, uncertain probes, ambiguous rankings, and weak shortlists.

\subsection{Candidate Interface and Proposal}
\label{sec:mpcr_propose}

The proposal stage defines the candidate interface used by all downstream MPCR components. It instantiates the $\textsc{Propose}_{E}$ call in Algorithm~\ref{alg:mpcr} with a strict structured-output contract, so that verification, scoring, and ranking can be applied deterministically. At each decision step, the proposal module must return exactly $N_{\mathrm{cand}}$ candidates. Each candidate contains a complete executable action string, including both the action name and its arguments, together with a concise rationale grounded in the current observation.

\begin{tcolorbox}[
colback=back_color,
colframe=frame_color,
title=Constrained Candidate-Proposal Prompt,
fonttitle=\bfseries,
breakable,
left=4pt,
right=4pt
]
\textbf{Instruction.}
Review the current page state, previous execution feedback, and all provided task information to propose the best next actions for accomplishing the goal. The output will be interpreted and executed by a browser-control program, so it must follow the required format exactly.\\

\textbf{Inputs.}
The prompt is instantiated with the following runtime fields:
\begin{itemize}
    \item task goal;
    \item current progress summary;
    \item currently open tabs and active URL;
    \item current page accessibility tree;
    \item current page screenshot, when available;
    \item available primitive actions and fixed-library skill actions;
    \item previous actions and execution feedback, when available.
\end{itemize}

\ 

\textbf{Candidate-generation rule.}
Produce exactly $N_{\mathrm{cand}}$ next-action candidates, ranked from best to worst. The candidates may use the same action type with different arguments. The ordering is treated as a proposal prior, but the final action is selected only after verification, scoring, and reranking.\\

\textbf{Action constraints.}
Each candidate must be an executable action string from the provided action space. Do not output an action outside the available primitive or skill actions. Do not invent hidden capabilities, unavailable tools, or unprovided skill calls.\\

\textbf{Reason constraint.}
Each candidate must include a concise reason grounded in the current observation, task goal, and available action space.\\

\textbf{Output.}
Return valid JSON only. Do not output markdown or extra text.

\begin{verbatim}
{
  "candidates": [
    {
      "action": "click | type | navigate | scroll | skill_call | ...",
      "reason": "brief justification"
    }
  ]
}
\end{verbatim}
\end{tcolorbox}

This contract serves three purposes. First, it bounds the action set considered at each decision step. Second, it makes the candidate set parseable by deterministic verification and scoring modules. Third, it exposes primitive actions and fixed-library skills through the same interface, while preventing the proposal module from bypassing MPCR with an unstructured direct action.

The proposal stage is therefore not a broad search procedure. It is a bounded prior-guided expansion of the base executor's next-action preferences. MPCR preserves several plausible alternatives before committing to one action, but all later decisions remain constrained to the generated candidate set.

\subsection{Local Feasibility Annotation}
\label{sec:mpcr_verify}

MPCR collects local feasibility evidence from the browser state at the moment of decision. The design goal is to add a small amount of grounded execution evidence without turning the executor into a broad search procedure or a new learning algorithm. The evidence layer is intentionally lightweight: it checks properties that are directly relevant to immediate action feasibility, while leaving final selection to the scoring and ranking stages.

\subsubsection{Element-Level Readiness Evidence}

For candidates that refer to an interface element, MPCR checks whether the target appears present and actionable in the current page state. The annotation considers whether the element can be found through available identifiers or selectors, whether it is visible, whether it appears enabled, whether it has a meaningful bounding region, whether it is inside the current viewport, and whether it may be blocked by another element.

The annotation is not used as a universal hard filter. Browser state can be incomplete or transient, and some useful actions may require scrolling, focusing, waiting, or navigating before the final target becomes available. MPCR therefore treats element-level evidence as a soft feasibility signal. This keeps the executor conservative without making it brittle.

\subsubsection{Viewport and Reachability Evidence}

MPCR also records page-level reachability evidence, especially whether the current viewport may omit task-relevant content that is reachable through scrolling. This distinction matters because an element that is not currently visible is not necessarily absent. In long pages, administrative dashboards, product tables, issue lists, and content-management interfaces, relevant controls often appear below the fold or inside scrollable regions.

The reachability annotation helps the executor avoid conflating ``not currently visible'' with ``not available.'' If the page appears scrollable, an otherwise weak candidate may be interpreted as requiring an intermediate exposure action rather than being discarded. Conversely, if the page offers little additional reachable content, the executor can treat missing or hidden targets as stronger evidence against immediate execution.

This design supports MPCR's conservative adaptation principle: use current browser evidence to reduce obvious execution mistakes, but preserve plausible recovery paths when the evidence indicates that the page state may be incomplete.

\subsection{Conservative Feasibility Scoring}
\label{sec:mpcr_score}

The scoring stage is a bounded and reproducible intermediate selection layer. Its role is not to define an optimal policy, but to provide a transparent conservative prior over immediate executability before the final pairwise comparison. This makes the ranker compare candidates that have already been grounded in the current page state, rather than asking it to resolve semantic relevance, feasibility, and risk from scratch.

For each candidate $c \in C_t$, MPCR computes a heuristic feasibility-aware score:
\[
s_t(c) =
\lambda_{\mathrm{rel}} R_t(c)
+ \lambda_{\mathrm{prog}} P_t(c)
+ \lambda_{\mathrm{ready}} E_t(c)
+ \lambda_{\mathrm{expose}} X_t(c)
+ \lambda_{\mathrm{stable}} U_t(c)
- \lambda_{\mathrm{risk}} Z_t(c),
\]
where $R_t(c)$ measures task relevance, $P_t(c)$ estimates immediate progress, $E_t(c)$ captures current-page readiness, $X_t(c)$ captures whether the action may expose required information or controls, $U_t(c)$ captures stability of the action structure or grounding evidence, and $Z_t(c)$ captures execution risk. The coefficients are fixed across evaluations.

The evidence families have the following interpretations:
\begin{itemize}
    \item \textbf{Task relevance} measures whether the candidate is semantically aligned with the user instruction and the current progress state.
    \item \textbf{Progress} estimates whether the candidate is likely to move the task forward immediately.
    \item \textbf{Readiness} measures whether the target appears executable from the current page state.
    \item \textbf{Exposure} rewards actions that reveal information, controls, or page regions needed for later progress.
    \item \textbf{Stability} captures whether the candidate is supported by robust identifiers, reliable action structure, or unambiguous grounding evidence.
    \item \textbf{Risk} penalizes actions that are destructive, irreversible, speculative, or likely to move the agent away from the task.
\end{itemize}

The scoring rule also accounts for page-region and readiness cues in a conservative way. Actions in task-relevant regions, such as main content, navigation, forms, or management panels, receive more favorable treatment than actions in low-value regions. Candidates that appear ready now are preferred over candidates that appear blocked or require unresolved preconditions. However, these cues are deliberately soft. A candidate marked as \texttt{blocked} or \texttt{needs\_scroll} is downweighted, but it may still survive if it remains highly task-relevant and no stronger alternative exists.

This soft treatment is central to MPCR's design. Under environment evolution, feasibility evidence is informative but imperfect. A hard filter can remove the correct recovery action when the page state is partially observed; an unconstrained prior can overcommit to an action that is locally infeasible. Conservative scoring occupies the middle ground: it reduces the priority of fragile candidates while preserving enough action-domain coverage for the ranker to make a context-sensitive comparison.

The scoring rule is fixed across the reported evaluations. This is important for attribution. MPCR's gains are not explained by per-condition tuning or by a hidden adaptive learner; they arise from applying the same feasibility-aware selection interface across the fixed learned library.

\subsection{Constrained Prior-Guided Reranking}
\label{sec:mpcr_rank}

The final ranking step uses the LLM as a constrained local comparator over the shortlist $S_t$, following the broader use of LLMs for reranking and pairwise comparison~\cite{sun2023chatgptsearch,qin2024prp,wu2025pairwisereranking}. The ranker is conditioned on the task state, candidate rationales, heuristic scores, and feasibility annotations. It must choose among the provided candidates only.

This constrained design preserves the prior-guided nature of MPCR. The proposal stage reflects the prior induced by the base executor and fixed skill library, while the verification and scoring stages expose current-page feasibility evidence. The ranker then compares a small shortlist using both sources of information. It is not allowed to invent a new action, call an unseen skill, or change the candidate action strings.

In the reported experiments, the shortlist size is set to $M=2$, so the pairwise comparator is invoked once per decision step. For larger shortlists, the same comparator can be applied in a deterministic tournament order sorted by the heuristic score. The current winner is compared against the next candidate, and ties are broken by the heuristic score. This keeps reranking bounded while ensuring that the final action is selected only from the verified shortlist.

\begin{tcolorbox}[
colback=back_color,
colframe=frame_color,
title=Constrained Pairwise Reranking Prompt,
fonttitle=\bfseries,
breakable,
left=4pt,
right=4pt
]
\textbf{Role.} You are a strict pairwise ranker for web-agent actions.\\

\textbf{Objective.} Choose the better next single action right now. Prefer the action that is more executable, more directly useful for the task, more likely to unblock progress, and lower risk.\\

\textbf{Decision criteria.}
\begin{enumerate}
    \item Prefer the candidate that better advances the current task goal.
    \item Prefer the candidate that appears executable in the current page state.
    \item Prefer actions that reveal required information or remove blockers.
    \item Avoid destructive, irreversible, or speculative actions unless required.
\end{enumerate}

\ 

\textbf{Executability evidence.}
Treat a candidate as weaker if its target is missing, invisible, disabled, occluded, or outside the reachable page state. For skill calls or multi-step candidates, compare them using their signatures, descriptions, current context, and available feasibility annotations.\\

\textbf{Candidate constraint.}
Choose only among the provided candidates. Do not invent a new action, call an unseen skill, assume hidden skill effects, or change the candidate action strings.\\

\textbf{Tie rule.}
Return \texttt{tie} only when neither candidate clearly dominates. If one action is moderately better in immediate usefulness or executability, choose it.\\

\textbf{Output.}
Return strict JSON only:
\begin{verbatim}
{"winner":"a"|"b"|"tie","why":"<=80 words"}
\end{verbatim}
\end{tcolorbox}

The pairwise prompt is intentionally local. It asks only which candidate is better as the next action under the current browser state. This prevents the ranker from replacing the executor with an open-ended planner and keeps MPCR aligned with its role as a test-time reranking mechanism.

\subsection{Totality, Fallbacks, and Recovery}
\label{sec:mpcr_recovery}

MPCR includes explicit recovery behavior so that the executor remains well-defined under parser uncertainty, incomplete probes, or ranker ambiguity. These fallbacks are safety guards, not additional adaptation mechanisms. They ensure that the executor always returns an action, while preserving the fixed-library and candidate-constrained design.

\begin{itemize}
    \item \textbf{Generation failure.} If candidate generation returns invalid JSON or an unusable candidate set, the executor follows a safe default recovery action.
    \item \textbf{Probe uncertainty.} If a candidate cannot be confidently verified, it is retained but downweighted through readiness-related penalties.
    \item \textbf{Ranking ambiguity.} If the pairwise ranker returns invalid or ambiguous output, the executor selects the highest-scored candidate from the heuristic stage.
    \item \textbf{Weak shortlist.} If all shortlisted candidates appear blocked or insufficiently ready, the executor still returns the highest-scored candidate, ensuring that the action-selection procedure remains total.
\end{itemize}

The guiding principle is conservative continuity. MPCR uses additional deliberation when feasibility signals are informative, but degrades to deterministic behavior when structured judgments are noisy. This makes the deployable executor robust to malformed model outputs, incomplete browser probes, transient page states, and ambiguous comparisons without introducing skill relearning or unconstrained exploration.
\subsection{Design Validation}
\label{sec:app_design_validation}

The preceding subsections specify what MPCR does. This subsection provides the empirical basis for two design decisions embedded in that specification: that each of the four stages is necessary, and that the default candidate budget is $N_{\mathrm{cand}}=2$.

Both studies are run on Magento with GPT-5 mini over the same fixed ASI library and the same six evaluation conditions used in Section~\ref{sec:experiments}. Because the main tables pool GitLab, Magento, and WordPress, absolute values here are not comparable to Table~\ref{tab:mpcr_gpt5mini_asi}; full MPCR averages $39.18\%$ on this single-site subset against $41.26\%$ pooled across three sites. All comparisons below are made within this subset.

\subsubsection{Component Ablation}
\label{sec:app_component_ablation}

We ablate the verification, scoring, and reranking stages of Algorithm~\ref{alg:mpcr}, keeping the proposal stage in every variant since it defines the candidate interface. Table~\ref{tab:app_component_ablation} reports success rate and average skill calls, averaged over the six conditions.

\begin{table}[htbp]
\centering
\small
\setlength{\tabcolsep}{4pt}
\renewcommand{\arraystretch}{1.2}
\begin{tabular}{@{}lcccrr@{}}
\toprule
\textbf{Variant} & \textbf{Ver.} & \textbf{Sco.} & \textbf{Rer.} & \textbf{SR (\%)} & \textbf{ASC} \\
\midrule
Propose only            & \xmark & \xmark & \xmark          & 32.73 & 0.60 \\
Verify + Score          & \cmark & \cmark & \xmark          & 27.93 & 0.15 \\
Rerank only (blind)     & \xmark & \xmark & \cmark          & 37.04 & 0.68 \\
Verify + Score + Random & \cmark & \cmark & rand.           & \phantom{0}8.09 & 0.31 \\
\midrule
\textbf{MPCR (full)}    & \cmark & \cmark & \cmark          & \textbf{39.18} & 0.24 \\
\bottomrule
\end{tabular}
\caption{Component ablation of MPCR on Magento, averaged over the six evaluation conditions. ``Ver.'', ``Sco.'', and ``Rer.'' denote the verification, scoring, and reranking stages. ``Rerank only'' applies the pairwise comparator to the raw proposal order without feasibility evidence, and therefore issues the same number of reranking calls as full MPCR. ``Verify + Score + Random'' keeps the verified shortlist but replaces the comparator with a fixed-seed random choice, and therefore issues no reranking call at all.}
\label{tab:app_component_ablation}
\end{table}

\paragraph{Additional deliberation is not what produces the gain.}
The randomized control is decisive on this point. It shares the pipeline, the feasibility annotations, and the shortlist with full MPCR, and differs only in that the final comparison is a fixed-seed coin flip that emits no reranking call. Success collapses from $39.18\%$ to $8.09\%$, far below the $32.73\%$ obtained by taking the base executor's first proposal and doing nothing else. Since the scaffolding is held constant and the compute is strictly lower, the gain cannot be attributed to the pipeline structure or to spending an additional model call per step. What matters is the informational content of the comparison. The same control also shows that a bounded shortlist is not self-improving: selecting arbitrarily within it is much worse than not shortlisting at all.

\paragraph{The gain decomposes across stages.}
Adding blind reranking to the proposal prior raises success from $32.73\%$ to $37.04\%$, and adding feasibility evidence on top of it raises success further to $39.18\%$. The second increment, $2.14$ points, is smaller than the first but consistent, holding in five of the six conditions. Neither stage subsumes the other: full MPCR exceeds both ``Rerank only,'' which has the comparator but no feasibility evidence, and ``Verify + Score,'' which has feasibility evidence but no comparator.

\paragraph{Feasibility evidence mainly buys reliable skill use.}
Reading success alone understates what the verification and scoring stages contribute. Full MPCR attains its higher success while issuing $0.24$ skill calls per task against $0.68$ for blind reranking, a reduction of roughly threefold. The effect concentrates exactly where it should: under Execution Shift, full MPCR reduces skill invocation to $0.04$ calls per task, that is, it correctly stops calling a skill that the runtime blocker has made unexecutable, whereas blind reranking continues to invoke and fail it. The feasibility mask is therefore better described as a mechanism for avoiding infeasible skill execution than as a primary driver of success.

\paragraph{Why feasibility evidence alone underperforms.}
``Verify + Score'' reaches only $27.93\%$, below the $32.73\%$ of the unmodified proposal prior. We read this as a property of the design rather than as a defect in the evidence. Feasibility annotations are constructed to be consumed by the comparator, which weighs them against task context; taking a hard top-1 by heuristic score discards that contextual judgment and commits to whichever candidate the fixed coefficients happen to favor. The conservative scoring rule of Section~\ref{sec:mpcr_score} is deliberately soft for this reason. That the evidence helps when it informs the comparator and hurts when it replaces it is direct evidence that the two stages are complementary rather than separable.

\subsubsection{Candidate-Budget Selection}
\label{sec:app_budget_selection}

MPCR is evaluated with $N_{\mathrm{cand}}=2$, where the candidate count equals the shortlist size, $M=N_{\mathrm{cand}}$. We sweep $N_{\mathrm{cand}} \in \{2,3,4\}$ to establish that this default is an empirical choice rather than an under-resourced configuration. Table~\ref{tab:app_budget_percondition} reports success rate and average skill calls under each of the six conditions, and Table~\ref{tab:app_budget_pooled} pools the same results by shift type.

\newsavebox{\budgetleftbox}
\begin{table}[htbp]
\centering
\sbox{\budgetleftbox}{%
\begin{minipage}[b]{0.5\linewidth}
\centering
\small
\setlength{\tabcolsep}{3.2pt}
\renewcommand{\arraystretch}{1.2}
\begin{tabular}{@{}lrrrrrr@{}}
\toprule
\multirow{2}{*}[-0.9ex]{\textbf{Condition}}
& \multicolumn{2}{c}{$N_{\mathrm{cand}}{=}2$}
& \multicolumn{2}{c}{$N_{\mathrm{cand}}{=}3$}
& \multicolumn{2}{c}{$N_{\mathrm{cand}}{=}4$} \\
\cmidrule(lr){2-3}\cmidrule(lr){4-5}\cmidrule(lr){6-7}
& \textbf{SR} & \textbf{ASC} & \textbf{SR} & \textbf{ASC} & \textbf{SR} & \textbf{ASC} \\
\midrule
\texttt{Src}      & 51.75 & 0.35 & 44.23 & 0.08 & 48.08 & 0.01 \\
\texttt{Src-Perc} & 56.14 & 0.26 & 40.38 & 0.12 & 50.96 & 0.07 \\
\texttt{Src-Exec} & 30.70 & 0.39 & 42.31 & 0.05 & 41.35 & 0.03 \\
\texttt{Tgt}      & 40.35 & 0.18 & 36.84 & 0.04 & 32.46 & 0.00 \\
\texttt{Tgt-Perc} & 37.72 & 0.23 & 31.58 & 0.01 & 30.70 & 0.00 \\
\texttt{Tgt-Exec} & 18.42 & 0.04 & 20.18 & 0.02 & 22.81 & 0.01 \\
\midrule
\textbf{Mean}     & \textbf{39.18} & 0.24 & 35.92 & 0.05 & 37.73 & 0.02 \\
\bottomrule
\end{tabular}
\caption{Per-condition success rate (\%) and average skill calls under candidate budgets $N_{\mathrm{cand}} \in \{2,3,4\}$ on Magento.}
\label{tab:app_budget_percondition}
\end{minipage}}%
\usebox{\budgetleftbox}\hfill
\raisebox{-\dp\budgetleftbox}{%
\begin{minipage}[b][\dimexpr\ht\budgetleftbox+\dp\budgetleftbox\relax][s]{0.46\linewidth}
\centering
\small
\setlength{\tabcolsep}{2.6pt}
\renewcommand{\arraystretch}{1.2}
\vfill
\begin{tabular}{@{}lrrr@{}}
\toprule
\textbf{Pooled view} & $N_{\mathrm{cand}}{=}2$ & $N_{\mathrm{cand}}{=}3$ & $N_{\mathrm{cand}}{=}4$ \\
\midrule
Non-execution & \textbf{46.49} & 38.26 & 40.55 \\
Execution     & 24.56 & 31.25 & \textbf{32.08} \\
\midrule
Overall       & \textbf{39.18} & 35.92 & 37.73 \\
\bottomrule
\end{tabular}
\vfill
\caption{Success rate (\%) pooled by shift type. ``Non-execution'' pools \texttt{Src}, \texttt{Src-Perc}, \texttt{Tgt}, and \texttt{Tgt-Perc}; ``Execution'' pools \texttt{Src-Exec} and \texttt{Tgt-Exec}.}
\label{tab:app_budget_pooled}
\end{minipage}}
\end{table}

\paragraph{Aggregate success is non-monotone and peaks at the default.}
Overall success is $39.18\%$, $35.92\%$, and $37.73\%$ for $N_{\mathrm{cand}}=2,3,4$. Enlarging the candidate set does not improve transfer on aggregate, which establishes the default as an empirically favorable lightweight operating point rather than a budget-constrained compromise. Given a single run per budget, we do not draw conclusions from the exact ordering of $N_{\mathrm{cand}}=3$ against $N_{\mathrm{cand}}=4$.

\paragraph{Budget interacts with shift type.}
The aggregate conceals the more informative pattern, shown in Table~\ref{tab:app_budget_pooled}. Under Execution Shift, larger budgets help: pooled success rises from $24.56\%$ to $31.25\%$ to $32.08\%$, and \texttt{Tgt-Exec} increases monotonically from $18.42\%$ to $22.81\%$. Under clean and perceptual conditions, larger budgets hurt: pooled success falls from $46.49\%$ to $40.55\%$. The interaction runs in the same direction on the source and target sides independently, which makes it the most defensible conclusion of this sweep. It is also the reading our framing predicts. When local executability is fragile, additional candidates give the feasibility-aware ranker more room to recover; when the learned prior is already reliable, additional candidates mostly introduce distractors.

\paragraph{Enlarging the candidate set dilutes the skill prior.}
Skill invocation collapses monotonically with the budget, from $0.24$ to $0.05$ to $0.02$ calls per task. Since the library and the proposal interface are unchanged, this indicates that a wider candidate set reduces the effective weight of the skill-library preference, so the agent progressively stops invoking skills it has already learned. This is the empirical counterpart of the KL regularization toward $\pi_{\mathrm{ref}}$ in SWOP and of the correspondence stated in Table~\ref{tab:app_swop_mpcr_map}. It also explains the interaction above: on clean and perceptual conditions the diluted prior was worth keeping, so dilution costs success, whereas under Execution Shift the prior is unreliable in the first place and the trade is favorable. Since the favorable direction of this trade-off flips with shift type, a fixed budget cannot be optimal for both regimes; a shift-adaptive candidate budget, widened under detected execution pressure and kept narrow otherwise, is a promising extension we leave to future work.

\section{Experimental Details and Supplementary Analyses}
\label{sec:appendix-d:1}

This appendix provides supplementary details for the experiments in Section~\ref{sec:experiments}. We first describe the compared learned-skill methods and their implementation protocols, then report descriptive skill-library statistics after learning, and provide backbone-level diagnostic results showing that the observed transfer patterns are not tied to a single LLM backbone. The remaining subsections report supplementary analyses around the main comparison: transfer of MPCR to a second skill-learning method, stability across random seeds, robustness to more aggressive Execution Shift schedules, and a comparison against relearning the skill library in the target environment. These supplementary studies are run on Magento, whereas the main tables pool GitLab, Magento, and WordPress; absolute values are therefore not comparable across the two, and each subsection states its scope.

\subsection{Compared-Method Cards}\label{sec:appendix_method_cards}

The compared published skill-learning baselines are AWM, ASI, SkillWeaver, and WALT~\cite{wang2025agentworkflowmemory,wang2025inducingprogrammaticskills,zheng2025skillweaver,prabhu2026walt}. These four systems provide representative coverage of design choices in skill representation, integration mechanism, learning signal source, and verification rigor. Concretely, they cover textual, executable, API-function, and tool-based skill forms; prompt injection, action-space extension, and tool-call integration; and episode-driven, exploration-driven, and function-discovery learning signals.

The four published baselines also differ in their underlying agent frameworks and action-selection strategies~\cite{wang2025agentworkflowmemory,wang2025inducingprogrammaticskills,zheng2025skillweaver,prabhu2026walt}. They observe page state through DOM, accessibility-tree, or snapshot representations, but differ in execution stack and decision mechanism. Such heterogeneity is an inherent consequence of comparing published systems: each method is tightly coupled to its native execution stack, and re-implementing all methods within a single framework would risk distorting their original design intent. Our aim is therefore to evaluate representative end-to-end skill-learning systems under shared tasks and perturbation conditions. The disturbance layer is applied consistently across backends, although this does not remove all execution-stack confounds. We provide detailed method cards below, focusing on each method's skill representation, acquisition signal, verification mechanism, execution interface, and relevance to the fixed-library transfer comparison.

\textbf{AWM} (Agent Workflow Memory) represents learned procedures as \emph{textual workflows} stored in agent memory~\cite{wang2025agentworkflowmemory}. Each workflow consists of a high-level natural-language description and a sequence of steps, where each step records the environment state, the agent's reasoning, and the executed action. Workflows are induced by prompting the LLM to extract common sub-routines from successful episodes, abstracting concrete values into variables (e.g., replacing a specific product name with \texttt{\{product\_name\}}). At inference time, relevant workflows are retrieved and injected into the LLM prompt as contextual guidance. Because workflows are consumed as text rather than executed as code, AWM does not extend the agent's action space. It therefore serves as a prompt-level workflow-reuse baseline rather than an executable-skill baseline.

\textbf{ASI} (Inducing Programmatic Skills for Agentic Tasks) represents skills as \emph{executable Python programs}. Given a successful episode, a skill induction module generates one or more Python functions (e.g., \texttt{search\_product(name)}) that encapsulate reusable sub-procedures~\cite{wang2025inducingprogrammaticskills}. Each induced skill undergoes a three-stage verification: (1)~\emph{correctness}---the task is still completed when the original trajectory is rewritten to invoke the new skill; (2)~\emph{skill usage}---the rewritten trajectory actually calls the skill; and (3)~\emph{skill validity}---every skill invocation causes an observable environment change~\cite{wang2025inducingprogrammaticskills}. Only skills that pass all three checks are added to the skill library. During evaluation, induced skills are available as callable high-level actions alongside the primitive action set.

\textbf{SkillWeaver} acquires skills through a \emph{self-driven curriculum}: the agent autonomously proposes candidate skills by observing the website, practices them through active exploration, distills successful execution traces into reusable Playwright API functions, and iteratively debugs these functions through automated test cases~\cite{zheng2025skillweaver}. Unlike ASI and AWM, which learn only from task-solving episodes, SkillWeaver generates learning opportunities through environment exploration independent of the benchmark task set. The resulting skills are Python \texttt{async} functions that operate on Playwright \texttt{Page} objects and are composed hierarchically, with higher-level skills calling lower-level ones~\cite{zheng2025skillweaver}. SkillWeaver therefore represents an exploration-driven API-function approach to skill acquisition.

\textbf{WALT} (Web Agents that Learn Tools) takes a fundamentally different approach by \emph{reverse-engineering website-native functionality} into structured, callable tools. WALT first discovers candidate website functions through systematic exploration, then applies a \emph{demonstrate--generate--validate} cycle: a browser agent demonstrates the function end-to-end, a tool-generation agent converts the execution trace into a structured tool definition with typed input parameters and step sequences, and a test agent validates the tool against pre-selected inputs~\cite{prabhu2026walt}. When possible, WALT further optimizes tools by replacing UI interaction sequences with direct URL-level or API-level operations~\cite{prabhu2026walt}. The resulting tools are deterministic action scripts with formal schemas~\cite{prabhu2026walt}.

\subsection{Skill-Library Supplements}
\label{sec:appendix_learning_budgets}

\begin{table*}
\centering
\small
\setlength{\tabcolsep}{4pt}
\renewcommand{\arraystretch}{1.3}
\begin{tabularx}{\linewidth}{@{} l *{9}{>{\centering\arraybackslash}X} @{}}
\toprule
& \multicolumn{3}{c}{\textbf{DeepSeek-V3.2}} & \multicolumn{3}{c}{\textbf{GPT-5 mini}} & \multicolumn{3}{c}{\textbf{Claude Haiku 4.5}} \\
\cmidrule(lr){2-4} \cmidrule(lr){5-7} \cmidrule(lr){8-10}
\multirow{2}{*}[3.7ex]{\textbf{Method}} & GitLab & Magento & WordPress & GitLab & Magento & WordPress & GitLab & Magento & WordPress \\
\midrule
AWM         & 35 & 23 &  5 & 27 & 15 &  6 & 10 &  6 &  3 \\
ASI         & 14 & 19 & 12 &  9 & 13 &  4 & 11 & 13 & 10 \\
SkillWeaver & 28 & 18 & 23 &  9 & 26 & 15 & 12 & 17 & 14 \\
WALT        &  8 &  8 &  6 &  6 &  7 &  5 &  5 &  6 &  3 \\
\bottomrule
\end{tabularx}
\caption{Skill library size after learning. Counts are descriptive rather than directly comparable across methods because the retained units have different semantics.}
\label{tab:skill_library_appendix}
\end{table*}

Table~\ref{tab:skill_library_appendix} reports the number of skills (or workflows/tools) retained in each method's library after the learning phase, broken down by LLM backbone and website. This table serves as a descriptive diagnostic of the learning regimes rather than a larger-is-better comparison. The counted units are not semantically identical across methods: an AWM item is a textual workflow, an ASI item is an executable Python skill, a SkillWeaver item is a Playwright-style API function, and a WALT item is a structured tool~\cite{wang2025agentworkflowmemory,wang2025inducingprogrammaticskills,zheng2025skillweaver,prabhu2026walt}. Therefore, the table is not used to normalize methods by raw library size. Instead, it shows how different skill-acquisition mechanisms expose different amounts and types of reusable structure before transfer evaluation.

Several patterns are worth noting. First, library size varies substantially by method. Averaged over the nine website--backbone combinations, SkillWeaver produces the largest libraries (18.0 items), followed by AWM (14.4), ASI (11.7), and WALT (6.0). This ordering reflects the methods' acquisition mechanisms. SkillWeaver actively explores the environment and can therefore accumulate many candidate API functions. AWM extracts reusable textual routines from successful trajectories, which also permits relatively large libraries when successful episodes contain repeated sub-procedures. In contrast, ASI and WALT apply stronger executable or tool-level validation, yielding smaller but more constrained libraries. WALT is especially compact, ranging only from 3 to 8 tools across all settings, because its demonstrate--generate--validate pipeline filters candidate tools before evaluation.

Second, the table shows that library scale is strongly affected by both the backbone and the website. DeepSeek-V3.2 generally produces more retained items than GPT-5 mini or Claude Haiku 4.5, especially for AWM and SkillWeaver. Across all methods and websites, the total number of retained items is 199 for DeepSeek-V3.2, 142 for GPT-5 mini, and 110 for Claude Haiku 4.5. This suggests that acquisition scale and transfer robustness should be distinguished. 
Library size measures how many reusable items are retained after learning, whereas transfer performance also depends on executor-side selection, grounding, and execution under changed environment conditions.

Third, website structure also affects how much reusable behavior can be extracted. Aggregating over methods and backbones, GitLab and Magento yield substantially larger libraries than WordPress, with 174, 171, and 106 retained items respectively. This is consistent with the nature of the environments: developer-platform and e-commerce-administration tasks expose more repeated navigation, filtering, editing, and object-management patterns, whereas the WordPress task set induces a smaller set of reusable operations. This site-level variation is useful for SkillShift-Bench because it prevents the benchmark from being reduced to a single library-density regime. The same transfer protocol is evaluated under websites that differ in how much reusable structure they naturally expose.

Most importantly, Table~\ref{tab:skill_library_appendix} helps rule out a simple ``more skills imply better robustness'' explanation. SkillWeaver often produces the largest libraries, while WALT uses the smallest libraries because its tools are more heavily validated. ASI occupies an intermediate regime: it learns fewer items than AWM or SkillWeaver in many settings, but retains executable skills that are directly usable during interaction. These patterns show that library size mainly reflects acquisition scale and method design, rather than transfer robustness itself. A larger library may provide more reusable candidates, but successful transfer still depends on whether the agent can select, ground, and execute the retained items under changed environment conditions.

\subsection{Backbone Analyses}\label{sec:appendix_site_backbone}

\begin{table*}[t]
\centering
\small
\setlength{\tabcolsep}{5.7pt}
\renewcommand{\arraystretch}{1.35}
\begin{tabular}{ccccccccccccc}
\toprule
\multirow{2}{*}[-1.6ex]{\textbf{Metric}}
& \multirow{2}{*}[-1.6ex]{\textbf{LLM}}
& \multicolumn{4}{c}{\textbf{Method}}
& \multicolumn{6}{c}{\textbf{Environment}}
& \multirow{2}{*}[-1.6ex]{\textbf{O/A}} \\
\cmidrule(lr){3-6}
\cmidrule(lr){7-12}
&
& \textbf{AWM}
& \textbf{ASI}
& \makecell{\textbf{Skill-}\\\textbf{Weaver}}
& \textbf{WALT}
& \textbf{Src}
& \makecell{\textbf{Src-}\\\textbf{Perc}}
& \makecell{\textbf{Src-}\\\textbf{Exec}}
& \textbf{Tgt}
& \makecell{\textbf{Tgt-}\\\textbf{Perc}}
& \makecell{\textbf{Tgt-}\\\textbf{Exec}}
& \\
\midrule

\multirow{3}{*}{\makecell{SR \\(\%)}}
& \cellcolor{ds} DS
& \cellcolor{ds} 35.30 & \cellcolor{ds} \textbf{38.32} & \cellcolor{ds} 17.12 & \cellcolor{ds} 36.44
& \cellcolor{ds} \textbf{47.89} & \cellcolor{ds} 37.32 & \cellcolor{ds} 23.98 & \cellcolor{ds} 35.81 & \cellcolor{ds} 29.76 & \cellcolor{ds} 16.01
& \cellcolor{ds} 31.80 \\
& \cellcolor{gpt} GPT
& \cellcolor{gpt} \textbf{37.56} & \cellcolor{gpt} 37.90 & \cellcolor{gpt} \textbf{24.99} & \cellcolor{gpt} \textbf{40.46}
& \cellcolor{gpt} 46.30 & \cellcolor{gpt} \textbf{45.41} & \cellcolor{gpt} \textbf{27.14} & \cellcolor{gpt} \textbf{38.72} & \cellcolor{gpt} \textbf{33.58} & \cellcolor{gpt} \textbf{20.23}
& \cellcolor{gpt} \textbf{35.23} \\
& \cellcolor{cld} Cld
& \cellcolor{cld} 33.00 & \cellcolor{cld} 32.64 & \cellcolor{cld} 18.32 & \cellcolor{cld} 27.50
& \cellcolor{cld} 45.07 & \cellcolor{cld} 40.60 & \cellcolor{cld} 14.87 & \cellcolor{cld} 32.03 & \cellcolor{cld} 24.73 & \cellcolor{cld} 9.89
& \cellcolor{cld} 27.86 \\

\midrule

\multirow{3}{*}{\makecell{AS}}
& \cellcolor{ds} DS
& \cellcolor{ds} 7.35 & \cellcolor{ds} 6.94 & \cellcolor{ds} 17.53 & \cellcolor{ds} 15.90
& \cellcolor{ds} 8.65 & \cellcolor{ds} 10.45 & \cellcolor{ds} 14.35 & \cellcolor{ds} 11.25 & \cellcolor{ds} 12.03 & \cellcolor{ds} 14.84
& \cellcolor{ds} 11.93 \\
& \cellcolor{gpt} GPT
& \cellcolor{gpt} \textbf{6.20} & \cellcolor{gpt} \textbf{6.21} & \cellcolor{gpt} \textbf{12.80} & \cellcolor{gpt} 9.87
& \cellcolor{gpt} \textbf{7.25} & \cellcolor{gpt} \textbf{7.50} & \cellcolor{gpt} 11.31 & \cellcolor{gpt} \textbf{7.51} & \cellcolor{gpt} \textbf{8.56} & \cellcolor{gpt} 10.49
& \cellcolor{gpt} \textbf{8.77} \\
& \cellcolor{cld} Cld
& \cellcolor{cld} 7.47 & \cellcolor{cld} 7.51 & \cellcolor{cld} 15.05 & \cellcolor{cld} \textbf{6.62}
& \cellcolor{cld} 8.12 & \cellcolor{cld} 8.81 & \cellcolor{cld} \textbf{10.67} & \cellcolor{cld} 8.15 & \cellcolor{cld} 8.96 & \cellcolor{cld} \textbf{10.25}
& \cellcolor{cld} 9.16 \\

\midrule

\multirow{3}{*}{\makecell{ASC}}
& \cellcolor{ds} DS
& \cellcolor{ds} -- & \cellcolor{ds} 0.32 & \cellcolor{ds} 1.89 & \cellcolor{ds} 0.03
& \cellcolor{ds} 0.55 & \cellcolor{ds} 0.69 & \cellcolor{ds} 0.76 & \cellcolor{ds} 0.80 & \cellcolor{ds} 0.73 & \cellcolor{ds} 0.94
& \cellcolor{ds} 0.74 \\
& \cellcolor{gpt} GPT
& \cellcolor{gpt} -- & \cellcolor{gpt} 0.06 & \cellcolor{gpt} 2.16 & \cellcolor{gpt} 0.18
& \cellcolor{gpt} 0.74 & \cellcolor{gpt} 0.72 & \cellcolor{gpt} 0.80 & \cellcolor{gpt} 0.77 & \cellcolor{gpt} 0.85 & \cellcolor{gpt} 0.92
& \cellcolor{gpt} 0.80 \\
& \cellcolor{cld} Cld
& \cellcolor{cld} -- & \cellcolor{cld} 0.11 & \cellcolor{cld} 0.00 & \cellcolor{cld} 0.05
& \cellcolor{cld} 0.07 & \cellcolor{cld} 0.06 & \cellcolor{cld} 0.07 & \cellcolor{cld} 0.05 & \cellcolor{cld} 0.05 & \cellcolor{cld} 0.03
& \cellcolor{cld} 0.05 \\

\midrule

\multirow{3}{*}{\makecell{SFR \\(\%)}}
& \cellcolor{ds} DS
& \cellcolor{ds} -- & \cellcolor{ds} 39.51 & \cellcolor{ds} 0.21 & \cellcolor{ds} \textbf{19.25}
& \cellcolor{ds} 5.09 & \cellcolor{ds} 19.71 & \cellcolor{ds} 33.38 & \cellcolor{ds} \textbf{7.10} & \cellcolor{ds} 18.89 & \cellcolor{ds} 33.78
& \cellcolor{ds} 19.66 \\
& \cellcolor{gpt} GPT
& \cellcolor{gpt} -- & \cellcolor{gpt} \textbf{13.86} & \cellcolor{gpt} 5.71 & \cellcolor{gpt} 20.98
& \cellcolor{gpt} \textbf{4.15} & \cellcolor{gpt} \textbf{7.46} & \cellcolor{gpt} \textbf{21.26} & \cellcolor{gpt} 11.77 & \cellcolor{gpt} 16.68 & \cellcolor{gpt} \textbf{19.77}
& \cellcolor{gpt} \textbf{13.52} \\
& \cellcolor{cld} Cld
& \cellcolor{cld} -- & \cellcolor{cld} 30.71 & \cellcolor{cld} \textbf{0.00} & \cellcolor{cld} 32.15
& \cellcolor{cld} 10.89 & \cellcolor{cld} 21.39 & \cellcolor{cld} 32.05 & \cellcolor{cld} 19.96 & \cellcolor{cld} \textbf{16.08} & \cellcolor{cld} 25.36
& \cellcolor{cld} 20.95 \\

\bottomrule
\end{tabular}
\caption{
Expanded transfer-oriented diagnostic results across methods and environment settings. 
O/A denotes the overall average across the corresponding settings. 
LLM abbreviations are as follows: DS = DeepSeek-V3.2, GPT = GPT-5 mini, and Cld = Claude Haiku 4.5~\cite{deepseek2025v32,openai2025gpt5mini,anthropic2025haiku45}.
Bold values mark the best SR, AS, and SFR entries.
ASC is reported as a non-directional diagnostic.
``--'' denotes non-applicable metrics: AWM uses textual workflows, so executable-skill metrics are undefined.
}
\label{tab:transfer_aggregate_expanded_llm}
\end{table*}

Table~\ref{tab:transfer_aggregate_expanded_llm} provides a backbone-level decomposition of the transfer results and shows that the benchmark stress is not an artifact of a single LLM choice. GPT-5 mini obtains the highest overall success (SR = 35.23\%), the lowest skill failure rate (SFR = 13.52\%), and the lowest trajectory cost (AS = 8.77). DeepSeek-V3.2 remains competitive in overall success (SR = 31.80\%), but requires substantially longer trajectories (AS = 11.93), suggesting stronger exploration or recovery at a higher interaction cost. Claude Haiku 4.5 has lower overall success (SR = 27.86\%) and more frequent skill failures (SFR = 20.95\%), indicating less stable transfer under shifted execution conditions.

The environment columns show a consistent degradation pattern across all three backbones. Clean source-context performance is highest or near-highest for each model, but success decreases under target-context transfer and drops most sharply under execution perturbations. For example, GPT-5 mini decreases from 46.30\% on \texttt{Src} to 38.72\% on \texttt{Tgt}, and further to 27.14\% on \texttt{Src-Exec} and 20.23\% on \texttt{Tgt-Exec}. DeepSeek-V3.2 and Claude Haiku 4.5 follow the same qualitative trend, with \texttt{Tgt-Exec} reaching only 16.01\% and 9.89\%, respectively. Thus, stronger backbones improve the absolute level of transfer, but do not remove the difficulty ordering induced by Contextual Shift, Perceptual Shift, and especially Execution Shift.

The trajectory-cost and skill-failure rows further separate two failure channels. Execution-shift settings increase the interaction burden for all backbones, with AS on \texttt{Tgt-Exec} reaching 14.84 for DeepSeek-V3.2, 10.49 for GPT-5 mini, and 10.25 for Claude Haiku 4.5. Skill failures also become more frequent in the same setting, with SFR reaching 33.78\%, 19.77\%, and 25.36\%, respectively. This indicates that Execution Shift does not only make tasks longer. It also makes selected skills or tool actions less locally executable. GPT-5 mini is comparatively more robust because it combines higher success with lower SFR, rather than relying on longer trajectories.

The method columns reveal additional interactions between backbone capability and skill format. SkillWeaver remains the weakest method under all three backbones, although GPT-5 mini improves its success to 24.99\%, compared with 17.12\% for DeepSeek-V3.2 and 18.32\% for Claude Haiku 4.5. WALT benefits strongly from GPT-5 mini and reaches 40.46\% success, but its SFR remains nontrivial at 20.98\%, showing that structured tools can still face local executability failures. ASI is competitive with both DeepSeek-V3.2 and GPT-5 mini, but its skill failures are much less frequent with GPT-5 mini (SFR = 13.86\%) than with DeepSeek-V3.2 (39.51\%) or Claude Haiku 4.5 (30.71\%). This suggests that the same executable-skill framework can behave differently depending on the backbone's ability to select, ground, and recover from skill executions.

Finally, skill-call frequency should be interpreted together with success and SFR, rather than as a standalone performance metric. GPT-5 mini and DeepSeek-V3.2 invoke executable skills at similar overall frequencies (ASC = 0.80 and 0.74), but GPT-5 mini achieves higher success and lower SFR, indicating more reliable skill use. Claude Haiku 4.5 has much lower skill-call frequency (ASC = 0.05), but this does not imply stronger robustness. Together with its lower success, it suggests limited effective use of the learned executable-skill interface. Overall, Table~\ref{tab:transfer_aggregate_expanded_llm} supports the main conclusion that stronger LLMs improve absolute transfer performance, while the benchmark's degradation pattern remains stable across backbones.

\subsection{Generalization to a Second Skill-Learning Method}
\label{sec:app_awm_generalization}

MPCR is evaluated in the main paper over ASI's executable skill library. To test whether its benefit depends on that particular skill format, we apply it to AWM, which reuses prompted textual workflows rather than executable skill invocations and for which the executable-skill diagnostics in Table~\ref{tab:transfer_aggregate_expanded} are undefined. The two methods therefore differ in what a ``candidate'' is: for ASI a skill candidate is a callable action in the extended action space, whereas for AWM the retrieved workflow enters the prompt as guidance and the executable candidates are the primitive actions it recommends. MPCR is attached at the same point in both cases, namely the proposal interface of Section~\ref{sec:mpcr_propose}: the proposal stage enumerates the primitive actions suggested by the retrieved workflow together with alternatives from the base executor, and verification, scoring, and reranking then operate unchanged. The workflow library itself remains frozen, exactly as the skill library does for ASI.

\begin{table}[htbp]
\centering
\small
\setlength{\tabcolsep}{8.5pt}
\renewcommand{\arraystretch}{1.2}
\begin{tabular}{@{}lccr@{}}
\toprule
\textbf{Condition} & \textbf{AWM} & \textbf{AWM + MPCR} & \textbf{Rel.} \\
\midrule
\texttt{Src}      & 32.46 & 41.23 & +27.0\% \\
\texttt{Src-Perc} & 31.00 & 40.35 & +30.2\% \\
\texttt{Src-Exec} & 13.27 & 20.35 & +53.4\% \\
\texttt{Tgt}      & 49.11 & 57.02 & +16.1\% \\
\texttt{Tgt-Perc} & 32.46 & 39.47 & +21.6\% \\
\texttt{Tgt-Exec} & 10.00 & 18.42 & +84.2\% \\
\midrule
\textbf{Average}  & \textbf{28.05} & \textbf{36.14} & \textbf{+28.8\%} \\
\bottomrule
\end{tabular}
\caption{Success rate (\%) of AWM with and without MPCR on Magento. ``Rel.'' is the relative gain over the AWM baseline.}
\label{tab:app_awm_mpcr}
\end{table}

Table~\ref{tab:app_awm_mpcr} reports the result on Magento. Two observations follow. First, the gain holds in all six conditions, with no condition in which the reranker degrades the base method. Because AWM's integration mechanism is architecturally different from ASI's, this indicates that feasibility-aware test-time selection is not tied to the executable-skill format. Second, and more informatively, the pressure-dependent pattern reported in Section~\ref{sec:RQ2} reappears independently: relative gains are largest under Execution Shift, at $+53.4\%$ on \texttt{Src-Exec} and $+84.2\%$ on \texttt{Tgt-Exec}, well above the clean and perceptual conditions. The correspondence in Table~\ref{tab:app_swop_mpcr_map} therefore predicts \emph{where} the gain concentrates on a second method whose library it was not derived from, which is stronger evidence for the underlying mechanism than the aggregate improvement alone.

\subsection{Stability across Random Seeds}
\label{sec:app_seed_stability}

To characterize how much of the MPCR gain is attributable to run-to-run variation, we evaluate MPCR and the fixed-library ASI baseline over three seeds on Magento across all six conditions. The three seeds comprise the run reported in the main paper plus two additional seeds; this is a variance characterization of the reported configuration rather than an independent replication. Since the main tables pool three sites, the absolute values in Table~\ref{tab:app_seed_stability} are not comparable to Table~\ref{tab:mpcr_gpt5mini_asi}, and we compare only within this subset.

\begin{table}[htbp]
\centering
\small
\setlength{\tabcolsep}{10pt}
\renewcommand{\arraystretch}{1.25}
\begin{tabular}{@{}lccr@{}}
\toprule
\textbf{Condition} & \textbf{MPCR} & \textbf{ASI} & \textbf{Rel.} \\
\midrule
\texttt{Src}      & 49.12\,$\pm$\,2.32 & 43.27\,$\pm$\,1.34 & +13.5\% \\
\texttt{Src-Perc} & 53.80\,$\pm$\,2.03 & 38.30\,$\pm$\,1.34 & +40.5\% \\
\texttt{Src-Exec} & 29.24\,$\pm$\,1.34 & 17.54\,$\pm$\,0.88 & +66.7\% \\
\texttt{Tgt}      & 41.52\,$\pm$\,1.34 & 40.06\,$\pm$\,0.51 & \phantom{0}+3.6\% \\
\texttt{Tgt-Perc} & 36.84\,$\pm$\,0.88 & 36.55\,$\pm$\,1.34 & \phantom{0}+0.8\% \\
\texttt{Tgt-Exec} & 19.01\,$\pm$\,0.51 & 17.25\,$\pm$\,1.01 & +10.2\% \\
\midrule
\textbf{Mean}     & \textbf{38.26} & \textbf{32.16} & \textbf{+18.9\%} \\
\bottomrule
\end{tabular}
\caption{Success rate (\%), mean\,$\pm$\,standard deviation over three seeds on Magento. ``Rel.'' is the relative gain of MPCR over the ASI baseline.}
\label{tab:app_seed_stability}
\end{table}

Cross-seed variation is small for both methods, with per-condition standard deviations at or below $2.32$ points, and the mean MPCR success exceeds the mean ASI success in every condition. Pooled over Magento, MPCR improves success by $18.9\%$ relative to ASI. The gain is therefore stable rather than seed-driven, and the direction and magnitude are consistent with the three-site result reported in the main paper.

The distribution of the gain is as informative as its size. The three largest improvements are all on the source side, and \texttt{Src-Exec} shows the largest relative improvement at $+66.7\%$, matching the ordering predicted by our local-executability account. The two near-zero cells, \texttt{Tgt} and \texttt{Tgt-Perc}, are equally consistent with it: in those conditions the fixed library already transfers well, so there is little executability mismatch left for test-time selection to recover. Finally, MPCR is not worse than ASI in any condition, so feasibility-aware selection helps where the learned prior loses feasibility and is harmless where the prior remains aligned.

\subsection{Robustness to Repeated and Randomized Interruption}
\label{sec:app_reinjection}

The Execution Shift conditions in the main paper use the \texttt{once} injection regime of Appendix~\ref{sec:app_injection_regimes}, under which a dismissed blocker is not reinstated. This subsection asks whether the fixed-library reranking advantage survives more aggressive schedules. We evaluate MPCR and ASI under all three regimes on Magento, in both the source and the target context.

\begin{table}[htbp]
\centering
\small
\setlength{\tabcolsep}{8.5pt}
\renewcommand{\arraystretch}{1.2}
\begin{tabular}{@{}lcrrr@{}}
\toprule
\textbf{Regime} & \textbf{Context} & \textbf{MPCR} & \textbf{ASI} & \textbf{Rel.} \\
\midrule
\multirow{2}{*}{\texttt{once}}     & \texttt{Src} & 30.70 & 16.96 & +81.0\% \\
                                   & \texttt{Tgt} & 18.42 & 16.67 & +10.5\% \\
\midrule
\multirow{2}{*}{\texttt{multiple}} & \texttt{Src} & 20.35 & 17.95 & +13.4\% \\
                                   & \texttt{Tgt} & \phantom{0}8.77 & \phantom{0}7.86 & +11.6\% \\
\midrule
\multirow{2}{*}{\texttt{random}}   & \texttt{Src} & 30.97 & 26.95 & +14.9\% \\
                                   & \texttt{Tgt} & 20.18 & 18.18 & +11.0\% \\
\bottomrule
\end{tabular}
\caption{Success rate (\%) under the three Execution Shift injection regimes on Magento. ``Rel.'' is the relative gain of MPCR over the ASI baseline.}
\label{tab:app_reinjection}
\end{table}

The comparative conclusion is robust to the schedule. MPCR exceeds ASI in all six regime--context cells, so the advantage is not an artifact of the single-dismissal design: it persists when the blocker returns after every navigation and when its timing is randomized.

Absolute success, by contrast, is clearly sensitive to injection frequency, and in a direction that reinforces the benchmark's framing. The \texttt{multiple} regime is uniformly the hardest, reducing MPCR from $30.70\%$ to $20.35\%$ in the source context and from $18.42\%$ to $8.77\%$ in the target context, while \texttt{random} is the mildest because probabilistic injection sometimes produces no interruption at all. The \texttt{once} setting used in the main results is thus a conservative instantiation of Execution Shift, and more aggressive schedules strengthen rather than weaken the finding that Execution Shift exerts the strongest pressure among the three modes.

The margin does narrow under the most aggressive schedule, from $+81.0\%$ to $+13.4\%$ in the source context. We read this as a limit of the mechanism rather than as an anomaly. When the blocker is reinstated after every clearance, the environment variation that SWOP measures as $V_K$ is large, and the variation-dependent guarantee of Theorem~\ref{thm:swop-dynreg} degrades accordingly: transfer-time action selection alone has limited headroom against an obstruction that reappears faster than it can be recovered from. Complementary recovery mechanisms, which act on the interruption rather than only on the choice of action, are a natural direction for future work.

\subsection{Comparison with Target-Environment Re-Learning}
\label{sec:app_relearning}

A natural alternative to test-time adaptation is to relearn the skill library directly in the evolved environment. We evaluate this on Magento by running ASI's induction procedure in \texttt{Tgt}, so that the induced library is matched to the target interface, and then testing that target-native library under \texttt{Tgt-Perc} and \texttt{Tgt-Exec}, against the fixed source library used by ASI and MPCR.

\begin{table}[htbp]
\centering
\small
\setlength{\tabcolsep}{8.5pt}
\renewcommand{\arraystretch}{1.2}
\begin{tabular}{@{}lcc@{}}
\toprule
\textbf{Method} & \texttt{Tgt-Perc} & \texttt{Tgt-Exec} \\
\midrule
ASI (fixed \texttt{Src} library)  & 36.84 & 16.67 \\
Re-learned in \texttt{Tgt}        & 35.09 & 14.91 \\
\midrule
MPCR (same fixed library)         & \textbf{37.72} & \textbf{18.42} \\
\bottomrule
\end{tabular}
\caption{Success rate (\%) of target-environment re-learning against fixed-library transfer under the two shifted target conditions on Magento.}
\label{tab:app_relearning}
\end{table}

Re-learning a static library does not address runtime executability mismatch. Under Perceptual Shift the target-native library is the weakest of the three methods ($35.09\%$), and under Execution Shift it falls to $14.91\%$, below not only MPCR ($18.42\%$) but also the fixed source library it was meant to replace ($16.67\%$). MPCR exceeds it by $7.5\%$ relative under Perceptual Shift and by $23.5\%$ relative under Execution Shift, and the margin widens exactly where local executability is the binding constraint. A modal blocker injected at deployment is a per-episode runtime event, and no library induced before deployment, however well matched to the target interface, can encode a procedure for it.

The comparison is not confounded by the adaptation budget. Re-learning pays its cost once, during induction, and leaves inference unmodified thereafter, so the target-native library is not a compute-starved configuration. Its collapse under Execution Shift therefore reflects the limits of the static-library paradigm rather than an insufficient re-learning budget.

Re-learning in the target environment therefore does not remove the need for test-time adaptation. It requires target rollouts, and the library it yields is still fixed at deployment, so it inherits the same runtime feasibility gap that it was introduced to close. MPCR requires no target training data and recovers exactly the execution-shift loss that a re-learned library leaves open.

\subsection{Compute Budget and Infrastructure}
\label{app:compute_budget}

All experiments were conducted in self-hosted browser environments, with LLM inference accessed through hosted API services. We evaluate DeepSeek-V3.2, GPT-5 mini, and Claude Haiku 4.5. DeepSeek-V3.2 is an open-weight MoE model with 685B total parameters, while the exact parameter counts of GPT-5 mini and Claude Haiku 4.5 are not publicly disclosed by their providers. Since inference was performed through APIs, we did not use local GPU training or fine-tuning. The local infrastructure was used for browser orchestration, logging, and automatic grading. MPCR used a fixed lightweight reranking configuration with at most 2 candidate actions per step and a top-2 shortlist, without hyperparameter search beyond the candidate-budget sweep reported in Appendix~\ref{sec:app_budget_selection}.

\end{document}